\documentclass[twoside]{article}

\usepackage[accepted]{aistats2022}
\newcommand{\citep}[1]{\cite{#1}}
\newcommand{\citet}[1]{\cite{#1}}
\usepackage{macros}

\begin{document}

%

%

\twocolumn[

\aistatstitle{Chernoff Sampling for Active Testing and Extension to Active Regression}

\aistatsauthor{ Subhojyoti Mukherjee* \And Ardhendu Tripathy* \And  Robert Nowak }

\aistatsaddress{ UW-Madison \And  Missouri S\&T \And UW-Madison } ]

\begin{abstract}
Active learning can reduce the number of samples needed to perform a hypothesis test and to estimate the parameters of a model. In this paper, we revisit the work of Chernoff that described an asymptotically optimal algorithm for performing a hypothesis test. We obtain a novel sample complexity bound for Chernoff’s algorithm, with a non-asymptotic term that characterizes its performance at a fixed confidence level. We also develop an extension of Chernoff sampling that can be used to estimate the parameters of a wide variety of models and we obtain a non-asymptotic bound on the estimation error. We apply our extension of Chernoff sampling  to actively learn neural network models and to estimate parameters in real-data linear and non-linear regression problems, where our approach performs favorably to state-of-the-art methods.
\end{abstract}

\section{Introduction}
In contrast to common machine learning algorithms that use independent and identically distributed (iid) samples for training, active learning promises to use fewer samples by allowing the algorithm to choose the samples it is trained on. While the benefit of active learning has been analyzed extensively for the problem of classification \citep{dasgupta2005coarse,hanneke2007bound,DasguptaHM08,balcan2009agnostic,balcan2013active,ZhangC14,katz2021improved}, there are fewer works \citep{cai2016batch, wu2018pool, wu2019active, bu2019active} that utilize active learning for regression. In this paper we extend an asymptotically optimal algorithm for active testing, that was developed by Chernoff \citep{chernoff1959sequential}, to active regression.
We empirically show that this resulted in more efficient estimation of parameters in regression models. 
In addition, we obtain non-asymptotic bounds  on the sample complexity and estimation error for Chernoff’s algorithm in  active testing and its extension in active regression, respectively. Non-asymptotic bounds characterize the performance of an algorithm when executed at a fixed confidence level, which is relevant for real-world applications. 
While theoretical results for active regression using maximum likelihood estimates were given by \citet{chaudhuri1993nonlinear} and \citet{chaudhuri2015convergence}, our method is  likelihood-free and is applicable to sub-Gaussian observations. 

To frame our contributions, let us first establish some basic notation and a problem statement.  Consider a sequential learning problem in which the learner may select one of $n$ possible actions at each step. 
A sample resulting from action $i$ is a realization of a sub-Gaussian random variable with mean $\mu_i(\btheta^\ast)$,
where the mean $\mu_i(\btheta)$ is a known function parameterized by $\btheta \in \bTheta$. 
The specific $\btheta^*\in \bTheta$  that governs the observations is not known.  Each  action may be performed multiple times, resulting in i.i.d.\ observations, and  observations from different actions are also statistically independent. This paper considers the problem of sequentially and adaptively choosing actions for 
the following goals: 

{\bf In Active Testing:} $\bTheta$ is finite and the goal is to correctly determine the true hypothesis $\btheta^*$.

{\bf In Active Regression:} $\bTheta$ is a compact (uncountable) space and the goal is to accurately estimate $\btheta^*$.

Below, we detail our contributions to both problems.


\begin{itemize}
\item We revisit Chernoff's sampling algorithm for the sequential design of experiments \citep{chernoff1959sequential}, which is equivalent to the situation where the parameter space $\bTheta$ is a finite set.  The algorithm provably minimizes the number of samples used to identify $\btheta^*\in \bTheta$ in the asymptotic high-confidence setting. 
We derive a non-asymptotic sample complexity bound for the algorithm in \cite{chernoff1959sequential}
that characterizes its performance in low/medium confidence regimes. We also provide theoretical guarantees for three variations of the Chernoff sampling algorithm.
We prove a minimax lower bound that shows that the algorithm in \cite{chernoff1959sequential} can be optimal in the medium confidence regime.
We also generalize it to handle sub-Gaussian distributions. Consequently, we replace the maximum likelihood criterion (which depends on the probability distribution) with the minimum sum-of-squared errors criterion (which depends only on the mean functions). 
\item We extend the algorithm in \cite{chernoff1959sequential} to handle smoothly parameterized functions $\mu_i(\btheta)$ where $\bTheta \subseteq \R^d$. A brute-force approach could involve using a finite, discrete covering of $\bTheta$, but this is impractical.  Instead, we prove that an optimal sampling distribution (according to  Chernoff's criterion) is generally sparse and may be obtained by solving a simple eigenvalue optimization problem closely related to the notion of E-optimality in experimental design \citep{dette1993geometry}. We provide a convergence guarantee for the smoothly parameterized setting that utilizes a new error metric. We demonstrate that the extension of \cite{chernoff1959sequential}  outperforms existing stage-based algorithms in benchmark real-life datasets and in a neural network experiment.
\end{itemize}
We derive our non-asymptotic sample complexity bound for \cite{chernoff1959sequential} using the techniques in \cite{naghshvar2013active}. The convergence proof for active regression extends the techniques of \cite{Frostig15} and applies it to our extension of \cite{chernoff1959sequential}.

\subsection{Related Work}
The algorithm in \cite{chernoff1959sequential} assumes that the probability distribution of an observation from any action $i \in [n]$ under any hypothesis $\btheta \in \bTheta$ is known to the learner.
Consider a partition of $\bTheta = \bTheta_1 \cup \bTheta_2$ and a hypothesis test between $\btheta^\ast \in \bTheta_1$ and $\btheta^\ast \in \bTheta_2$. 
The objective is to choose actions such that the hypothesis test can be performed using as few samples as possible. 
Using past observations, a maximum likelihood estimate $\widehat{\btheta}$ is found and let $\widehat{\btheta} \in \bTheta_1$. The algorithm in \cite{chernoff1959sequential} chooses the next action according to a probability mass function (pmf) $\mathbf{p}$ over actions obtained by
\begin{equation}
\label{eq:chern_sampling}
\arg\max_{\mathbf{p}} \inf_{\btheta' \in \bTheta_2} \sum_{i=1}^n p(i) \, \KL(\nu_i(x; \widehat{\btheta}) \Vert \nu_i(x; \btheta')),
\end{equation}
where $\mathbf{p} = (p(1),\dots,p(n))$, $\nu_i(\cdot; \btheta)$ denotes the probability distribution of an observation from action~$i$ if $\btheta$ were the true hypothesis, and $\KL$ denotes the Kullback-Leibler divergence. 
The optimization \eqref{eq:chern_sampling} is similar to those appearing in sample complexity lower bounds for best-arm identification in multi-armed bandits \citet{garivier2016optimal,combes2017minimal,degenne2020gamification}. 
In those works the $\inf$ is taken over all $\btheta'$ having an optimal action that is different from that under the true $\btheta^*$.

While the algorithm in \cite{chernoff1959sequential} is asymptotically optimal under certain assumptions, subsequent works \citet{blot73sequential}, \cite{naghshvar2013active}, \cite{nitinawarat2013controlled} have proposed modifications that work well outside the asymptotic limit, strengthen theoretical guarantees, and reduce the number of assumptions needed. 
\citet{naghshvar2013active} proposed a two-phase Bayesian policy \twoph which conducts forced exploration in the first phase and computes a posterior belief over the hypotheses. Then in the second phase, it switches to the optimal Chernoff sampling proportion in \cref{eq:chern_sampling} if the probability of one hypothesis crosses a threshold. 
\twoph can relax an assumption made in \cite{chernoff1959sequential} which stated that sampling any action always provides some information about the true $\btheta^*$. 
If that assumption is true, then  Chernoff Sampling (\cher), which has no such separation of phases, empirically outperforms \twoph. It also enjoys both moderate and optimal asymptotic guarantees. 
\citet{nitinawarat2013controlled} have 
modified \cite{chernoff1959sequential} by adding a small amount of uniform exploration to relax the previous assumption. 
In a different problem \citet{vaidhiyan2017learning} have modified \cher to quickly identify an odd Poisson point process having a different rate of arrival than others.

For estimating parameters of a regression model, efficient methods for selecting actions have been studied in the area of Optimum Experiment Design \citep{silvey1980optimal, pukelsheim2006optimal, pronzato2013design}. However a major focus in these works has been on large-sample asymptotic properties of estimators obtained from a fixed sampling distribution. 
While adaptive sampling proportions have also been proposed (e.g. Section~8.5 in \cite{pronzato2013design}), there have been fewer works characterizing their theoretical properties. 
Most theoretical work on active learning has focused on learning binary classifiers that belong to a particular hypothesis class \citep{dasgupta2005coarse,hanneke2007bound,DasguptaHM08,balcan2009agnostic,balcan2013active,ZhangC14,katz2021improved}. 
The works of \cite{chaudhuri1993nonlinear} and \cite{chaudhuri2015convergence} propose adaptive sampling methods for obtaining maximum likelihood estimates of the parameters in a regression model. 
\cite{chaudhuri2015convergence} propose a two-stage algorithm \actives that first samples uniformly at random to obtain a preliminary estimate of the parameters, which is then used to find a sampling proportion for the second stage by solving an optimization problem. 
In contrast \cher is a fully adaptive algorithm and enjoys a similar convergence under slightly stronger assumptions. 

\citet{sabato2014active} provides an active learning algorithm for linear regression problems under model mismatch. 
The same setting under heteroscedastic noise has been studied by \citet{chaudhuri2017active} where they propose a
two-stage process adapted to the noise. \citet{fontaine2019online} also studies the linear regression setting under heteroscedastic noise but proposes a fully adaptive adaptive algorithm that is similar to A-optimal design. 
\citet{bu2019active} studies a different setting where $\btheta^*$ is changing with time. They modify the algorithm of \citet{chaudhuri2015convergence} to fully adaptive process where the optimization needs to be solved at every round. \cite{wu2018pool} studies the linear regression setting where the goal is to maximize the diversity of the samples. \citet{cai2016batch} studies both the linear and non-linear regression setting and proposes the heuristic \emcm without any convergence guarantee. Similarly, \citet{wu2019active} also studies the active regression for noiseless setting but they provide no convergence proof. As opposed to these works  \cher has a convergence guarantee and performs well in real-world benchmark problems. 

Another line of work is the non-parametric setup of \cite{castro2005faster}, where the objective is to estimate an unknown function over its entire domain. Here the error rates for learning are $O\left(t^{-\gamma}\right)$, and the exponent $\gamma$ decreases as the complexity of the hypothesis class of functions increases (i.e., there is a slower decrease in error when learning a more complicated function). For example, if the hypothesis class consists of Holder smooth functions defined on domain $[0,1]^{d}$ then $\gamma = 1/(d-1+ 1/d)$. In contrast, our work is in the parametric setting, where we only want to estimate a single parameter $\btheta^{*}$ and $\gamma=1$. 
We show that the \cher algorithm has a smaller  problem-dependent constant in the error bound. 
The work of \citet{goetz2018active} is also in a similar framework as that of \citet{castro2005faster}. 
Other forms of optimal experiment design have been explored in the context of active learning by \citet{yu2006active}, and in different bandit problems by \citet{soare2014best, fiez2019, degenne2020gamification}.
Note that our objective of identifying $\btheta^*$ is a strictly more difficult objective than best-arm identification in bandit problems. 

\section{Active Testing}
\label{sec:active-testing}
A sequential policy $\pi$ tasked to find $\btheta^\ast$ interacts with the environment in an iterative fashion. 
At time~$t$, the policy samples action~$I_t$ and receives a random observation $Y_t$ that follows the distribution $\nu_{I_t}(\cdot;\btheta^\ast)$, where $\btheta^\ast$ is the true value of the unknown parameter that belongs to a set $\bTheta$. 
In active testing, $\bTheta$ contains $J$ discrete hypotheses. 
Let $\F_t \colonequals \sigma(I_1, Y_1, I_2, Y_2, \ldots, I_t, Y_t)$ denote the sigma-algebra generated by the sequence of actions and observations till time~$t$. Then $\pi$ is said to be $\delta$-PAC if: (1) at each $t$ the sampling rule $I_t$ is $\F_{t-1}$ measurable, (2) it has a finite stopping time $\tau_\delta$ with respect to $\F_t$, and (3) its final prediction $\widehat{\btheta}(\tau_\delta)$ is based on $\F_{\tau_\delta}$ and satisfies $\Pb(\widehat{\btheta}(\tau_\delta) \neq \btheta^*) \leq \delta$. 
A table of notation is provided in Appendix \ref{app:notation}. Based on the observations $(Y_1, Y_2, \ldots, Y_t)$, we define for every $\btheta\in\bTheta$ the sum of squared errors and the difference between the sum of squared errors for $\btheta$ and $\btheta^*$ as follows:
\begin{align}
\label{eq:log-likelihood}
L_{t}(\btheta)
&\colonequals
\sum_{s=1}^t (Y_s - \mu_{I_s}(\btheta))^2, \\
\Delta_t(\btheta) &\colonequals
L_t(\btheta) - L_t(\btheta^*). 
\end{align}
\begin{assumption}
\label{assm:sub-gauss}
An observation from any action under any hypothesis has bounded range, i.e., $Y_s\in[-\nicefrac{\sqrt{\eta}}{2}, \nicefrac{\sqrt{\eta}}{2}]$ almost surely at every round $s$ for some fixed $\eta>0$.
\end{assumption}
Suppose $\nu_{i}(Y;\btheta)$ are Gaussian distributions with mean $\mu_{i}(\btheta)$ and variance $\nicefrac{1}{2}$. Let $\wtheta(t)$ denote the estimate for $\btheta^*$ at time $t$. Using $\bTheta_1 = \{\wtheta(t)\}, \bTheta_2 = \bTheta \setminus \{\wtheta(t)\}$ and  expressions for KL divergence of Gaussian distributions in \cref{eq:chern_sampling}, we obtain that \cher, which is asymptotically optimal, samples the next action according to a probability mass function (pmf) that is a solution to the following $\max \min$ optimization:
\begin{align}
\mathbf{p}_{\wtheta(t)} \!\! \colonequals \!
\argmax_{\mathbf{p}} \!\! \min_{\btheta'\neq \wtheta(t)} \!\sum\limits_{i=1}^n {p}(i)
(\mu_{i}(\btheta')-\mu_{i}(\wtheta(t)))^2. \label{eq:opt-lower00}
\end{align}

We can solve 
\eqref{eq:opt-lower00} by formulating it as a linear program: 
\begin{align}
    &\hspace*{-0.5em}\max_{\mathbf{p}} z 
    \: \textbf{s.t.}\! \sum_{i=1}^n p(i)(\mu_{i}(\btheta'){-}\mu_{i}(\wtheta(t)))^2 \geq z\: \forall \btheta'{\neq} \wtheta(t),\label{eq:opt:linear0}
\end{align}
where the optimization variables are the scalar $z$ and pmf
$\mathbf{p}$ satisfying the constraints $p(i)\geq0 \forall i$ and $\sum_{i=1}^{n}p(i)=1$. 

\textbf{Chernoff Sampling (\cher):} Inspired by the sampling proportion in \cref{eq:opt-lower00}, 
we use the same sampling strategy even though the distributions $\{\nu_i(Y;\btheta)\}_{i=1}^n$ are only assumed to be sub-Gaussian (Algorithm \ref{alg:cher}). Our estimate of the most likely hypothesis (breaking ties at random) given the data is 
$\widehat{\btheta}(t) \colonequals \argmin_{\btheta \in \bTheta} L_{t}(\btheta)$. 
The action sampled at the next time $t+1$ is chosen by the randomized rule $\Pb(I_{t+1} = i) = p_{\widehat{\btheta}(t)}(i), \forall i \in [n]$. 
We stop sampling at $\tau_\delta$ if the sum of squared errors for all competing hypothesis is greater than that of $\widehat{\btheta}(\tau_\delta)$ by a threshold 
$\beta(J,\delta)$ to be defined later. 

\begin{algorithm}[H]
\caption{Chernoff Sampling for Active Testing}
\label{alg:cher}
\begin{algorithmic}[1]
\State{\textbf{Input:} Confidence parameter $\delta$, threshold $\beta(J,\delta)$ 
}
\State{Sample $I_1\in [n]$ randomly, observe $Y_1$ and find $\widehat{\btheta}(1)$.}
\For{$t=2,3,\ldots$}
\State{Sample $I_{t} \sim \mathbf{p}_{\widehat{\btheta}(t-1)}$ from \eqref{eq:opt-lower00} and observe $Y_{t}$.}
\State{Calculate $L_{t}(\btheta)$ from \eqref{eq:log-likelihood} $\forall\btheta\in \bTheta$, find $\widehat{\btheta}(t)$.}
\If{$ L_{t}(\btheta')  {-}  L_{t}(\widehat{\btheta}(t)) > \beta(J,\delta) \forall \btheta' \neq \widehat{\btheta}(t)$}
\State{\textbf{Return} $\widehat{\btheta}(t)$ as the true hypothesis.}
\EndIf
\EndFor
\end{algorithmic} 
\end{algorithm} 

In \citet{chernoff1959sequential} they provide a sample complexity upper bound only for the asymptotic regime when $\delta\rightarrow 0$. We give a non-asymptotic fixed confidence sample complexity upper bound in Theorem \ref{thm:chernoff-upper}. The following assumption, originally made by \citet{chernoff1959sequential}, is used to prove Theorem \ref{thm:chernoff-upper}.
\begin{assumption}
\label{assm:positive-D1}
The mean of the observation from any action under $\btheta^*$ is different from its mean under any other hypothesis, i.e., $\min_{i \in [n]}\min_{\btheta \neq \btheta^\ast}|\mu_i(\btheta) - \mu_i(\btheta^*)|>0$.
\end{assumption}
In subsequent works by \citet{nitinawarat2013controlled} and \citet{naghshvar2013active}, it was shown that the above assumption can be relaxed if the algorithm is modified.
We make the assumption since we give a non-asymptotic sample complexity bound for the original algorithm. 
\begin{definition}
\label{def:bounded}
Define the smallest squared difference of means between any two actions under any pair of hypotheses $\btheta, \btheta'$ as
$\eta_0 \colonequals \min_{i \in [n]}\min_{\btheta \neq \btheta'} (\mu_{i}(\btheta) - \mu_{i}(\btheta'))^2$. By \Cref{assm:positive-D1} we have $\eta_0 > 0$. 
\end{definition}
We define the threshold $\beta(J, \delta) = \log(CJ/\delta)$ where $C$ is a constant depending on $\eta$, and $\eta_0$. The values $\eta$ (or an upper bound to it) and $\eta_0$ are known to the learner. 
\begin{customtheorem}{1}
\label{thm:chernoff-upper}\textbf{(\cher Sample Complexity)}
Let $\tau_\delta$ denote the stopping time of \cher in Algorithm \ref{alg:cher}. 
Let $D_0$ be the objective value of the $\max\min$ optimization in \eqref{eq:opt-lower00} when $\btheta=\btheta^\ast$, i.e.,
\begin{align*}
D_0 &\colonequals \max_{\text{ }\mathbf{p}} \min_{\btheta' \neq \btheta^\ast} \sum_{i=1}^n p(i) (\mu_i(\btheta') -\mu_i(\btheta^\ast))^2.
\end{align*}
Denote $\mathbf{p}_{\btheta}$ as the solution of \eqref{eq:opt-lower00} when $\widehat{\btheta}(t)$ is replaced by any $\btheta \in \bTheta$, and $D_1$ is the minimum possible objective value 
over all $\mathbf{p}_{\btheta}$ when $\widehat{\btheta}(t)$ is replaced by $\btheta^*$, i.e., 
\begin{align*}
D_1 &\colonequals \min_{\{\mathbf{p}_{\btheta} : \btheta \in \bTheta\}} \min_{\btheta' \neq \btheta^\ast} \sum_{i=1}^n p_{\btheta}(i)(\mu_i(\btheta') - \mu_i(\btheta^*))^2.
\end{align*}
Assumption~\ref{assm:positive-D1} 
ensures that $D_1 > 0$. The sample complexity of the $\delta$-PAC \cher has the following upper bound, where $J \colonequals |\bTheta|$, $C = O((\eta/\eta_0)^2)$ is a constant:
\begin{align*}
\E[\tau_\delta] \leq O\left(\dfrac{\eta\log( C)\log J}{D_1} + \dfrac{\log(J/\delta)}{D_0} + JC^{\frac{1}{\eta}}\delta^{\frac{D_0}{\eta^2}}\right).
\end{align*}
\end{customtheorem}
\begin{proof} \textbf{(sketch)}
Algorithm \ref{alg:cher} stops at $\tau_\delta$ when the error for the returned hypothesis is smaller than the error for all the other hypotheses by an amount of $\beta(J,\delta)$. To obtain an upper bound to $\mathbb{E}[\tau_\delta]$, we instead look at a different random time $\tau_{\btheta^*}\colonequals \min \{t : \Delta_t(\btheta) > \beta(J,\delta), \forall \btheta \neq \btheta^*\}$, which is the first time when the error for the true hypothesis $\btheta^*$ is smaller than the error for all other hypotheses by $\beta(J,\delta)$. 
Either the hypothesis returned by the Algorithm \ref{alg:cher} is $\wtheta(\tau_\delta) = \btheta^*$, in which case $\tau_\delta = \tau_{\btheta^*}$, or $\tau_{\btheta^*}$ has not occurred yet and $\tau_\delta < \tau_{\btheta^*}$. Hence we focus on bounding $\mathbb{E}[\tau_{\btheta^*}]$. 
The key random quantity in the definition of $\tau_{\btheta^*}$ is $\Delta_t(\btheta)$, and 
using \Cref{assm:sub-gauss} we can show that $\Delta_{t}(\btheta)$ concentrates to its expected value. The expected value $\mathbb{E}[\Delta_t(\btheta)]$ is increasing with $t$ for each $\btheta \neq \btheta^*$ and for large enough $t$ it will be greater than $\beta(J, \delta)$. 
Since $\Delta_t(\btheta)$ concentrates to $\mathbb{E}[\Delta_t(\btheta)]$, for large enough $t$, $\Delta_t(\theta)$ will also be greater than $\beta(J,\delta)$ and $\tau_{\btheta^*}$ would occur. 
To quantify when $\tau_{\btheta^*}$ occurs, we lower bound  $\mathbb{E}[\Delta_t(\btheta)]$ as follows:
\begin{equation}
\mathbb{E}[\Delta_t(\btheta)] \geq \mathbb{E}[\tilde{\tau}_{\btheta^*}D_1 + (t-\tilde{\tau}_{\btheta^*})D_0],
\end{equation}
where $\tilde{\tau}_{\btheta^*}$ is the last time after which the error for the true hypothesis $\btheta^*$ is always smaller than the errors for all other hypotheses. 
Till the time $\tilde{\tau}_{\btheta^*}$, the \cher sampling proportion $\mathbf{p}_{\wtheta(t)}$ may not be $\mathbf{p}_{\btheta^*}$, and $\mathbb{E}[\Delta_t(\btheta)]$ grows at the slower ``exploration'' rate $D_1$ defined in the Theorem. After $\tilde{\tau}_{\btheta^*}$ the \cher proportion is $\mathbf{p}_{\btheta^*}$, and $\mathbb{E}[\Delta_t(\btheta)]$ increases at the optimal ``verification'' rate $D_0$. We finally bound the sample complexity by using $\!\E[\tau_{\btheta^*}] \!=\! \sum_t \!\Pb(\tau_{\btheta^*} = t) \leq M \!+\!  \Pb(\tau_{\btheta^*} > M \cap \tilde{\tau}_{\btheta^*} \leq M) \! + \!\Pb(\tau_{\btheta^*} > M \cap \tilde{\tau}_{\btheta^*} > M)$
where, $M = \frac{\eta\log (C)\log J}{D_1} + \frac{\log(J/\delta)}{D_0}$. The two tail events above is shown to be bounded by $O(JC^{1/\eta}\delta^{D_0/\eta^2})$ in \Cref{conc:lemma:1}. The full proof is in Appendix \ref{app:thm:chernoff-upper}. 
\end{proof}

In the result of \Cref{thm:chernoff-upper} the first term $\eta\log(C)\log{(J)}/D_1$ bounds the number of samples taken during the exploration phase when $\widehat{\btheta}(t) \neq \btheta^\ast$. 
This is the non-asymptotic term that is not present in the analysis of 
\citet{chernoff1959sequential}. 
The second term  $\log{(J/\delta)}/D_0$ is the dominating term when $\delta \rightarrow 0$, and it matches the asymptotic sample complexity expression of \citet{chernoff1959sequential}. 
\citet{naghshvar2013active} also derive a moderate confidence bound for their policy called \twoph but suffer from a worse non-asymptotic term $\log(J/\delta)/D_{\text{NJ}}$ where $D_{\text{NJ}} < D_1$ ($D_{\text{NJ}}$ is denoted as $I_1(M)$ in \citet{naghshvar2013active}). \twoph do not require the assumption that $D_1 > 0$ as it conducts forced exploration in the first stage.  
The \twoph policy is asymptotically optimal but performs poorly in some instances (see \Cref{ex:unif-dominates} and \Cref{sec:expt}) due to the fixed  exploration in the first stage. We discuss further results in Appendix~\ref{app:theory}. 

The following example shows that the non-asymptotic term of \cher may dominate.
\begin{example}\textbf{(Non-asymptotes matter)}
\label{ex:unif-dominates}
Consider an environment with two actions and $\bTheta = \{\btheta^\ast, \btheta', \btheta''\}$. The following table describes the values of $\mu_1(\cdot), \mu_2(\cdot)$ under these three hypotheses.
%
\begin{align*}
\begin{matrix}
     \btheta &= & \btheta^* &\btheta'  & \btheta^{''} \\\hline
    \mu_1(\btheta) &=  & 1 & 0.001 & 0 \\
    \mu_2(\btheta) &=   & 1 & 1.002 & 0.998
\end{matrix}
\end{align*}
%
For a choice of $\delta=0.1$, 
we can evaluate that $\log (J) / D_1 \approx 3 \times 10^5$ and $\log(J/\delta)/D_0 \approx 3.4$. 
While Theorem \ref{thm:chernoff-upper} is only an upper bound, empirically we do see that the non-asymptotic term dominates the sample complexity. The Figure \ref{fig:Figure-2enva} shows a box plot of the stopping times of four algorithms over $100$ independent trials on the above environment. A box plot depicts a set of numerical data using their quartiles \citep{tukey1977exploratory}.
A uniform sampling baseline (\unif) performs much better than \cher, as it samples action~$1$ half the time in expectation, and action~$1$ is the best choice to distinguish $\btheta^\ast$ from both $\btheta'$ and $\btheta''$.
\twoph also performs poorly compared to \unif and \cher in this setting. The non-asymptotic term of \twoph scales as $\log(J)/D_{\text{NJ}} \approx 4\times 10^6$.
\end{example}

\textbf{\topllr Sampling:} 
Instead of sampling according to the optimal verification proportion in line~4 of Algorithm~\ref{alg:cher}, we can use the following heuristic argument. 
At each time, consider the current most likely $\widehat{\btheta}(t)$ and its ``closest'' competing hypothesis defined as $\widetilde{\btheta}(t) \colonequals
\arg \min_{\btheta \neq \widehat{\btheta}(t)} L_{t}(\btheta)$. The $\delta$-PAC heuristic called Top-2 sampling (abbreviated as \topllr) samples an action that best discriminates between them, i.e.,
\begin{align}
I_{t+1} \colonequals \arg\max_{i \in [n]} (\mu_i(\widehat{\btheta}(t)) - \mu_i(\widetilde{\btheta}(t)))^2.
\label{eq:max-deriv-finite0}
\end{align}
This strategy requires lesser computation as we don't need to compute  $\bm{p}_{\widehat{\btheta}(t)}$ which could be useful when $J$ or $n$ is very large. It was proposed by \citet{chernoff1959sequential} without any sample complexity proof. 


\begin{customproposition}{1}\textbf{(\topllr Sample Complexity)}
\label{thm:topllr-upper}
Let $\tau_\delta$ denote the stopping time of \topllr following the sampling strategy of \eqref{eq:max-deriv-finite0}. 
Consider the set $\Is(\btheta, \btheta') \subset [n]$ of actions that could be sampled following \eqref{eq:max-deriv-finite0} when $\widehat{\btheta}(t) = \btheta$ and $\tilde{\btheta}(t) = \btheta'$, and let $\mathbf{u}_{\btheta\btheta'}$ denote a uniform pmf supported on $\Is(\btheta, \btheta')$. 
Define 
\begin{align*}
D_0' \colonequals \min\limits_{\btheta,\btheta' \neq \btheta^*}\sum_{i=1}^n u_{\btheta^*\btheta}(i)(\mu_i(\btheta')- \mu_i(\btheta^*))^2, \quad 
\\
D_1' \colonequals \min_{\btheta \neq  \btheta', \btheta' \neq \btheta^*}\sum_{i=1}^n u_{\btheta\btheta'}(i)(\mu_i(\btheta') - \mu_i(\btheta^*))^2, 
\end{align*}
where we assume that $D_1' > 0$. Then for a constant $C >0$ 
the sample complexity of \topllr has the following upper bound:
\begin{align*}
    \E[\tau_\delta] \leq O\left(\frac{\eta\log(C)\log J}{D_1'} + \frac{\log(J/\delta)}{D_0'} + JC^{\frac{1}{\eta}}\delta^{\frac{D_0'}{\eta^2}}\right).
\end{align*}
\end{customproposition}
The bound above has a similar form as in \Cref{thm:chernoff-upper} with three terms. The first term does not scale with error probability $\delta$, while the second term scales with $\log(J/\delta)$. The denominators of the two terms are different from $D_0$ and $D_1$ due to the different sampling rule of \topllr and hence \topllr is not asymptotically optimal. 

\textbf{Batch Updates:} 
We can solve $\max\min$ optimization for $\mathbf{p}_{\widehat{\btheta}}(t)$ every $B$ rounds instead of at each round. This reduces computation while increasing the sample complexity by only an additive term as shown below. 
\begin{customproposition}{2}\textbf{(\bcher Sample Complexity)}
\label{prop:batch-cher}
Let $\tau_\delta$, $D_0$, $D_1$ be defined as in \Cref{thm:chernoff-upper} and $B$ be the batch size. Then the sample complexity of $\delta$-PAC \bcher is
\begin{align*}
    \E[\tau_\delta] \!\leq\! O\left(B \!+\! \frac{\eta\log( C)\log J}{D_1} \!+\! \frac{\log(J/\delta)}{D_0} \!+\! BJC^{\frac{1}{\eta}}\delta^{\frac{D_0}{\eta^2}}\right).
\end{align*}
\end{customproposition}

\textbf{\cher with Exploration:} Recall that $D_1 > 0$ (\Cref{assm:positive-D1}) is required to prove \Cref{thm:chernoff-upper}. We now relax this assumption with the policy \rcher which 
follows the proportion $p_{\wtheta(t)}$ with probability $1-\epsilon_t$ and uniform randomly explores any other action $i\in [n]$ with probability $\epsilon_t$. The exploration parameter $\epsilon_t$ is chosen to reduce with time. 
Define $D_e \colonequals   \min_{\btheta' \neq \btheta^\ast} \sum_{i=1}^n \frac{1}{n}(\mu_i(\btheta') - \mu_i(\btheta^*))^2$ as the objective value for uniform sampling, then $D_e > 0$. 
\begin{customproposition}{3}\textbf{(\rcher Sample Complexity)}
\label{prop:cher-e}
Let $\tau_\delta$, $D_0$, $C$ be defined as in \Cref{thm:chernoff-upper}, $D_e$ be defined as above, and $\epsilon_t \colonequals 1/\sqrt{t}$. Then the sample complexity bound of $\delta$-PAC \rcher with $\epsilon_t$ exploration is given by 
\begin{align*}
    E[\tau_\delta] \!\leq\! O\left(\frac{\eta\log(C)\log J}{D^{}_e} + \frac{\log(J/\delta)}{D_0} + JC^{\frac{1}{\eta}}\delta^{\frac{D_0}{\eta^2}} \right).
\end{align*}
\end{customproposition}
We can see that the non-asymptotic term does not depend on $D_1$ and scales with $D_e$. Note that  $D_{0}>D_{1}$ and $D_{0}>D_{e}$ separately but $D_{e}$ and $D_{1}$ are not comparable because $D_{1}$ is defined as the minimum over verification proportions for all the hypotheses while $D_{e}$ is defined using a uniform sampling proportion. 
When $\delta \rightarrow 0$ then the asymptotic term dominates and so \rcher is asymptotically optimal. In \Cref{ex:unif-dominates} we can calculate that $\log(J)/D_e \approx 2.1$ and $\log(J/\delta)/D_0 \!= \!3.4$. So \rcher  performs similar to \unif (see \Cref{fig:Figure-2enva}) as well as theoretically enjoy asymptotic and moderate confidence guarantee similar to \cher.

We now provide a brief proof sketch of the three propositions stated before. These proofs follows the technique of \Cref{thm:chernoff-upper} with some key changes which we state now. For \Cref{thm:topllr-upper} observe that \topllr 
does not sample by $\mathbf{p}_{\btheta}$ but by the pmf  $\mathbf{u}_{\btheta\btheta'}$ 
defined in \Cref{thm:topllr-upper}. This enables us to  calculate  
$\mathbb{E}[\Delta_{t}(\btheta)] \geq \E[ \tilde{\tau}_{\btheta^*}D'_{1} +  (t-\tilde{\tau}_{\btheta^*})D'_{0}]$, 
where $\tilde{\tau}_{\btheta^*}$ is defined in \Cref{thm:chernoff-upper}. After this we can follow a similar line of reasoning as \Cref{thm:chernoff-upper} and bound \topllr sample complexity. The proof is in \Cref{app:topllr-upper}. For \Cref{prop:batch-cher} the key difference with \Cref{thm:chernoff-upper} is that we calculate the $\mathbf{p}_{\btheta}$ after each batch of size $B$. We bound the number of total number of batches $m_{\delta}$ instead of $\tau_\delta$. 
As the stopping condition is only checked at the end of every batch we divide the time $\tau_\delta$ into batches of size $B$ and use a similar argument as in \Cref{thm:chernoff-upper} to bound $m_{\delta}$. The proof is in \Cref{app:prop:batch-cher}. Finally, for \Cref{prop:cher-e} the key difference is the new exploration term $D_e$.
By setting $\epsilon_s \colonequals 1/\sqrt{s}$ 
we obtain
$\mathbb{E}[\Delta_{t}(\btheta)] \geq  \E[\tilde{\tau}_{\btheta^*}D_{e} + (t-\tilde{\tau}_{\btheta^*})D_{0}]$. 
Then following the same argument as in \Cref{thm:chernoff-upper} we obtain the upper bound to $\E[\tau_{\delta}]$. The proof is given in \Cref{app:rcher-upper-bound}.

\textbf{Minimax lower bound:} 
While \citet{chernoff1959sequential} had shown the policy to be optimal as $\delta \rightarrow 0$, we demonstrate an environment where \cher has optimal sample complexity for any fixed value of $\delta$.
Let $\Gamma \!=\! \sqrt{\eta}/2$. The following table depicts the values for $\mu_1(\cdot), \mu_2(\cdot), \ldots, \mu_n(\cdot)$ under $J$ different hypotheses: 
\begin{align}
\label{eq:minimax-environment}
\hspace*{-0.8em}\begin{matrix} 
    \btheta &= & \btheta^* &\btheta_2  &  \btheta_3 & \ldots & 
    \btheta_J \\\hline
    \mu_1(\btheta) &=  & \Gamma & \Gamma\!-\!\frac{\Gamma}{J} & \Gamma\!-\!\frac{2\Gamma}{J} & \ldots & \Gamma\!-\!\frac{(J-1)\Gamma}{J}\\
    \mu_2(\btheta) &=   & \iota_{21} & \iota_{22} & \iota_{23} & \ldots & \iota_{2J}\\
    &\vdots & &&\vdots\\
    \mu_n(\btheta) &=   & \iota_{n1} & \iota_{n2} & \iota_{n3} & \ldots & \iota_{nJ}
\end{matrix}
\end{align}
Each $\iota_{ij}$ is distinct and satisfies $\iota_{ij} < \Gamma/4J$.  $\mu_1(\cdot)$ is such that the difference of means across any pair of hypotheses is at least $\Gamma/J$. Theorem~\ref{thm:minimax} is proved in Appendix~\ref{app:minimax} by a change of measure argument. 
Note that action~$1$ is better than all others in discriminating between any pair of hypotheses, and any policy to identify $\btheta^\ast$ cannot do better than allocating all its samples to action~$1$. 

\begin{customtheorem}{2}\textbf{(Lower Bound)}
\label{thm:minimax} 
Any $\delta$-PAC policy $\pi$ that identifies $\btheta^\ast$ 
in \eqref{eq:minimax-environment} 
satisfies $\E[\tau_\delta] \geq \Omega\left({J^2\Gamma^{-2}} \log({1}/{\delta})\right)$. 
Applying Theorem~\ref{thm:chernoff-upper} to the same environment, the sample complexity of \cher is  $O\left(J^2\Gamma^{-2}\log(J/\delta)\right)$ which matches the lower bound upto log factors.
\end{customtheorem}


\section{Active Regression}
\label{sec:active-regression}
In this section, we extend the Chernoff sampling policy to smoothly parameterized hypothesis spaces, such as $\bTheta \subseteq \R^d$.  
The original sampling rule in \eqref{eq:opt-lower00} asks to solve a $\max \min$ optimization, where the $\min$ is over all possible choices of the parameter that are not equal to the parameter $\btheta$ being verified. 
An extension of the rule for when $\btheta^\ast$ can take infinitely many values was first given by \citet{albert1961sequential}. In it, they want to identify which of two partitions $\bTheta_1 \cup \bTheta_2 = \bTheta$ does the true $\btheta^\ast$ belong to. For any given $\btheta_1 \in \bTheta_1$, their verification sampling rule (specialized to the case of Gaussian noise) is
\begin{align}\label{eq:albert-opt}
\mathbf{p}_{\btheta_1} = \argmax_{\text{ } \mathbf{p}} \inf_{\btheta_2 \in \bTheta_2} \sum_{i=1}^n p(i) (\mu_i(\btheta_1) - \mu_i(\btheta_2))^2.
\end{align}
Recall that $\wtheta(t) \colonequals \argmin_{\btheta \in \bTheta} L_{t}(\btheta)$. 
Suppose we want to find the optimal verification proportion for testing $\wtheta(t)$, the current best estimate of $\btheta^\ast$. 
Let $\mathcal{B}^\complement_r(\wtheta(t)) \colonequals \{\btheta \in \mathbb{R}^d : \lVert \btheta - \wtheta(t) \rVert > r\}$ denote the complement of a ball of radius $r > 0$ centered at $\wtheta(t)$. Instantiate \eqref{eq:albert-opt} with $\btheta_1 = \wtheta(t)$, $\bTheta_1 = \mathbf{\Theta} \setminus \mathcal{B}^\complement_r(\wtheta(t))$ and $\bTheta_2 = \mathcal{B}^\complement_r(\wtheta(t))$. 
Denote the solution of this optimization as $\mathbf{p}_{\wtheta(t), r}$ and let $\mathbf{p}_{\wtheta(t)} \colonequals \lim_{r \rightarrow 0} \mathbf{p}_{\wtheta(t), r}$. 
In case of multiple solutions, we let $\mathbf{p}_{\wtheta(t), r}$ denote the set of all possible maxima, and the limit is defined to be the limit of a sequence of sets. 
We show in \Cref{thm:max-min-eig} that $\mathbf{p}_{\wtheta(t)}$ 
(or an element from it)
can be computed efficiently. 
For any $i \in [n]$ the gradient of $\mu_i(\cdot)$ evaluated at $\btheta$ is a column vector denoted as $\nabla \mu_i (\btheta)$. 
\begin{customtheorem}{3} 
\label{thm:max-min-eig}
Assume that $\mu_i(\btheta)$ for all $i \in [n]$ 
is a differentiable function, 
and the set 
$\{\nabla \mu_i (\wtheta(t)) : i \in [n]\}$ 
of gradients evaluated at $\wtheta(t)$ span 
$\mathbb{R}^d$. 
Consider a p.m.f. $\mathbf{p}_{\wtheta(t), r}$ 
from \eqref{eq:albert-opt} for verifying $\wtheta(t)$ against 
all alternatives in $\mathcal{B}^\complement_r(\wtheta(t))$. 
The limiting value of $\mathbf{p}_{\wtheta(t), r}$ as $r \rightarrow 0$ is
\begin{align*}
\mathbf{p}_{\wtheta(t)}:=
\arg\max_{\text{ }\mathbf{p}}  
\lambda_{\min}
\left(
\sum_{i=1}^{n} p(i) \nabla\mu_i(\wtheta(t))
\nabla\mu_i(\wtheta(t))^T
\right).
\end{align*}
\end{customtheorem}
\begin{proof}
\textbf{(sketch)}
Define $g_i(\btheta) \colonequals (\mu_i(\btheta) - \mu_i(\widehat{\btheta}(t)))^2$ for any  $\btheta$. 
Introducing a probability density function $\mathbf{q}$ over $\bTheta$, we can rewrite the optimization for $\mathbf{p}_{\widehat{\btheta}(t), r}$ from \eqref{eq:albert-opt} as
\begin{align}\label{eq:mm_overview}
\max_{\text{ } \mathbf{p}} \inf_{\mathbf{q}: q(\btheta) = 0 \forall \btheta \in \mathcal{B}_r(\wtheta(t))} 
\int_{\bTheta} q(\btheta) \sum_{i=1}^{n} p(i) 
g_i(\btheta) d\btheta.
\end{align}
We consider a family $\mathcal{Q}_r$ of pdfs supported on the boundary of $\mathcal{B}_r(\widehat{\btheta}(t))$. 
We show that the value of \eqref{eq:mm_overview} in the limit as $r \rightarrow 0$ is equal to
\begin{align}\label{eq:mm_overview_2}
\lim_{r \rightarrow 0} \max_{\mathbf{p}} \inf_{\mathbf{q} \in \mathcal{Q}_r} \int_{\bTheta} 
q(\btheta) \sum_{i=1}^{n} p(i) g_i(\btheta) d\btheta.
\end{align}
We use the Taylor series expansion for $g_i(\btheta)$ around $\wtheta(t)$ in \eqref{eq:mm_overview_2}. 
Then $\nabla g_i(\wtheta(t)) = 0$
and the second-order term in the Taylor series is
$0.5(\btheta - \wtheta(t))^T\nabla^2 g_i(\wtheta(t))(\btheta - \wtheta(t))  = (\btheta - \wtheta(t))^T\nabla \mu_i (\wtheta(t)) \nabla \mu_i (\wtheta(t))^T(\btheta - \wtheta(t))$. Using this in \eqref{eq:mm_overview_2} along with the variational characterization of the minimum eigenvalue gives us the result. The full proof is given in Appendix \ref{app:thm:max-min-eig}.
\end{proof}

To illustrate the use of Theorem~\ref{thm:max-min-eig}, consider the problem of active learning in a hypothesis space of parametric functions $\{f_{\btheta}\, : \, \btheta\in \bTheta\}$.  The target function is $f_{\btheta^*}$ for an unknown $\btheta^* \in \bTheta$.  Assume that the learner may query the value of $f_{\btheta^*}$ at points $\bx_1,\dots,\bx_n$ in its domain.  If point $\bx_i$ is queried, then the learner observes the value $f_{\btheta^*}(\bx_i)$ plus a realization of a zero-mean sub-Gaussian noise.  This coincides with the setting above by setting $\mu_i(\btheta) := f_{\btheta}(\bx_i)$. 

\begin{algorithm}[H]
\caption{Chernoff Sampling for Active Regression}
\label{alg:cher_smooth}
\begin{algorithmic}[1]
\State{\textbf{Input:} Parametric model $\{\mu_i(\btheta): \btheta \in \mathbf{\Theta}, i \in [n] \}$.}
\State{Sample $I_1 \in [n]$ randomly, observe $Y_1$ and find $\widehat{\btheta}(1)$.}
\For{$t=2,3,\ldots$}
\State{Sample $I_{t} \sim \mathbf{p}_{\widehat{\btheta}(t-1)}$ (Theorem~\ref{thm:max-min-eig}), observe $Y_{t}$.}
\State{Compute $\wtheta(t) = \arg\min_{\btheta \in \bTheta} L_t(\btheta)$.}
\EndFor
\end{algorithmic}
\end{algorithm} 
Step~4 of Algorithm~\ref{alg:cher_smooth} requires solving the convex eigenvalue optimization in \Cref{thm:max-min-eig}, which takes $O((n^3 + n^2d^2 + nd^3)\sqrt{n+d})$ operations ignoring $\log$ factors \cite[Chap.\ 6]{nesterov1994interior}. 
The results of \citet{albert1961sequential} imply that for any $r>0$ the iterates in Algorithm~\ref{alg:cher_smooth} will converge to within $r$ of $\btheta^*$ using an optimal number of samples in the high confidence ($\delta\rightarrow 0$) regime. 
In Theorem~\ref{thm:cher-smooth} we obtain a finite time bound on the expected loss of $\wtheta(t)$. Define $\ell_s(\btheta) \colonequals (Y_s - \mu_{I_s}(\btheta))^2$ as the squared error for $\btheta$ at round $s$. The average empirical loss is defined as $\wP_t(\btheta) \colonequals \frac{1}{t} \sum_{s=1}^{t} \ell_{s}(\btheta)$.
\begin{assumption}
\label{assm:posdef2}
We assume that  $\lambda_{\max }\left(\nabla^{2} \mu_{i}\left(\btheta\right)\right) \leq \lambda_{1}$ for each $i\in[n]$ and all $\btheta\in\bTheta$. 
\end{assumption}
\Cref{assm:posdef2} is a mild assumption on the curvature of the mean function at any $\btheta \in \bTheta$. 
\begin{customtheorem}{4}\textbf{(Dense \cher Sample Complexity)}
\label{thm:cher-smooth}
Suppose $\ell_{1}(\btheta), \ell_{2}(\btheta), \cdots, \ell_{t}(\btheta): \mathbb{R}^{d} \rightarrow \mathbb{R}$ are squared loss functions from a distribution that satisfies \Cref{assm:posdef2} and \Cref{assm:thm} in  \Cref{app:cher-smooth-upper-assm}. Further define 
    $P_t(\btheta) = \frac{1}{t}\sum_{s=1}^t\E_{I_s\sim \mathbf{p}_{\wtheta_{s-1}}}[\ell_s(\btheta)|\F^{s-1}]$
where, $\wtheta_t =\argmin_{\btheta \in \bTheta} \sum_{s = 1}^t \ell_{s}(\btheta)$. If $t$ is large enough such that $ \frac{\gamma\log(dt)}{t}\leq c^{\prime} \min \left\{\frac{1}{C_{1}C_{2} }, \frac{\operatorname{diameter}(\mathcal{B})}{C_{2}}\right\}$
then for a constant $\gamma \geq 2$ and universal constants $C_1,C_2,c'$,  we show that 
\begin{align*}
&\left(1-\rho_{t}\right) \frac{\sigma_t^2}{t}- \frac{C_1^2}{t^{\gamma / 2}} \\
&\leq \E\left[P_t(\wtheta_t)-P_t\left(\btheta^{*}\right)\right] \\
&\leq \left(1+\rho_{t}\right) \frac{\sigma_t^2}{t}\!+\!\frac{\max\limits_{\btheta \in \bTheta}\left(\!P_{t}(\btheta)\!-\!P_{t}\left(\btheta^{*}\!\right)\right)}{t^{\gamma}},
\end{align*}
where 
$\sigma^{2}_t \coloneqq \E_{}\left[\frac{1}{2}\left\|\nabla \wP_{t}\left(\btheta^{*}\right)\right\|_{\left(\nabla^{2} P_t\left(\btheta^{*}\right)\right)^{-1}}^{2}\right]$, 
and $\rho_t \coloneqq \left(C_1C_2 + 2\eta^2\lambda_1^2\right)\sqrt{\frac{\gamma\log(dt)}{t}}$.
\end{customtheorem}

\begin{proof}\textbf{(sketch)} 
The first step in the proof is to relate $\nabla^2 \wP_t(\btheta)$ to $\nabla^2P_t(\btheta^*)$ for any $\btheta$ in a ball $\mathcal{B}$ around $\btheta^*$. The ball $\mathcal{B}$ is assumed in \Cref{assm:thm} to be a neighborhood where $\nabla^2 \ell_s(\btheta)$ satisfies a Lipschitz property. \Cref{assm:thm} in  \Cref{app:cher-smooth-upper-assm} are standard and have also been made by \citet{Frostig15} and \citet{chaudhuri2015convergence}. 
Using \Cref{assm:posdef2} and \Cref{assm:thm}, we can show that for $t$ as large as mentioned in the Theorem statement, (1) $\nabla^2 P_t(\btheta^*)$ is sandwiched in the positive semidefinite order by scaled multiples of $\nabla^2 \wP_t(\btheta)$ for any $\btheta \in \mathcal{B}$, and (2) the empirical error minimizing $\wtheta(t)$ is in the ball $\mathcal{B}$ with probability $1 - 1/t^\gamma$, which is the good event $\mathcal{E}$.   
Using a Taylor series expansion around $\wtheta(t)$ and the fact that $\nabla \wP_t(\wtheta(t)) = 0$ along with the relation between $\nabla^2 \wP_t(\btheta)$ and $\nabla^2 P_t(\btheta^*)$, we can obtain an upper bound to $\lVert \wtheta(t) - \btheta^*\rVert_{\nabla^2 P_t(\btheta^*)}$ in terms of $\lVert \nabla \wP_t(\btheta^*) \rVert_{(\nabla^2 P_t(\btheta^*))^{-1}}$ that can be shown to be decreasing with $t$. 
Further, $\lVert \wtheta(t) - \btheta^*\rVert_{\nabla^2 P_t(\btheta^*)}$ can also be used to obtain an upper bound to $P_t(\wtheta(t)) - P_t(\btheta^*)$ using a Taylor series expansion. 
Finally we can bound $\E[P_{t}(\wtheta_{t})-P_{t}(\btheta^{*})] =\E[(P_{t}(\wtheta_{t})-P_{t}(\btheta^{*})) I(\mathcal{E})]+\E[(P_{t}(\wtheta_{t})-P_{t}(\btheta^{*})) I(\mathcal{E}^\complement)]$ where $I(\cdot)$ is the indicator. Since $\Pb(\mathcal{E}^\complement) \leq 1/t^\gamma$, the second term can be bounded as $\max_{\btheta \in \bTheta}\left(P_{t}(\btheta)-P_{t}\left(\btheta^{*}\right)\right)/t^{\gamma}$, while the first term simplifies to $(1 + \rho_t)\sigma_t^2/t$. 
The full proof is in  \Cref{app:cher-smooth-upper}.
\end{proof}
The quantity $\mathbb{E}_{I_s \sim \mathbf{p}_{\btheta}}[\ell_s(\btheta)]$ characterizes the worst-case loss we could suffer at time $s$ due to estimation error. 
This is because
$\mathbb{E}[\ell_s(\btheta)] = 
\mathbb{E}_{I_s}\mathbb{E}[(Y_s - \mu_{I_s}(\btheta))^2 \mid I_s] = \mathbb{E}_{I_s}[(\mu_{I_s}(\btheta^*) - \mu_{I_s}(\btheta))^2 + \nicefrac{1}{2}]$,
and the definition of $\mathbf{p}_{\btheta}$ ensures that 
$
\sum_{i=1}^n \mathbf{p}_{\btheta}(i) (\mu_i(\btheta') - \mu_i(\btheta))^2
$
is maximized for the most confusing $\btheta' \not \in \mathcal{B}_{r}(\btheta)$ at small enough $\epsilon$. 
We contrast this with the average-case loss 
    $\mathbb{E}_{I_s \sim \text{Uniform}([n])} [(\mu_{I_s}(\btheta') - \mu_{I_s}(\btheta)^2 ], $
which is not larger than the expected value under $I_s \sim \mathbf{p}_{\btheta}$ due to the $\arg\max$ in \eqref{eq:albert-opt}. 
It can thus be seen that the Chernoff sampling allows us to bound a more stringent notion of risk than traditionally looked at in the literature. 
The theorem bounds the running average of the worst-case losses at each time step. 
The simplified bound of \Cref{thm:cher-smooth} scales as $O(d \sqrt{\log (d t)}/t+1/t^{2})$ (See  \Cref{tab:Active-regression0} 
in \Cref{app:cher-smooth-upper}).

The term $\sigma_t^2$ includes a norm under $(\nabla^2 P_t(\btheta^*))^{-1}$. It is bounded by a quantity proportional to the maximum eigenvalue of $(\nabla^2 P_t(\btheta^*))^{-1}$, equivalently the inverse of the minimum eigenvalue of $\nabla^2 P_t(\btheta^*) = \frac{2}{t} \sum_{s=1}^t \sum_{i=1}^n p_{\wtheta(s-1)}(i) \nabla \mu_i(\btheta^*) \nabla \mu_i(\btheta^*)^T$.
If $\wtheta(t) \approx \btheta^*$, which is true for large $t$, then $\mathbf{p}_{\wtheta(t)}$ is 
the optimization solution 
in \Cref{thm:max-min-eig} and it maximizes the minimum eigenvalue of $\sum_{i=1}^n p_{\wtheta(t)}(i) \nabla \mu_i(\wtheta(t)) \nabla \mu_i(\wtheta(t))^T$. Thus \cher approximately minimizes an upper bound to the estimation error. 


\section{Experiments}
\label{sec:expt}
In this section we show numerical evaluations of \cher against other algorithms. The confidence intervals (CI) that we plotted are just mean $\pm 1$ standard deviation. A wider CI means more variability in performance across different trials. We run each experiment over $50$ independent trials.

\begin{figure*}
\centering
\includegraphics[scale = 0.51]{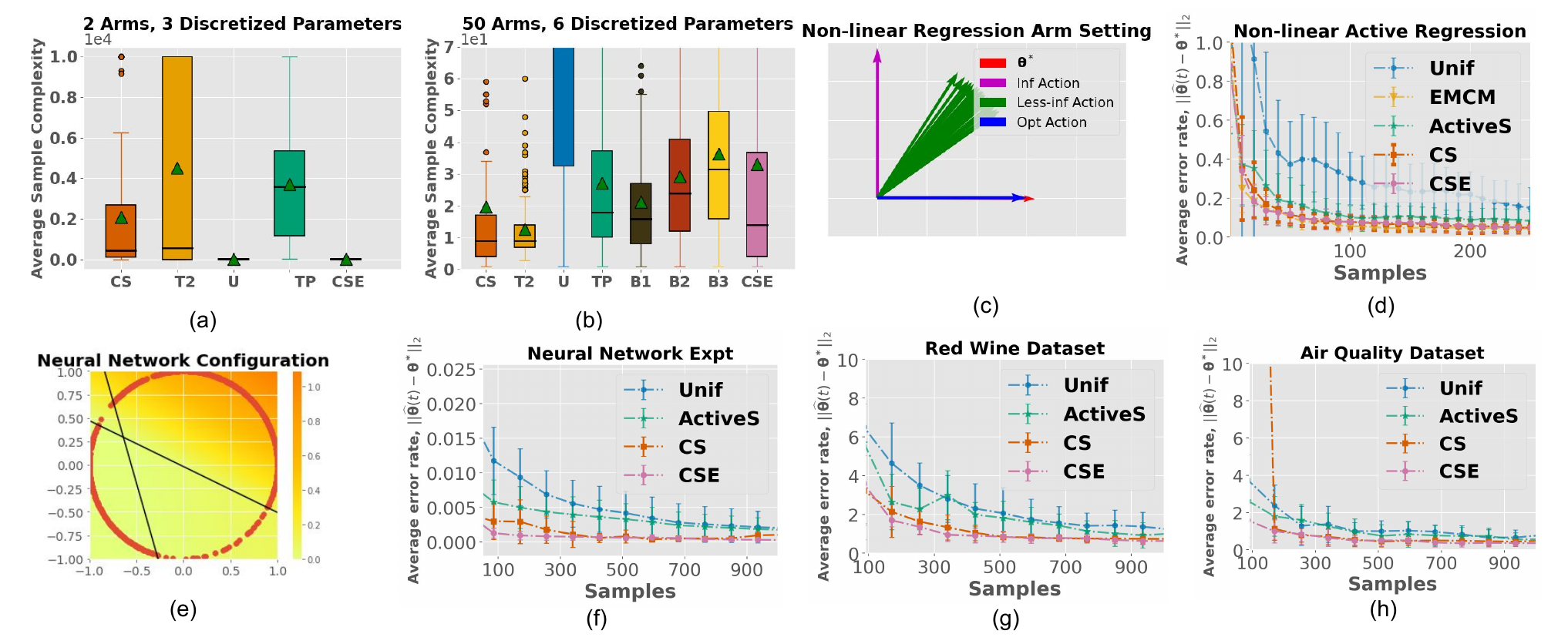}
\vspace*{-2.7em}
\begin{tabular}{ccc}
\subfloat{\label{fig:Figure-2enva}} &
\subfloat{\label{fig:Figure-2envb}} & \subfloat{\label{fig:nonlinerregression}}\\
\subfloat{\label{fig:1d}} &
\subfloat{\label{fig:nndata}} & \subfloat{\label{fig:nonlin}}\\
\subfloat{\label{fig:red-wine}} &
\subfloat{\label{fig:air-quality}} & \\
\end{tabular}
\caption{\textbf{(a)} Sample complexity over Example \ref{ex:unif-dominates} with $\delta = 0.1$. \textbf{(b)} Sample complexity over $3$ Group setting with $\delta = 0.1$. U = \unif, B1 = \bcher ($B=5$), B2 = \bcher ($B=10$), B3 = \bcher ($B=15$). The green triangle marker represents the mean stopping time. The horizontal line across each box represents the median. \textbf{(c)} The action feature vectors in $\mathbb{R}^2$ for active regression experiment. Inf = Informative action, Opt = Optimal action, Less-inf = Less informative action. \textbf{(d)} Shows average error rate for Active  Regression. \textbf{(e)} Training data for neural network example shown as red dots. Yellow to red colors indicate increasing values for the target mean function. Black lines denote ``activation'' boundaries of the two hidden layer neurons. \textbf{(f)} Learning a neural network model. 
\textbf{(g)}, \textbf{(h)} Experiment with Red Wine and Air Quality Dataset.}
\label{fig:FigureFinal}
\end{figure*}
\textbf{Active testing experiment ($3$ Group setting):} 
Consider an environment of $50$ actions and $6$ parameters. The actions can be divided into three groups. 
The first group contains a single action that most effectively discriminates $\btheta^\ast$ from all other hypotheses. The second group consists of $5$ actions, each of which can discriminate one hypothesis from the others. Finally, the third group of $44$ actions are barely informative as their means are similar under all hypotheses. The mean values under all hypotheses are given in Appendix \ref{app:testing-addl-expt}. We show the empirical performance of different policies in the  \Cref{fig:Figure-2envb}. Both \cher and \topllr outperforms \twoph which conducts a uniform exploration over the actions in the second group in its first phase. \topllr performs well as it samples the action in the first group when $\btheta^\ast$ is either the most-likely or the second most-likely hypothesis. \unif performs the worst as it uniformly samples all actions, including the non-informative actions in the third group. \cher outperforms \rcher as it does not conduct forced exploration over non-informative actions. 
\bcher with $B = 5, 10, 15$ has increasing sample complexity with larger batches.

\textbf{Non-Linear Model:} Consider a non-linear class of functions parameterized as $\mu_i(\btheta) \colonequals 1/(1+ \exp(-\mathbf{x}_i^T\btheta))$, where $\lVert \btheta\rVert_2 = 1$, each action has an associated feature vector $\mathbf{x}_i \in \mathbb{R}^2$ and it returns a value whose expectation is $\mu_i(\btheta^\ast)$. The goal is to choose actions such that the estimation error $\lVert \widehat{\btheta}(t) - \btheta^\ast \rVert_2$ reduces using as few samples as possible. The setting consist of three groups of actions: \textbf{a)} the optimal action, \textbf{b)} the informative action (orthogonal to optimal action) that maximally reduces the uncertainty of $\widehat{\btheta}(t)$ and \textbf{c)} the $48$ less-informative actions as shown in Figure \ref{fig:nonlinerregression}. Appendix~\ref{app:active-learning-expt-addl}  contains more implementation details. We apply \cher and \rcher to this problem and compare it to baselines \unif, and \actives. 
\Cref{fig:1d} shows that \cher outperforms \actives and is able to find $\btheta^*$ quickly. \cher performs similar to \emcm but note that \emcm has no convergence guarantees and requires hyper-parameter tuning. 

\textbf{Neural Network:} 
Consider a collection of data points $\{\mathbf{x}_i \in \mathbb{R}^2 : i \in [n]\}$, each of which is assigned a ground truth scalar mean value by a non-linear function. 
The particular form for the mean function of action~$i$ is the following:
    $\mu_i(\btheta^\ast) = c_1\sigma(\mathbf{w}_1^T \mathbf{x}_i + b_1) + c_2\sigma(\mathbf{w}_2^T \mathbf{x}_i + b_2)$,
where $\btheta^\ast = (\mathbf{w}_1, b_1, \mathbf{w}_2, b_2, c_1, c_2)$ is the parameter characterizing the ground truth function, 
and $\sigma(\cdot):=\max\{0,\cdot\}$ is the non-linear ReLU activation function. This is a single hidden-layer neural network with input layer weights $\mathbf{w}_1, \mathbf{w}_2 \in \mathbb{R}^2$, biases $b_1, b_2 \in \mathbb{R}$, and output weights $c_1, c_2 \in \{-1, 1\}$. 
Our objective in the experiment is to learn the neural network from noisy observations:  
$\{(\mathbf{x}_{I_s}, \mu_{I_s}(\btheta^\ast) + \text{noise}): s \in [t], \text{noise is i.i.d.\ zero mean Gaussian}\}$ 
collected by sampling $(I_1, I_2, \ldots, I_t)$ according to the Chernoff proportions defined in \Cref{alg:cher_smooth}. The architecture of the network is known, but the weights and biases must be learned. 
The data points are shown in a scatter plot in \Cref{fig:nndata}. More implementation details are in \Cref{app:active-learning-expt-NN-addl}. 
The non-uniformity of the data distribution increases the difficulty of the learning task. The performance of a learning algorithm is measured by tracking the estimation error $\|\widehat{\btheta}(t) - \btheta^\ast\|_2$ during the course of training. The plot in \Cref{fig:nonlin} shows the average and standard deviation of the estimation error over $10$ trials for \cher, \rcher, \actives and \unif baseline. Again both \cher and \rcher outperforms \unif and \actives. 
We note that other works \citep{
cohn1996neural, fukumizu2000statistical} have also considered active training of neural networks. 
Other approaches for active sampling have been described in the survey by \citet{settles.tr09}. 
Our neural network learning experiment shows the generality of our approach in active sampling.

\textbf{Real Dataset:} 
We consider two real world datasets from UCI called Red Wine \citep{cortez2009modeling} ($1600$ actions) and Air Quality \citep{de2008field} ($1500$ actions). 
The performance is shown in \Cref{fig:red-wine} and \Cref{fig:air-quality} respectively where \cher and \rcher outperforms \actives and \unif. 
Further experiment details are in \Cref{app:real-dataset-expt},
including \Cref{fig:red-wine-prop} which shows that in the real-world dataset \cher proportion is sparse over the actions. 

\section{Conclusions, Limitations, \& Future Work}
\label{sec:conc}
This paper proposes a unifying approach to solve active testing and active regression problems. We obtain non-asymptotic guarantees on the performance of \cher in active testing and extend it to the problem of active regression. 
\cher has comparable performance to existing state-of-the-art methods and is a relatively easy algorithm to implement.
Nevertheless, its sampling proportion is updated before collecting each sample, which increases the computational cost (one solution for this is \bcher which solves the optimization only after collecting a batch of samples). Further, \Cref{assm:positive-D1} is a strong assumption, and other works have removed that in the context of active testing by modifying the \cher strategy (one solution for this is \rcher which incorporates a certain amount of random sampling). 

\cher can be excessively aggressive in the initial stages. This is because it chooses actions according to a sampling proportion that is optimal when $\wtheta(t)=\btheta^*$, which is not true initially. Other methods of exploration could be useful in the earlier stages. 

The extension of \cher to active regression requires to find the least squares estimate $\wtheta(t)$ which could be computationally expensive. 
Theoretical guarantees require several assumptions which may not always be satisfied. In addition to the regularity assumptions, we also need the smoothness of the mean function for \Cref{thm:cher-smooth}. Future directions include obtaining a lower bound for the active testing in the moderate confidence regime and incorporating the geometry of the actions in the sampling strategy for the regression setting. Another direction is to obtain the sampling proportions when the mean function is not differentiable everywhere.



\newpage

\textbf{Acknowledgements: } This work was partially supported by AFOSR grant FA9550-18-1-0166. The first author was supported by 2019-20 Chancellor's Opportunity Fellowship by the University of Wisconsin-Madison.

\bibliographystyle{apalike}
\bibliography{biblio1}

\newpage
\appendix
\onecolumn

\section*{Contents}
\begin{enumerate}[label=\arabic*,leftmargin=*,labelsep=2ex,ref=\arabic*]
    \item Theoretical Comparison and Probability Tools \dotfill \ref{app:theory-tools}
      \begin{enumerate}[label*=.\arabic*,leftmargin=*,labelsep=2ex]
        \item Theoretical Comparison \dotfill \ref{app:theory}
        \item Probability Tools \dotfill \ref{app:probability-tools}
      \end{enumerate}
    \item Chernoff Sample Complexity Proof
\dotfill \ref{app:chernoff-sample-comp}
    \begin{enumerate}[label*=.\arabic*,leftmargin=*,labelsep=2ex]
        \item Concentration Lemma  \dotfill \ref{app:cher-conc-lemma}
        \item Concentration of $T_{Y^t}$ \dotfill \ref{app:cher-conc-tyt}
    \item Proof of correctness for General Sub-Gaussian Case \dotfill
\ref{app:cher-proof-correct-SG}
            \item Stopping time Correctness Lemma for the Gaussian Case \dotfill \ref{app:cher-proof-correct-G}
            \item Proof of \cher Sample Complexity (\Cref{thm:chernoff-upper}) \dotfill \ref{app:thm:chernoff-upper}
        \end{enumerate}
    \item Proof of \topllr Sample Complexity \dotfill \ref{app:thm:topllr-upper}
        \begin{enumerate}[label*=.\arabic*,leftmargin=*,labelsep=2ex]
            \item Concentration Lemma \dotfill \ref{app:topllr-conc}
            \item Proof of \topllr Sample Complexity (\Cref{thm:topllr-upper}) \dotfill \ref{app:topllr-upper}
        \end{enumerate}
    \item Proof of Sample Complexity of \bcher (\Cref{prop:batch-cher}) \dotfill \ref{app:prop:batch-cher}
    \item Proof of \rcher Sample Complexity (\Cref{prop:cher-e}) \dotfill \ref{app:rcher-upper-bound}
    \item Minimax Optimality Proof (Theorem \ref{thm:minimax})\dotfill \ref{app:minimax}
    \item Proof of Theorem \ref{thm:max-min-eig} (Continuous hypotheses) \dotfill \ref{app:thm:max-min-eig}
    \begin{enumerate}[label*=.\arabic*,leftmargin=*,labelsep=2ex]
        \item How to solve the optimization \dotfill \ref{app:optimization-todo}
    \end{enumerate}
    \item \cher Proof for Continuous Hypotheses \dotfill \ref{app:cher-smooth-upper}
    \begin{enumerate}[label*=.\arabic*,leftmargin=*,labelsep=2ex]
        \item Theoretical Comparisons for Active Regression \dotfill\ref{app:active-regression-discussion}
        \item Discussion on Definitions and Assumptions for Continuous Hypotheses \dotfill \ref{app:cher-smooth-upper-assm}
        \item Concentration Lemma for Continuous Hypotheses \dotfill \ref{app:cher-conc-tyt1}
        \item Support Lemma for Continuous Hypotheses\dotfill \ref{app:smooth-support-lemma}
        \item Proof of \cher Convergence for Continuous Hypotheses (\Cref{thm:cher-smooth}) \dotfill \ref{app:thm:cher-smooth}
    \end{enumerate}
    \item Additional Experiment Details \dotfill \ref{app:expt}
    \begin{enumerate}[label*=.\arabic*,leftmargin=*,labelsep=2ex]
        \item Hypothesis Testing Experiments \dotfill \ref{app:testing-addl-expt}
        \item Active Regression Experiment for Non-linear Reward Model \dotfill \ref{app:active-learning-expt-addl}
        \item Active Regression Experiment for Neural Networks \dotfill \ref{app:active-learning-expt-NN-addl}
        \item Active Regression for the UCI Datasets
        \dotfill \ref{app:real-dataset-expt}
    \end{enumerate}
    \item Table of Notations \dotfill\ref{app:notation}
\end{enumerate}

\section{Appendix}

\subsection{Theoretical Comparison of Active Testing and Probability Tools}
\label{app:theory-tools}

In this section we compare theoretically our work against other existing works in active testing. We also state a few standard lemmas in Probability Tools that we use to prove our results. 

\subsubsection{Theoretical Comparison of Active Testing}
\label{app:theory}

\textbf{Active Testing: } We compare our work against \citet{chernoff1959sequential}, \citet{albert1961sequential} which extended  \citet{chernoff1959sequential} to continuous hypotheses space partitioned into two disjoint space, and the recent active testing algorithms of \citet{naghshvar2013active,nitinawarat2013controlled}. We first recall the following problem complexity parameters as follows:
\begin{align}
    D_0 &\colonequals \max_{\text{ }\mathbf{p}} \min_{\btheta' \neq \btheta^\ast} \sum_{i=1}^n p(i) (\mu_i(\btheta') -\mu_i(\btheta^\ast))^2, \qquad
    D_1 \colonequals \min_{\{\mathbf{p}_{\btheta} : \btheta \in \bTheta\}} \min_{\btheta' \neq \btheta^\ast} \sum_{i=1}^n p_{\btheta}(i)(\mu_i(\btheta') - \mu_i(\btheta^*))^2\nonumber\\
    D_{\text{NJ}} \!&\!\colonequals \left(\min_{\btheta\in \bTheta} \min_{\btheta' \neq \btheta} \sum_{i = 1}^n p_{\btheta'}(i) (\mu_i(\btheta) \!-\! \mu_i(\btheta'))^2\right)^2 \max_{\mathbf{p}}\min_{\btheta\in \bTheta} \min_{\btheta' \neq \btheta} \sum_{i = 1}^n p_{}(i) (\mu_i(\btheta) \!-\! \mu_i(\btheta'))^2\label{eq:D1-prime}\\
    D_e &\colonequals   \min_{\btheta' \neq \btheta^\ast} \sum_{i=1}^n \frac{1}{n}(\mu_i(\btheta') - \mu_i(\btheta^*))^2 \label{eq:De}
\end{align}
where, the quantity $D_{NJ}$ is defined in \citet{naghshvar2013active} and is an artifact of the forced exploration conducted by their algorithm. Note that $D_{NJ} < D_1$ by definition. Also, $D_0 > D_e$ and $D_0 > D_1$ and $D_e > 0$ by definition.
Note that due to the factor $D_{NJ}$, \twoph can perform worse than \cher in certain instances (see Active testing experiment in \Cref{fig:Figure-2enva}, and \Cref{fig:Figure-2envb}). 
We summarize out result in context of other existing results in the following table:
\begin{table}[ht]
    \centering
\begin{tabular}{|p{23.5em}|p{14em}|}
\hline
 \textbf{Sample Complexity Bound} & \textbf{Comments}\\
\hline \hline
    $\E[\tau_\delta] \leq \frac{C\log J/\delta}{D_0}  + o\left(\log\frac{1}{\delta}\right)$ & Upper bound in  \citet{chernoff1959sequential}.Optimal for $\delta\rightarrow 0$.\\ 
    $\E[\tau_\delta] \leq \frac{C\log J/\delta}{D_0}  + o\left(\log\frac{1}{\delta}\right)$ & Upper bound in \citet{albert1961sequential}. Optimal for $\delta\rightarrow 0$. Extension to compound hypotheses.\\
    $\E[\tau_\delta] \leq \frac{C\log J/\delta}{D_0}  + o\left(\log\frac{1}{\delta}\right)$ & Upper bound in \citep{nitinawarat2013controlled}. Bound valid for discrete hypotheses. It does not require \Cref{assm:positive-D1}.\\
   $\E[\tau_\delta] \leq O\left( \frac{\log(J)}{D_{\text{NJ}}} + \frac{\log(J/\delta)}{D_0} \right)$ & Upper bound of \twoph in \citet{naghshvar2013active}. Valid for any $\delta\in(0,1]$. Asymptotically optimal for $\delta\rightarrow 0$. \\ 
    $\E[\tau_\delta] \leq O\left(\frac{\log(J)}{D_1} + \frac{\log(J/\delta)}{D_0}\right)$ & \cher \textcolor{red}{(Ours)}. Valid for any $\delta\in(0,1]$. Asymptotically optimal for $\delta\!\rightarrow\! 0$. 
    \\
     $\E[\tau_\delta] \leq O\left(\frac{\log(J)}{D_e} + \frac{\log(J/\delta)}{D_0}\right)$ & \rcher \textcolor{red}{(Ours)}. Valid for any $\delta\in(0,1]$. Asymptotically optimal for $\delta\rightarrow 0$. Does not require \Cref{assm:positive-D1}.\\
    \hline
    \end{tabular}
\caption{Active Testing (top) and Regression (bottom) comparison. $D_{\text{NJ}}< D_1$, $D_1 < D_0$, and $D_e < D_0$. }
    \label{tab:Active-Testing}
\end{table}

\subsubsection{Probability Tools}
\label{app:probability-tools}
\begin{lemma}
\label{lemma:divergence-decomp}
\textbf{(Restatement of Lemma 15.1 in \citet{lattimore2018bandit}, Divergence Decomposition)}
Let $B$ and $B'$ be two bandit models having different optimal hypothesis $\btheta^*$ and $\btheta^{'*}$ respectively. Fix some policy \(\pi\) and round $n$. Let \(\Pb_{B, \pi}\) and \(\Pb_{B', \pi}\) be two probability measures induced by some $n$-round interaction of \(\pi\) with \(B\) and \(\pi\) with
\(B'\) respectively. Then
\begin{align*}
    \KL\left(\Pb_{B, \pi}|| \Pb_{B', \pi}\right)=  \sum_{i=1}^{n} \E_{B, \pi}[Z_i(n)]\cdot\KL(\mu_{i}(\btheta)||\mu_{i}(\btheta^*)) 
\end{align*}
where, $\KL\left(.||.\right)$ denotes the Kullback-Leibler divergence between two probability measures and $Z_i(n)$ denotes the number of times action $i$ has been sampled till round $n$.
\end{lemma}

\begin{lemma}
\label{lemma:tsybakov}
\textbf{(Restatement of Lemma 2.6 in \citet{tsybakov2008introduction})}
Let \(\mathbb{P}, \mathbb{Q}\) be two probability measures on the same measurable space \((\Omega, \F)\) and let \(\xi \subset \F \) be any arbitrary event then
\[
\mathbb{P}(\xi)+ \mathbb{Q}\left(\xi^{\complement}\right) \geqslant \frac{1}{2} \exp\left(-\KL(\mathbb{P}|| \mathbb{Q})\right)
\]
where \(\xi^{\complement}\) denotes the complement of event \(\xi\) and \(\KL(\mathbb{P}||\mathbb{Q})\) denotes the Kullback-Leibler divergence between \(\mathbb{P}\) and \(\mathbb{Q}\).
\end{lemma}

\begin{lemma}\textbf{(Hoeffding's Lemma)}
\label{lemma:hoeffding}
Let $Y$ be a real-valued random variable with expected value $\mathbb{E}[Y]= \mu$, such that $a \leq Y \leq b$ with probability one. Then, for all $\lambda \in \mathbb{R}$
$$
\mathbb{E}\left[e^{\lambda Y}\right] \leq \exp \left(\lambda \mu +\frac{\lambda^{2}(b-a)^{2}}{8}\right)
$$
\end{lemma}

\begin{lemma}\textbf{(Proposition 2 of \citep{hsu2012tail})}
\label{lemma:vector-martingale}
Let $\mathbf{u}_{1}, \ldots, \mathbf{u}_{n}$ be a martingale difference vector sequence (i.e., $\mathbb{E}\left[\mathbf{u}_{i} \mid \mathbf{u}_{1}, \ldots, \mathbf{u}_{i-1}\right]=$ 0 for all $i=1, \ldots, n$ ) such that
$$
\sum_{i=1}^{n} \mathbb{E}\left[\left\|\mathbf{u}_{i}\right\|^{2} \mid \mathbf{u}_{1}, \ldots, \mathbf{u}_{i-1}\right] \leq v \quad \text { and } \quad\left\|\mathbf{u}_{i}\right\| \leq b
$$
for all $i=1, \ldots, n,$ almost surely. For all $t>0$
$$
\operatorname{Pr}\left[\left\|\sum_{i=1}^{n} \mathbf{u}_{i}\right\|>\sqrt{v}+\sqrt{8 v t}+(4 / 3) b t\right] \leq e^{-t}
$$
\end{lemma}

\subsection{Chernoff Sample Complexity Proof}
\label{app:chernoff-sample-comp}

\subsubsection{Concentration Lemma}
\label{app:cher-conc-lemma}
\begin{lemma}
\label{conc:lemma:1}
Define $L_{t}(\btheta^*)$ as the sum squared errors for the hypothesis parameterized by  $\btheta^*$. Let $\tau_{\btheta^*} = \min\{t: L_{t}(\btheta') - L_{t}(\btheta^*) > \beta(J,\delta), \forall \btheta' \neq \btheta^* \}$. Then we can bound the probability that $\tau_{\btheta^*}$ is larger than $t$ as
\begin{align*}
     \Pb(\tau_{\btheta^*} > t) \leq J C_1\exp\left(-C_2 t\right)
\end{align*}
where, $J \colonequals |\bTheta|$, $C_1 \colonequals 110 + 55\max\left\{1, \dfrac{\eta^2}{ 2D_1^2c }\right\}$, $C_2 \colonequals \dfrac{ 2D_1^2\min\{(c-1)^2, c\}}{\eta^2}$, $\eta >0$ defined in Definition \ref{def:bounded} and $D_1 \colonequals  \min\limits_{\btheta \in \bTheta, \btheta' \neq \btheta^*}\sum_{i=1}^n p_{\btheta}(i)(\mu_i(\btheta') - \mu_i(\btheta^*))^2$.
\end{lemma}

\begin{proof}
We consider the following events when the difference of squared errors is below certain values:
\begin{align*}
    \xi_{\btheta'\btheta^*}(t) &\colonequals \{L_{t}(\btheta') - L_{t}(\btheta^*) < \beta(J,\delta)\},\\
    \tilde{\xi}_{\btheta'\btheta^*}(t) &\colonequals \{L_{t}(\btheta') - L_{t}(\btheta^*) < \alpha(J)\}.
\end{align*}
Then we define the time $\tau_{\btheta^*}$ as follows:
\begin{align*}
    \tau_{\btheta^*} &\colonequals \min\{t: L_{t}(\btheta') - L_{t}(\btheta^*) > \beta(J,\delta),  \forall \btheta' \neq \btheta^* \}
\end{align*}
which is the first round when  $L_{t}(\btheta')$ crosses $\beta(J,\delta)$ threshold against $L_{t}(\btheta^*)$ for all $\btheta'\neq \btheta^*$. 
We also define the time $\tilde{\tau}_{\btheta'\btheta^*}$ as follows:
\begin{align}
    \tilde{\tau}_{\btheta'\btheta^*} &\colonequals \min\{t: L_{t'}(\btheta') - L_{t'}(\btheta^*) > \alpha(J), \forall t' > t\}\label{eq:tau-theta-thetap}
\end{align}
which is the first round when $L_{t}(\btheta')$ crosses $\alpha(J)$ threshold against $L_{t}(\btheta^*)$. We will be particularly interested in the time
\begin{align}
    \tilde{\tau}_{\btheta^*} \colonequals \max_{\btheta' \neq \btheta^*}\{\tilde{\tau}_{\btheta'\btheta^*}\}. \label{eq:tau-tilde}
\end{align}
Let $\Delta_{t}(\btheta')$ 
denote the difference of squared errors between hypotheses $\btheta'$ and $\btheta^*$ (for Gaussian noise model, it is equal to the log-likelihood ratio between hypotheses parameterized by $\btheta'$ and $\btheta^\ast$) shown below.
\begin{align}
\Delta_{t}(\btheta') 
= L_{t}(\btheta') - L_{t}(\btheta^*). \label{eq:log-likelihood-ratio}
\end{align}


A key thing to note that for $D_1 > 0$ (Assumption \ref{assm:positive-D1}) the $\E[\Delta_{t}(\btheta')] > tD_1$ which is shown as follows:
\begin{align*}
    \E_{I^t,Y^t}[\Delta_{t}(\btheta')] = \E_{I^t,Y^t}[L_{t}(\btheta') - L_{t}(\btheta^*)] &= \E_{I^t,Y^t}\left[\sum_{s=1}^t(Y_s - \mu_{I_s}(\btheta^*))^2 - \sum_{s=1}^t(Y_s - \mu_{I_s}(\btheta))^2\right]\\
    & = \sum_{s=1}^t\E_{I_s}\E_{Y_s|I_s}\left[\left(\mu_{I_s}(\btheta^*) - \mu_{I_s}(\btheta)\right)^2|I_s\right]\\
    &= \sum_{s=1}^t \sum_{i=1}^n\Pb(I_s = i)\left(\mu_{i}(\btheta^*) - \mu_{i}(\btheta)\right)^2
    \overset{(a)}{\geq} tD_1
\end{align*}
where, $(a)$ follows from the definition of $D_1$ in Theorem \ref{thm:chernoff-upper}.
Then it follows that,
\begin{align*}
    \Pb(\tilde{\xi}_{\btheta'\btheta^*}(t)) = \Pb(\Delta_{t}(\btheta')  < \alpha(J)) &\overset{}{=} \Pb(\Delta_{t}(\btheta') - \E[\Delta_{t}(\btheta')] < \alpha(J) -\E[\Delta_{t}(\btheta')])\\
     &\overset{(a)}{\leq} \Pb(\Delta_{t}(\btheta') - \E[\Delta_{t}(\btheta')] < \alpha(J) - t D_1)
\end{align*}
where, in $(a)$ the choice of $tD_1$ follows as $\E[\Delta_{t}(\btheta')] \geq tD_1$ for $D_1 > 0$.
Similarly, we can show that $\E[\Delta_{s}(\btheta')] \geq D_0$ for all rounds $s \geq\tilde{\tau}_{\theta^*}$ where $\tilde{\tau}_{\btheta^*}$ is defined in \eqref{eq:tau-tilde}.
Then we can show that,
\begin{align*}
    \Pb({\xi}_{\btheta',\btheta^*}(t) { \mid \tilde{\tau}_{\theta^*}}) 
    &= \Pb(\Delta_{t}(\btheta') < \beta(J,\delta) { \mid \tilde{\tau}_{\theta^*}})
    \overset{}{=} \Pb(\Delta_{t}(\btheta') - \E[\Delta_{t}(\btheta')] < \beta(J,\delta) - \E[\Delta_{t}(\btheta')] { \mid \tilde{\tau}_{\theta^*}})\\
    &\overset{(a)}{\leq} \Pb\left(\Delta_{t}(\btheta') - \E[\Delta_{t}(\btheta')] < \beta(J,\delta) - (t- \tilde{\tau}_{\btheta^*})D_0 { \mid \tilde{\tau}_{\theta^*}}\right)\\
    &\overset{(b)}{=} \Pb\left(\Delta_{t}(\btheta') - \E[\Delta_{t}(\btheta')] < D_0\left(\dfrac{ \beta(J,\delta)}{D_0} - t + \tilde{\tau}_{\btheta^*})\right) { \mid \tilde{\tau}_{\theta^*}}\right)
\end{align*}
where, in $(a)$ the choice of $(t-\tilde{\tau}_{{\btheta^*}})D_0$ follows as $\E[\Delta_{s}(\btheta')]\geq D_0$ for any round $s > \tilde{\tau}_{\btheta^*}$, and finally in $(b)$ the quantity $D_0\left(\dfrac{ \beta(J,\delta)}{D_0} - t + \tilde{\tau}_{\btheta^*})\right)$ is negative for $t \geq (1+c)\left(\dfrac{\alpha(J)}{D_1} + \dfrac{\beta(J,\delta)}{D_0}\right)$ and $\tilde{\tau}_{\btheta^*} < \dfrac{\alpha(J)}{D_1} + \dfrac{tc}{2}$ which allows us to apply the concentration inequality for conditionally independent random variables stated in \Cref{lemma:cond-Mcdiarmid}. 
Then we can show that,
\begin{align*}
    &\Pb(\tau_{\btheta^*} > t) \leq \Pb(\bigcup_{\btheta'\neq \btheta^*} \xi_{\btheta'\btheta^*}(t)) \\
    &= \Pb\left(\left\{\bigcup_{\btheta'\neq \btheta^*}\xi_{\btheta'\btheta^*}(t)\right\}\bigcap\{\tilde{\tau}_{\btheta^*} < \dfrac{\alpha(J)}{D_1} + \frac{tc}{2}\}\right) + \Pb\left(\bigcup_{\btheta'\neq \btheta^*}\left\{\xi_{\btheta'\btheta^*}(t)\right\}\bigcap\{\tilde{\tau}_{\btheta^*} \geq \dfrac{\alpha(J)}{D_1} + \frac{tc}{2}\}\right) \\
    &
    \leq \sum_{\btheta'\neq \btheta^*}\Pb\left(\left\{\xi_{\btheta'\btheta^*}(t)\right\}\bigcap\{\tilde{\tau}_{\btheta^*} < \dfrac{\alpha(J)}{D_1} + \frac{tc}{2}\}\right) 
    + \sum_{\btheta'\neq \btheta^*}\Pb\left(\tilde{\tau}_{\btheta^*} \geq \dfrac{\alpha(J)}{D_1} + \frac{tc}{2}\right)
    \\
    &
    = \sum_{\btheta'\neq \btheta^*}\Pb\left(\xi_{\btheta'\btheta^*}(t)\mid \tilde{\tau}_{\btheta^*} < \dfrac{\alpha(J)}{D_1} + \frac{tc}{2}\right) \Pb\left(\tilde{\tau}_{\btheta^*} < \dfrac{\alpha(J)}{D_1} + \frac{tc}{2}\right) \\
    &\qquad + \sum_{\btheta'\neq \btheta^*}\Pb\left( \tilde{\xi}_{\btheta'\btheta^*}(t') \text{ is true for some } t' > \dfrac{\alpha(J)}{D_1} + \frac{tc}{2}\right)
    \\
    & 
    \leq \sum_{\btheta'\neq \btheta^*}\Pb\left(\xi_{\btheta'\btheta^*}(t)\mid \tilde{\tau}_{\btheta^*} < \dfrac{\alpha(J)}{D_1} + \frac{tc}{2}\right) 
    + \sum_{\btheta'\neq \btheta^*}\sum_{t':t' \geq \frac{\alpha(J)}{D_1} + \frac{tc}{2} }\Pb\left(\tilde{\xi}_{\btheta'\btheta^*}(t')\right)
    \\
    &
    \leq \sum_{\btheta'\neq \btheta^*}
    \Pb\left(\Delta_{t}(\btheta') - \E[\Delta_{t}(\btheta')] < D_0\left(\dfrac{ \beta(J,\delta)}{D_0} - t + \tilde{\tau}_{\btheta^*} \right) \mid \tilde{\tau}_{\theta^*}< \dfrac{\alpha(J)}{D_1} + \frac{tc}{2} \right) \\
    &\qquad + \sum_{\btheta'\neq \btheta^*}\sum_{t':t' \geq \frac{\alpha(J)}{D_1} + \frac{tc}{2} }\Pb\left(\Delta_{t'}(\btheta') - \E[\Delta_{t'}(\btheta')] < \alpha(J) - t'D_1\right)
\end{align*}
\begin{align*}
    &\leq \sum_{\btheta'\neq \btheta^*} \Pb\left(\Delta_{t}(\btheta') - \E[\Delta_{t}(\btheta')] < D_0\left(\dfrac{\beta(J,\delta)}{ D_0} + \dfrac{\alpha(J)}{D_1} - t + \frac{tc}{2} \right)\right)\\
    &\quad + \sum_{\btheta'\neq \btheta^*}\sum_{t':t' \geq \frac{\alpha(J)}{D_1} + \frac{tc}{2} }\Pb\left(\Delta_{t'}(\btheta') - \E[\Delta_{t'}(\btheta')] < \alpha(J) - t'D_1\right)\\
    &\overset{(a)}{\leq} \exp\left(4\right)\sum_{\btheta'\neq \btheta^*}\exp\left(-\dfrac{2D_1^2t(\nicefrac{c}{2} - 1)^2 }{\eta^2}\right) + \exp\left(4\right)\sum_{\btheta'\neq \btheta^*}\dfrac{\exp\left(-\dfrac{2 D_1^2  tc}{\eta^2}\right)}{1 - \exp\left(-\dfrac{ 2D_1^2 }{\eta^2}\right)}  \\
    &\overset{(b)}{\leq} \exp\left(4\right)\sum_{\btheta'\neq \btheta^*}\exp\left(-\dfrac{2D_1^2t(\frac{c}{2} - 1)^2 }{\eta^2}\right) + \exp\left(4\right)\sum_{\btheta'\neq \btheta^*}\exp\left(-\dfrac{2 D_1^2  tc}{\eta^2}\right)\left(1 + \max\left\{1, \dfrac{\eta^2}{2 D_1^2  }\right\}\right)\\
    &\overset{(c)}{\leq} 55\sum_{\btheta'\neq \btheta^*}\exp\left(-\dfrac{2 D_1^2t(\frac{c}{2} - 1)^2 }{\eta^2}\right) + 55\sum_{\btheta'\neq \btheta^*}\exp\left(-\dfrac{2 D_1^2  tc}{\eta^2}\right) + 55\sum_{\btheta'\neq \btheta^*}\exp\left(-\dfrac{2 D_1^2  tc}{\eta^2}\right) \max\left\{1, \dfrac{\eta^2}{2 D_1^2 }\right\}\\
    &\overset{}{\leq} \sum_{\btheta'\neq \btheta^*}\exp\left(-\dfrac{2D_1^2t\min\{(\frac{c}{2} - 1)^2, c\} }{\eta^2}\right)\left[110 + 55 \max\left\{1, \dfrac{\eta^2}{2D_1^2 }\right\}  \right]  \overset{(d)}{\leq} J C_1\exp\left(-C_2t\right)
\end{align*}
where, $(a)$ follows from Lemma \ref{lemma:supp:1} and Lemma  \ref{lemma:supp:2}, $(b)$ follows from the identity that $1/(1-\exp(-x)) \leq 1 + \max\{1,1/x\}$ for $x >0$, $(c)$ follows for $0 < c < 1$, 
and in $(d)$ we substitute $C_1 \colonequals 110 + 55\max\left\{1, \dfrac{\eta^2}{2 D_1^2 }\right\}$, $C_2 \colonequals \dfrac{ 2D_1^2\min\{(\nicefrac{c}{2}-1)^2, c\}}{\eta^2}$ and $J \colonequals |\bTheta|$.
\end{proof}

\begin{lemma}
\label{lemma:supp:1}
Let $\Delta_{t}(\btheta') \colonequals  L_{t}(\btheta') - L_{t}(\btheta^*)$ from \eqref{eq:log-likelihood-ratio}, $D_1 \colonequals  \min\limits_{\btheta \in \bTheta, \btheta' \neq \btheta^*}\sum_{i=1}^n p_{\btheta}(i)(\mu_i(\btheta') - \mu_i(\btheta^*))^2$, 
and $\alpha(J)$ and $\beta(J,\delta)$ be the two thresholds. Then we can show that
\begin{align*}
    \sum_{\btheta'\neq \btheta^*} &\Pb\left(\Delta_{t}(\btheta') - \E[\Delta_{t}(\btheta')] < D_0\left(\dfrac{\beta(J,\delta)}{D_0} + \dfrac{\alpha(J)}{D_1} - t + \dfrac{tc}{2} \right)\right) \leq \exp\left(4\right)\sum_{\btheta'\neq \btheta^*}\exp\left(-\dfrac{2D_1^2t(\nicefrac{c}{2} - 1)^2 }{\eta^2}\right).
\end{align*}
for some constant $c$ such that $ 0 < c < 1$.
\end{lemma}
\begin{proof}
Let us recall that the critical number of samples is given by $(1+c)M$ where
\begin{align}
    M \colonequals \dfrac{\beta(J,\delta)}{D_0} + \dfrac{\alpha(J)}{D_1}.
\end{align}
and $c$ is a constant. Then we can show that for some $0 < c < 1$,
\begin{align}
    \sum_{\btheta'\neq \btheta^*} &\Pb\left(\Delta_{t}(\btheta') - \E[\Delta_{t}(\btheta')] < D_0\left(\dfrac{\beta(J,\delta)}{D_0} + \dfrac{\alpha(J)}{D_1} - t + \dfrac{tc}{2} \right)\right)\nonumber\\
    &\overset{(a)}{\leq} \sum_{\btheta'\neq \btheta^*}\exp\left(-\dfrac{ 2D_0^2\left(\dfrac{\beta(J,\delta)}{D_0} + \dfrac{\alpha(J)}{D_1} + t(\nicefrac{c}{2} - 1) \right)^2}{t\eta^2}\right)
    \overset{(b)}{\leq} \sum_{\btheta'\neq \btheta^*}\exp\left(-\dfrac{2 D_1^2\left(M + t(\nicefrac{c}{2} - 1) \right)^2}{t\eta^2}\right)\nonumber\\
    & \overset{(c)}{\leq} \sum_{\btheta'\neq \btheta^*}\exp\left(-\dfrac{2D_1^2t^2(\nicefrac{c}{2} - 1)^2 + 4D_1^2tM(\nicefrac{c}{2} - 1) }{t\eta^2}\right) \overset{(d)}{\leq} \sum_{\btheta'\neq \btheta^*}\exp\left(-\dfrac{2D_1^2t^2(\nicefrac{c}{2} - 1)^2 + 4\eta_0^2 M t(\nicefrac{c}{2} - 1) }{t\eta^2}\right)\nonumber\\
    &\overset{(e)}{\leq} \sum_{\btheta'\neq \btheta^*}\exp\left(-\dfrac{2D_1^2t^2(\nicefrac{c}{2} - 1)^2 + 4\eta_0^2 t(\nicefrac{c}{2} - 1) }{t\eta^2}\right)
    \overset{(f)}{=} \sum_{\btheta'\neq \btheta^*}\exp\left(\dfrac{4\eta_0^2(1 - \nicefrac{c}{2})}{\eta^2}\right)\exp\left(-\dfrac{2 D_1^2t(\nicefrac{c}{2} - 1)^2 }{\eta^2}\right) \nonumber\\
    &\overset{(g)}{\leq} \sum_{\btheta'\neq \btheta^*}\exp\left(4(1 - \nicefrac{c}{2})\right)\exp\left(-\dfrac{2D_1^2t(\nicefrac{c}{2} - 1)^2 }{\eta^2}\right) \leq \exp\left(4\right)\sum_{\btheta'\neq \btheta^*}\exp\left(-\dfrac{2D_1^2t(\nicefrac{c}{2} - 1)^2 }{\eta^2}\right) \nonumber
\end{align}
 where $(a)$ follows from Lemma \ref{lemma:cond-Mcdiarmid} and noting that $D_0\left(\dfrac{\beta(J,\delta)}{D_0} + \dfrac{\alpha(J)}{D_1} - t + \dfrac{tc}{2} \right) < 0$ for $t > (1 + c)M$, the inequality $(b)$ follows from definition of $M$ and noting that $D_0 \geq D_1$, $(c)$ follows as for $M > 1$ we can show that $M^2 + 2tM(\nicefrac{c}{2}-1) + t^2(\nicefrac{c}{2}-1)^2 \geq 2tM(\nicefrac{c}{2}-1) + t^2(\nicefrac{c}{2} - 1)^2$, $(d)$ follows as $D_1 \geq \eta_0$, $(e)$ follows as $M > 1$, $(f)$ follows as $0 < c < 1$, and $(g)$ follows as $\eta_0 \leq \eta$.
\end{proof}

\begin{lemma}
\label{lemma:supp:2}
Let $\Delta_{t'}(\btheta') \colonequals L_{t'}(\btheta') - L_{t'}(\btheta^*)$ from \eqref{eq:log-likelihood-ratio}, $D_1 =  \min\limits_{\btheta \in \bTheta, \btheta' \neq \btheta^*}\sum_{i=1}^n p_{\btheta}(i)(\mu_i(\btheta') - \mu_i(\btheta^*))^2$, and $\alpha(J)$ be the threshold depending only on $J$. Then for some constant $ 0 < c < 1$ and $\eta$ defined in Definition \ref{def:bounded} we show that
\begin{align*}
    \sum_{\btheta'\neq \btheta^*}\sum_{t':t' > \frac{\alpha(J)}{D_1} + \frac{tc}{2}} &\Pb\left(\Delta_{t'}(\btheta') - \E[\Delta_{t'}(\btheta')] < \alpha(J) - t'D_1\right) \leq \exp\left(4\right)\sum_{\btheta'\neq \btheta^*}\dfrac{\exp\left(-\dfrac{2 D_1^2  tc}{\eta^2}\right)}{1 - \exp\left(-\dfrac{ 2D_1^2 }{\eta^2}\right)}.
\end{align*}
.
\end{lemma}
\begin{proof}
Let us recall that
\begin{align*}
    &\sum_{\btheta'\neq \btheta^*}\sum_{t':t' > \frac{\alpha(J)}{D_1} + \frac{tc}{2}} \Pb\left(\Delta_{t'}(\btheta') - \E[\Delta_{t'}(\btheta')] < \alpha(J) - t'D_1\right)
    \overset{(a)}{\leq} \sum_{\btheta'\neq \btheta^*}\sum_{t':t' > \frac{\alpha(J)}{D_1} + \frac{tc}{2}}\exp\left(-\dfrac{2(\alpha(J) - t'D_1)^2}{t'\eta^2}\right)\\
    &\overset{}{=} \sum_{\btheta'\neq \btheta^*}\sum_{t':t' > \frac{\alpha(J)}{D_1} + \frac{tc}{2}}\exp\left(-\dfrac{2 D_1^2(t')^2 + 2\alpha(J)^2 - 4D_1\alpha(J)t'}{t'\eta^2}\right) \\
    &\overset{(b)}{\leq} \sum_{\btheta'\neq \btheta^*}\sum_{t':t' > \frac{\alpha(J)}{D_1} + \frac{tc}{2}}\!\!\!\!\!\exp\left(-\dfrac{2 D_1^2(t')^2}{t'\eta^2}\right)\exp\left(\dfrac{4\eta \alpha(J) t' - 2\alpha(J)^2}{t'\eta^2}\right) \\
    &\overset{(c)}{\leq} \sum_{\btheta'\neq \btheta^*}\sum_{t':t' > \frac{\alpha(J)}{D_1} + \frac{tc}{2}}\exp\left(4\right)\exp\left(-\dfrac{2 D_1^2t' }{\eta^2}\right)
    \overset{(d)}{\leq} \exp\left(4\right)\sum_{\btheta'\neq \btheta^*}\dfrac{\exp\left(-{ 2D_1^2\left(\dfrac{\alpha(J)}{D_1} + \dfrac{tc}{2}\right) }/{\eta^2}\right)}{1 - \exp\left(-\dfrac{2D_1^2}{\eta^2}\right)} \\
    &= \exp\left(4\right)\!\!\sum_{\btheta'\neq \btheta^*}\dfrac{\exp\left(-\dfrac{ 2D_1\alpha(J)}{\eta^2} - \dfrac{ D_1^2 tc}{\eta^2} \right)}{1 - \exp\left(-\dfrac{2 D_1^2 }{\eta^2}\right)}
    \overset{(e)}{\leq} \exp\left(4\right)\!\!\sum_{\btheta'\neq \btheta^*}\dfrac{\exp\left(-\dfrac{2\eta_0\alpha(J)}{\eta^2}\right)\exp\left(-\dfrac{2 D_1^2  tc}{\eta^2}\right)}{1 - \exp\left(-\dfrac{ 2D_1^2 }{\eta^2}\right)} \nonumber\\
    &\overset{(f)}{\leq} \exp\left(4\right)\!\!\sum_{\btheta'\neq \btheta^*}\frac{\exp\left(-\frac{2 D_1^2  tc}{\eta^2}\right)}{1 - \exp\left(-\frac{ 2D_1^2 }{\eta^2}\right)} 
\end{align*}
where $(a)$ follows from Lemma \ref{lemma:cond-Mcdiarmid} and noting that $\alpha(J) - t'D_1 < 0$ for $t'> \frac{\alpha(J)}{D_1} + \frac{tc}{2}$, $(b)$ follows as $D_1\leq \eta$, 
the inequality $(c)$ follows as for $\alpha(J) > 1$ we can show that 
\begin{align*}
    &\dfrac{4\eta\alpha(J) t' - 2\alpha(J)^2}{t'\eta^2} \leq 4
    \implies  4\eta\alpha(J)t' - 2\alpha(J)^2  \leq 4t'\eta^2
    \implies t' \geq \dfrac{2\alpha(J)^2}{ 4\eta\alpha - 4\eta^2} = \dfrac{\alpha(J)}{2\eta}\left(\dfrac{1}{1 - \frac{\eta}{\alpha(J)}}\right)
\end{align*}
which is true in this lemma as $t'> \frac{\alpha(J)}{D_1} + \frac{tc}{2}$, $D_1 \leq \eta$ and $\alpha(J) > \eta$. Then
$(d)$ follows by applying the infinite geometric progression formula, $(e)$ follows as $D_1 \geq \eta_0$, and $(f)$ follows as $\exp\left(-\dfrac{2\eta_0\alpha(J)}{\eta^2}\right) \leq 1$.

\end{proof}

\subsubsection{Concentration of $\Delta_{t}$}
\label{app:cher-conc-tyt}

\begin{lemma}
\label{lemma:cond-Mcdiarmid}
Define $\Delta_{s}(\btheta) = \ell_{s}(\btheta) - \ell_{s}(\btheta^*)$.
Let $\epsilon > 0$ be a constant and $\eta >0$ is the constant defined in Assumption \ref{assm:sub-gauss}. Then we can show that,
\begin{align*}
    \Pb(\Delta_{t}(\btheta) -\E[\Delta_{t}(\btheta)] \leq - \epsilon) \leq \exp\left(-\dfrac{2\epsilon^2}{t\eta^2}\right).
\end{align*}
\end{lemma}

\begin{proof}
Recall that $\Delta_{s}(\btheta) = \ell_{s}(\btheta) - \ell_{s}(\btheta^*) = (2Y_s - \mu_{I_s}(\btheta) - \mu_{I_s}(\btheta^*))(\mu_{I_s}(\btheta^*) - \mu_{I_s}(\btheta))$. Define $V_s = \Delta_{s}(\btheta) - \E[\Delta_{s}(\btheta)]$. Note that $\E[V_s] = 0$ which can be shown as follows:
\begin{align*}
    \E[V_s] = \E[\Delta_{s}(\btheta) - \E[\Delta_{s}(\btheta)]] = \sum_{i=1}^n\Pb(I_s = i)(\mu_{i}(\btheta^*) - \mu_{i}(\btheta))^2 - \sum_{i=1}^n\Pb(I_s = i)(\mu_{i}(\btheta^*) - \mu_{i}(\btheta))^2 = 0.
\end{align*}
Also note that $\sum_{s=1}^t V_s = \Delta_{t}(\btheta) - \E[\Delta_{t}(\btheta)]$. Next, we show that the moment generating function of the random variable $V_s$ is bounded. First note that the reward $Y_s$ is bounded between $\nicefrac{-\sqrt{\eta}}{2}$ and $\nicefrac{\sqrt{\eta}}{2}$. It then follows that:
\begin{align*}
    V_s &= \Delta_{s}(\btheta) - \E[\Delta_{s}(\btheta)] \\
    &= (2Y_s - \mu_{I_s}(\btheta) - \mu_{I_s}(\btheta^*))(\mu_{I_s}(\btheta^*) - \mu_{I_s}(\btheta)) - \sum_{i=1}^n\Pb(I_s = i)(\mu_{i}(\btheta^*) - \mu_{i}(\btheta))^2 \leq 2\eta.
\end{align*}
Similarly, it can be shown that $V_s \geq -2\eta$. Hence, for the bounded random variable $V_s\in[-2\eta, 2\eta]$ we can show from Hoeffding's lemma in Lemma \ref{lemma:hoeffding} that
\begin{align*}
    \E[\exp\left(\lambda V_s\right)] \leq \exp\left(\dfrac{\lambda^2}{8}\left(2\eta - (-2\eta)\right)^2\right) = \exp\left(2\lambda^2\eta^2\right)
\end{align*}
for some $\lambda\in\mathbb{R}$. Now for any $\epsilon>0$ we can show that
\begin{align*} 
\Pb(\Delta_{t}(\btheta) -\E[\Delta_{t}(\btheta)] \leq - \epsilon) &=\Pb\left(\sum_{s=1}^{t} V_{s} \leq -\epsilon\right) =\Pb\left(-\sum_{s=1}^{t} V_{s} \geq \epsilon\right) \\ 
&=\Pb\left(e^{-\lambda \sum_{s=1}^{t} V_{s}} \geq e^{\lambda \epsilon}\right) \overset{(a)}{\leq} e^{-\lambda \epsilon} \E\left[e^{-\lambda \sum_{s=1}^{t} V_{s}}\right] \\
&=  e^{-\lambda \epsilon} \E\left[\E\left[e^{-\lambda \sum_{s=1}^{t} V_{s}}\big|\wtheta(t-1)\right] \right]\\
&\overset{(b)}{=} e^{-\lambda \epsilon} \E\left[\E\left[e^{-\lambda  V_{t}}|\wtheta(t-1)\right]\E\left[e^{-\lambda \sum_{s=1}^{t-1} V_{s}} \big|\wtheta(t-1)\right]  \right]\\
&\leq e^{-\lambda \epsilon} \E\left[\exp\left(2\lambda^2\eta^2\right)\E\left[e^{-\lambda \sum_{s=1}^{t-1} V_{s}}\big |\wtheta(t-1)\right]  \right]\\
& \overset{}{=} e^{-\lambda \epsilon} e^{2\lambda^{2} \eta^{2}} \mathbb{E}\left[e^{-\lambda \sum_{s=1}^{t-1} V_{s}}\right] \\ 
& \vdots \\ 
& \overset{(c)}{\leq} e^{-\lambda \epsilon} e^{2\lambda^{2} t \eta^{2}} \\
& \overset{(d)}{\leq} \exp\left(-\dfrac{2\epsilon^2}{t\eta^2}\right)
\end{align*}
where $(a)$ follows by Markov's inequality, $(b)$ follows as $V_s$ is conditionally independent given $\wtheta(s-1)$, $(c)$ follows by unpacking the term for $t$ times and $(d)$  follows by taking $\lambda= \epsilon / 4t\eta^2$.
\end{proof}

\subsubsection{Proof of correctness for General Sub-Gaussian Case}
\label{app:cher-proof-correct-SG}

\begin{lemma}
\label{lemma:stop-time-general}
Let $L_{t}(\btheta)$ be the sum of squared errors of the hypothesis parameterized by  $\btheta$ based on observation vector $\mathbf{Y}^{t}$ from an underlying sub-Gaussian distribution. Let $\tau_{\btheta^*\btheta} \colonequals \min\{t: L_{t}(\btheta^*) - L_{t}(\btheta) > \beta(J,\delta)\}$. Then we can show that
\begin{align*}
    \Pb\left( L_{\tau_{\btheta^*\btheta}}(\btheta^*) - L_{\tau_{\btheta^*\btheta}}(\btheta) > \beta(J,\delta)\right) \leq \dfrac{\delta}{J}
\end{align*}
where, $\beta(J,\delta) \colonequals \log\left(\dfrac{\left(1+\nicefrac{\eta^2}{\eta_0^2}\right)J}{\delta}\right) $.
\end{lemma}

\begin{proof}
Let $-\Delta_{t}(\btheta) \colonequals L_{t}(\btheta^*) - L_{t}(\btheta)$ be the difference of sum of squared errors between hypotheses parameterized by $\btheta^*$ and $\btheta$. We again define $\tau_{\btheta^*\btheta}(Y^t) \colonequals \min\{t: -\Delta_{t}(\btheta) > \beta(J,\delta)\}$. For brevity in the following proof we drop the $Y^t$ in $\tau_{\btheta^*\btheta}(Y^t)$. Then we can show that
\begin{align*}
    &\Pb(\exists t < \infty,  -\Delta_{t}(\btheta) > \beta(J,\delta)) = \sum_{t=1}^{\infty}\Pb\left(\tau_{\btheta^*\btheta} = t, -\Delta_{t}(\btheta) > \beta(J,\delta)\right)\\ 
    &= \sum_{t=1}^{\infty}\Pb\left(\tau_{\btheta^*\btheta} = t, -\Delta_{t}(\btheta) - \E[-\Delta_{t}(\btheta)] > \beta(J,\delta) - \E[-\Delta_{t}(\btheta)]\right)\\
    &\leq \sum_{t=1}^{\infty}\Pb\left(\tau_{\btheta^*\btheta} = t, -\Delta_{t}(\btheta) - \E[-\Delta_{t}(\btheta)] > \beta(J,\delta) + tD_1\right)\\
    &\overset{(a)}{\leq}  \sum_{t=1}^{\infty}\exp\left(-\dfrac{2\left(\beta(J,\delta) + tD_1\right)^2}{t\eta^2}\right)
    \overset{(b)}{\leq}  \sum_{t=1}^{\infty}\exp\left(-\left({\beta(J,\delta)}{} + \dfrac{t^2D_1^2}{t\eta^2}\right)\right)\\
    &\overset{}{=}  \sum_{t=1}^{\infty}\exp\left(-\left({\beta(J,\delta)}{} + \dfrac{tD_1^2}{\eta^2}\right)\right)
    = \exp\left(-{\beta(J,\delta)}\right)\left[1 + \exp(-D_1^2/\eta^2) + \exp(-2 D_1^2/\eta^2) + \exp(-3 D_1^2/\eta^2) + \ldots\right]\\
    &\overset{(c)}{=}  \exp\left(-{\beta(J,\delta)}{}\right)\dfrac{1}{1 - \exp\left(-D_1^2/\eta^2\right)}
    \overset{(d)}{\leq} \exp\left(-{\beta(J,\delta)}{}\right)\left(1 + \dfrac{\eta^2}{D_1^2}\right) \overset{(e)}{\leq} \dfrac{\delta}{J}.
\end{align*}
where, $(a)$ follows from Lemma \ref{lemma:cond-Mcdiarmid} and noting that $- \E[-\Delta_{t}(\btheta)] \leq -tD_0$, 
$(b)$ follows as $(\beta(J,\delta) + tD_0)^2 \geq 2\beta(J,\delta) + t^2D_0^2$ for $a,b >0$, 
$(c)$ follows from the infinite geometric series sum formula, $(d)$ follows as $\nicefrac{1}{1 - \exp(-x)} \leq 1 + \nicefrac{1}{x}$ for $x > 0$, 
and $(e)$ follows as $\beta(J,\delta) \colonequals \log\left(\dfrac{\left(1+\nicefrac{\eta^2}{\eta_0^2}\right)J}{\delta}\right)$ and noting that $D_1 \geq \eta_0$.
\end{proof}

\subsubsection{Stopping time Correctness Lemma for the Gaussian Case}
\label{app:cher-proof-correct-G}
\begin{lemma}
\label{lemma:stop-time}
Let $L_{t}(\btheta)$ be the sum of squared errors of the hypothesis parameterized by  $\btheta$ based on observation vector $\mathbf{Y}^{t}$ from an underlying Gaussian distribution. Let $\tau_{\btheta^*\btheta} \colonequals \min\{t: L_{t}(\btheta^*) - L_{t}(\btheta) > \beta(J,\delta)\}$. Then we can show that
\begin{align*}
    \Pb\left( L_{\tau_{\btheta^*\btheta}}(\btheta^*) - L_{\tau_{\btheta^*\btheta}}(\btheta) > \beta(J,\delta)\right) \leq \dfrac{\delta}{J}
\end{align*}
where we define the threshold function as, 
\begin{align}
    \beta(J,\delta) \colonequals \log(J/\delta) \label{eq:threshold-func}
\end{align}
\end{lemma}
\begin{proof}
Let $-\Delta_{t}(\btheta) \colonequals L_{t}(\btheta^*) - L_{t}(\btheta)$ be the log-likelihood ratio between hypotheses parameterized by $\btheta^*$ and $\btheta$. Define $\tau_{\btheta^*\btheta}(Y^t) \colonequals \min\{t: -\Delta_{t}(\btheta) > \beta(J,\delta)\}$. For brevity in the following proof we drop the $Y^t$ in $\tau_{\btheta^*\btheta}(Y^t)$. Then we can show that at time $t\geq\tau_{\btheta^*\btheta}$ we have 
\begin{align}
    -\Delta_{t}(\btheta) > \beta(J,\delta) &\implies \exp\left( -\Delta_{t}(\btheta)\right) > \exp\left(\beta(J,\delta)\right)\nonumber\\
    &\implies \left(\dfrac{\prod\limits_{s=1}^{t}\Pb(Y_{I_s} = y_s | I_s, \btheta)}{ \prod\limits_{s=1}^{t}\Pb(Y_{I_s} = y_s | I_s, \btheta^*) }\right) > \exp\left(\beta(J,\delta)\right)\nonumber\\
    &\implies \exp\left(-\beta(J,\delta)\right)\left(\dfrac{\prod\limits_{s=1}^{t}\Pb(Y_{I_s} = y_s | I_s, \btheta) }{\prod\limits_{s=1}^{t}\Pb(Y_{I_s} = y_s | I_s, \btheta^*) }\right) > 1.\label{eq:prod-stop-time}
\end{align}
Following this we can show that the probability of the event $\{-\Delta_{t}(\btheta) > \beta(J,\delta)\}$ is upper bounded by

\begin{align*}
    \Pb(\exists t < \infty,  -\Delta_{t}(\btheta) &> \beta(J,\delta)) = \sum_{t=1}^{\infty}\Pb\left(\tau_{\btheta^*\btheta} = t\right) = \sum_{t=1}^{\infty}\E[\indic{\tau_{\btheta^*\btheta} = t}]\\ 
    &\overset{(a)}{\leq} \sum_{t=1}^{\infty}\E\left[\indic{\tau_{\btheta^*\btheta} = t}\exp\left(-\beta(J,\delta)\right)\left(\dfrac{\prod\limits_{s=1}^{t}\Pb(Y_{I_s} = y_s | I_s, \btheta) }{\prod\limits_{s=1}^{t}\Pb(Y_{I_s} = y_s | I_s, \btheta^*) }\right)\right]\\
    &=  \exp\left(-\beta(J,\delta)\right)\sum_{t=1}^{\infty}\int_{\R^{t}}\indic{\tau_{\btheta^*\btheta} = t}\dfrac{\prod\limits_{s=1}^{t}\Pb(Y_{I_s} = y_s | I_s, \btheta) }{\prod\limits_{s=1}^{t}\Pb(Y_{I_s} = y_s | I_s, \btheta^*) }\prod\limits_{s=1}^{t}\Pb(Y_{I_s} = y_s | I_s, \btheta^*)dy_1dy_2\ldots dy_{t}\\
    &=  \exp\left(-\beta(J,\delta)\right)\sum_{t=1}^{\infty}\int_{\R^{t}}\indic{\tau_{\btheta^*\btheta} = t}\prod\limits_{s=1}^{t}\Pb(Y_{I_s} = y_s | I_s, \btheta) dy_1dy_2\ldots dy_{t}\\
    &=  \exp\left(-\beta(J,\delta)\right)\sum_{t=1}^{\infty}\Pb(\tau_{\btheta^*\btheta} = t | I^t, \btheta)
    \leq \exp\left(-\beta(J,\delta)\right) \overset{(b)}{\leq} \dfrac{\delta}{J}.
\end{align*}
where, $(a)$ follows from \eqref{eq:prod-stop-time}, and $(b)$ follows from \eqref{eq:threshold-func}. 
The claim of the lemma follows.
\end{proof}

\subsubsection{Proof of \cher Sample Complexity (\Cref{thm:chernoff-upper})}
\label{app:thm:chernoff-upper}

\begin{customtheorem}{1}\textbf{(Restatement)}
Let $\tau_\delta$ denote the stopping time of \cher in Algorithm \ref{alg:cher}. 
Let $D_0$ be the objective value of the $\max\min$ optimization in \eqref{eq:opt-lower00} when $\btheta=\btheta^\ast$, i.e.,
\begin{align*}
D_0 &\colonequals \max_{\text{ }\mathbf{p}} \min_{\btheta' \neq \btheta^\ast} \sum_{i=1}^n p(i) (\mu_i(\btheta') -\mu_i(\btheta^\ast))^2.
\end{align*}
Denoting $\mathbf{p}_{\btheta}$ as the solution of \eqref{eq:opt-lower00} when $\widehat{\btheta}(t)$ is replaced by any $\btheta \in \bTheta$, let $D_1$ be the minimum possible objective value of \eqref{eq:opt-lower00} over all $\mathbf{p}_{\btheta}$ when $\widehat{\btheta}(t)$ is replaced by $\btheta^*$, i.e., 
\begin{align*}
D_1 &\colonequals \min_{\{\mathbf{p}_{\btheta} : \btheta \in \bTheta\}} \min_{\btheta' \neq \btheta^\ast} \sum_{i=1}^n p_{\btheta}(i)(\mu_i(\btheta') - \mu_i(\btheta^*))^2.
\end{align*}
Assumption~\ref{assm:positive-D1} 
ensures that $D_1 > 0$. The sample complexity of the $\delta$-PAC \cher has the following upper bound, where $J \colonequals |\bTheta|$, $C = O((\eta/\eta_0)^2)$ is a constant:
\begin{align*}
\E[\tau_\delta] \leq O\left(\dfrac{\eta\log( C)\log J}{D_1} + \dfrac{\log(J/\delta)}{D_0} + JC^{\frac{1}{\eta}}\delta^{\frac{D_0}{\eta^2}}\right).
\end{align*}\end{customtheorem}

\begin{proof}
\textbf{Step 1 (Definitions):} 
Define $L_{t}(\btheta)$ as the total sum of squared errors of hypothesis $\btheta$ till round $t$. 
Let, $-\Delta_{t}(\btheta) \colonequals L_{t}(\btheta^*) -  L_{t}(\btheta)$ be the 
difference of squared errors between 
$\btheta^*$ and $\btheta$. 
Note that the p.m.f. $\mathbf{p}_{\btheta}$ is the Chernoff verification proportion for verifying hypothesis $\btheta$.

\textbf{Step 2 (Define $\tau_\delta$ and partition):} We define the stopping time $\tau_\delta$ for the policy $\pi$ as follows:
\begin{align}
    \tau_\delta \colonequals \min\{t: \exists \btheta\in\bTheta, L_{t}(\btheta') - L_{t}(\btheta) > \beta(J,\delta), \forall \btheta'\neq \btheta\} \label{eq:stop-time-cher1}
\end{align}
where, $\beta(J,\delta)$ is the threshold function.

\textbf{Step 3 (Define bad event):} We define the bad event $\xi^\delta(\btheta)$ for the sub-optimal hypothesis $\btheta \neq \btheta^*$ as follows:
\begin{align}
    \xi^\delta(\btheta) = \{\widehat{\btheta}(\tau_\delta) = \btheta,  L_{\tau_\delta}(\btheta') - L_{\tau_\delta}(\btheta) > \beta(J,\delta), \forall \btheta'\neq \btheta\}.
    \label{eq:bad-event-cher1}
\end{align}
The event $\xi^\delta(\btheta)$ denotes that a sub-optimal hypothesis $\btheta$ is declared the optimal hypothesis when it has a smaller sum of squared errors than any other hypothesis $\btheta'$ at $\tau_\delta$.

\textbf{Step 4 (Decomposition of bad event):} In this step we decompose the bad event to show that only comparing $\btheta$ against $\btheta^*$ is enough to guarantee a $\delta$-PAC policy. First we decompose the bad event $\xi^\delta(\btheta)$ as follows:
\begin{align}
    \xi^{\delta}(\btheta) = &\{\widehat{\btheta}(\tau_\delta) = \btheta, L_{\tau_\delta}(\btheta') - L_{\tau_\delta}(\btheta)  > \beta(J,\delta), \forall \btheta'\neq \btheta\}\nonumber\\
    \overset{(a)}{=} &\{\underbrace{\widehat{\btheta}(\tau_\delta) = \btheta, L_{\tau_\delta}(\btheta') -  L_{\tau_\delta}(\btheta)  > \beta(J,\delta), \forall \btheta' \in \bTheta\setminus\{\btheta^*\}\}}_{\textbf{part A}}\nonumber\\
    &\quad\bigcap \{\underbrace{\widehat{\btheta}(\tau_\delta) = \btheta, L_{\tau_\delta}(\btheta^*) - L_{\tau_\delta}(\btheta)  > \beta(J,\delta)}_{\textbf{part B}}\}\nonumber\\
    \subseteq & \{\widehat{\btheta}(\tau_\delta) = \btheta, L_{\tau_\delta}(\btheta^*) - L_{\tau_\delta}(\btheta)  > \beta(J,\delta)\}
    \label{eq:cher-event-decompose} 
\end{align}
where, $(a)$ follows by decomposing the event in two parts containing $\btheta \in \bTheta \setminus\{\btheta^*\}$ and $\{\btheta^*\}$, 
$(b)$ follows by noting that the intersection of events holds by taking into account only the event in part B.

\textbf{Step 5 (Proof of correctness):} In this step we want to show that based on the $\tau_\delta$ definition and the bad event $\xi^{\delta}(\btheta)$ the \cher stops and outputs the correct hypothesis $\btheta^*$ with $1-\delta$ probability. As shown in Step $4$, we can define the error event $\xi^{\delta}(\btheta)$ as follows:
\begin{align*}
    \xi^{\delta}(\btheta) \subseteq & \{\widehat{\btheta}(\tau_\delta) = \btheta, L_{\tau_\delta}(\btheta^*) - L_{\tau_\delta}(\btheta) > \beta(J,\delta)\}
\end{align*}
Define $\tau_{\btheta^*\btheta} = \min\{t: -\Delta_{t}(\btheta) > \beta(J,\delta)\}$. Then we can show for the stopping time $\tau_\delta$, the round $\tau_{\btheta^*\btheta}$ from Lemma \ref{lemma:stop-time-general}
and the threshold $\beta(J,\delta) \colonequals \log\left(\dfrac{\left(1+\nicefrac{\eta^2}{\eta_0^2}\right)J}{\delta}\right)$ we have
\begin{align*} 
\Pb\left(\tau_{\delta}<\infty, \widehat{\btheta}(\tau_{\delta}) \neq \btheta^*\right) & \leq \Pb\left(\exists \btheta \in \bTheta \setminus \{\btheta^*\}, \exists t \in \mathbb{N}: -\Delta_{t}(\btheta)>\beta(J, \delta)\right) \\ 
& \leq \sum_{\btheta \neq \btheta^*} \Pb\left(L_{\tau_{\btheta^*\btheta}}(\btheta^*) - L_{\tau_{\btheta^*\btheta}}(\btheta) > \beta(J,\delta), \tau_{\btheta^*\btheta}< \infty\right) 
\overset{(a)}{\leq} \sum_{\btheta \neq \btheta^*} \dfrac{\delta}{J} \leq \delta
\end{align*}
where, $(a)$ follows from Lemma \ref{lemma:stop-time-general}.

\textbf{Step 6 (Sample complexity analysis):} In this step we bound the total sample complexity satisfying the $\delta$-PAC criteria. We define the stopping time $\tau_\delta$ as follows:
\begin{align*}
    \tau_\delta \colonequals \min \left\{t: L_{t}(\btheta') - L_{t}(\widehat{\btheta}(t)) > \beta(J,\delta),  \forall \btheta' \neq\widehat{\btheta}(t)\right\}
\end{align*}
We further define the time $\tau_{\btheta^*}$ for the hypothesis $\btheta^*$ as follows:
\begin{align}
    \tau_{\btheta^*}:=\min \left\{t: L_{t}(\btheta') -  L_{t}(\btheta^*) > \beta(J,\delta), \forall \btheta' \neq \btheta^{*}\right\}.
\end{align}
We also define the critical number of samples as $(1+c)M$ where $M$ is defined as follows:
\begin{align}
    M \colonequals \left(\dfrac{\alpha(J)}{D_1} + \dfrac{C'+ \log(J/\delta)}{D_0} \right) \label{eq:cric-pulls}
\end{align}
where, $C' = \log\left(1 + \nicefrac{\eta^2}{\eta_0^2}\right)$. Hence $C' + \log(J/\delta)$ follows from the definition of $\beta(J,\delta)$ in \Cref{lemma:stop-time-general}. We define the term $D_1$ as follows:
\begin{align}
    D_1 \colonequals \min_{\{\mathbf{p}_{\btheta} : \btheta \in \bTheta\}} \min_{\btheta' \neq \btheta^\ast} \sum_{i=1}^n p_{\btheta}(i)(\mu_i(\btheta') - \mu_i(\btheta^*))^2
\end{align}
and the term $D_0$ as follows:
\begin{align}
    D_0 \colonequals \min\limits_{\btheta' \neq \btheta^*}\sum_{i=1}^n p_{\btheta^*}(i)(\mu_i(\btheta')- \mu_i(\btheta^*))^2
\end{align}
It then follows that
\begin{align}
    \E[\tau_\delta] &\leq \E[\tau_{\btheta^{*}}] = \sum_{t=0}^{\infty}\Pb(\tau_{\btheta^*} = t) 
    \overset{(a)}{\leq} 1 + (1 + c)M + \sum_{t: t > (1+c)M}\Pb(\tau_{\btheta^*} > t)\nonumber\\
    &\overset{(b)}{\leq} 1 + (1 + c)\left(\dfrac{\alpha(J)}{D_1} + \dfrac{C' + 4\log(J/\delta)}{D_0} \right) + J\sum_{t: t> \left(\dfrac{\alpha(J)}{D_1} + \dfrac{\log(J/\delta)}{D_0}\right)(1+c)} C_1\exp(-C_2t)\nonumber\\
    & \overset{(c)}{\leq} 1 + (1 + c)\left(\dfrac{\alpha(J)}{D_1} + \dfrac{C'+ 4\log(J/\delta)}{D_0} \right) + J  C_1 \dfrac{\exp\left(-C_2(1+c)\left(\dfrac{\alpha(J)}{D_1} + \dfrac{4\log(J/\delta)}{D_0} \right)\right)}{1 -\exp\left(-C_2\right)}\nonumber\\
    & \overset{(d)}{\leq} 1 + (1 + c)\left(\dfrac{\alpha(J)}{D_1} + \dfrac{C' + 4\log(J/\delta)}{D_0} \right) + J  C_1 \dfrac{\exp\left(-C_2(1+c)\left(\dfrac{\alpha(J)}{D_0} + \dfrac{4\log(J)}{D_0} + \dfrac{4\log(1/\delta)}{D_0} \right)\right)}{1 -\exp\left(-C_2\right)}\nonumber\\
     & \overset{(e)}{\leq} 1 + (1 + c)\left(\dfrac{\alpha(J)}{D_1} + \dfrac{C'+ \log(J/\delta)}{D_0} \right) + J  C_1 \dfrac{\exp\left(-C_2\left(\dfrac{\alpha(J) + 4\log J}{D_0} \right)\exp\left(-4C_2 \dfrac{\log(1/\delta)}{D_0} \right)\right)}{1 -\exp\left(-C_2\right)}\nonumber\\
     & \overset{(f)}{=} 1 + (1 + c)\left(\dfrac{b\log J}{D_1} + \dfrac{C'+ \log(J/\delta)}{D_0} \right) + J  C_1 \dfrac{\exp\left(-C_2\left(\dfrac{(b + 4)\log J}{D_0} \right)\exp\left(-4C_2 \dfrac{\log(1/\delta)}{D_0} \right)\right)}{1 -\exp\left(-C_2\right)}\nonumber\\
     & \leq 1 + (1 + c)\left(\dfrac{b\log J}{D_1} + \dfrac{C' + \log(J/\delta)}{D_0} \right) + J  C_1 \left(\dfrac{1}{J}\right)^{C_2(b + 4)/D_0}\delta^{4C_2/D_0}\left(1 + \max\{1 , \dfrac{1}{C_2}\}\right)\nonumber\\
     & \leq 1 + (1 + c)\left(\dfrac{b\log J}{D_1} + \dfrac{C' + \log(J/\delta)}{D_0} \right) + C_1(2+\dfrac{1}{C_2})J^{1- \dfrac{C_2(b + 4)}{D_0}} \delta^{\dfrac{4C_2}{D_0}}\nonumber\\
     &\overset{(g)}{\leq} 1 + 2\left(\dfrac{b\log J}{D_1} + \dfrac{C'+ \log(J/\delta)}{D_0} \right) + \left(165 + \dfrac{\eta^2}{D_1^2}\right)\left(2 +  \dfrac{\eta^2}{D_1^2}\right)J^{1 - \dfrac{D_1^2(b + 4)}{2\eta^2D_0}}\delta^{\dfrac{D_1^2}{\eta^2D_0}} \nonumber\\
     & \overset{}{\leq} 1 + 2\left(\dfrac{b\log J}{D_1} + \dfrac{C'+ \log(J/\delta)}{D_0} \right) + \left(165 + \dfrac{\eta^2}{D_1^2}\right)^2 J^{1- \dfrac{D_1^2(b + 4)}{2 \eta^2 D_0 }} \delta^{\dfrac{D_0}{\eta^2}}\nonumber\\
     & \overset{(h)}{\leq} 1 + \underbrace{2\left(\dfrac{b\log J}{D_1} + \dfrac{C'+ \log(J/\delta)}{D_0} \right)}_{\textbf{Term A}} + \underbrace{\left(165 + \dfrac{\eta^2}{\eta_0^2}\right)^2 J^{\big(1- \dfrac{\eta_0^2(b + 4)}{ 2\eta^3}\big)}}_{\textbf{Term B}}\times\underbrace{ \bigg(\delta^{\dfrac{D_0}{\eta^2}}\bigg)}_{\textbf{Term C}}\label{eq:final-sample-cher1}
\end{align}
where, $(a)$ follows from definition of $M$ in \eqref{eq:cric-pulls}, $(b)$ follows from Lemma \ref{conc:lemma:1},  $C_1 \colonequals 110 + 55\max\left\{1, \dfrac{\eta^2}{2 D_1^2 }\right\}$, $C_2 \colonequals \dfrac{ 2D_1^2\min\{(\nicefrac{c}{2}-1)^2, c\}}{\eta^2}$, $(c)$ follows by applying the geometric progression formula, $(d)$ follows as $D_1\leq D_0$. The inequality $(e)$ follows as $c > 0$, $(f)$ follows by setting $\alpha(J) = b\log J$ for some constant $b > 1$, $(g)$ follows by setting $c=\frac{1}{2}$ in $C_1$ and $C_2$, and $(h)$ follows as $D_1 \geq \eta_0$, and $D_0\leq \eta$.


Now, note that in \eqref{eq:final-sample-cher1} the {Term C} $\leq 1$ as $\delta \in (0,1)$. Now for the Term B we need to find an $b$ such that Term B $\leq J^{1-\dfrac{\eta_0^2(b + 4)}{2\eta^3} + \dfrac{b}{\eta\log J}}$. Hence,
\begin{align*}
    \left(165 + \dfrac{\eta^2}{\eta_0^2}\right)^2 J^{1- \dfrac{\eta_0^2(b + 4)}{ 2\eta^3}} &\leq J^{1-\dfrac{\eta_0^2(b + 4)}{2\eta^3} + \dfrac{b}{\eta\log J}}\implies
    \left(165 + \dfrac{\eta^2}{\eta_0^2}\right)^2 \leq J^{\dfrac{b}{\eta\log J}} 
    \implies \eta\log\left(165 + \dfrac{\eta^2}{\eta_0^2}\right) \leq b
\end{align*}
So for a constant $b > 1$ such that if $b$ satisfies the following condition
\begin{align}
    b = \eta\log\left(165 + \dfrac{\eta^2}{\eta_0^2}\right)
    >  \log\left(1 + \dfrac{\eta^2}{\eta_0^2}\right) \label{eq:const-C-eta}
\end{align}
then we have that Term B $\leq J^{1-\dfrac{\eta_0^2(b + 4)}{2\eta^3} + \dfrac{b}{\eta\log J}}$. Hence we set the value of $C= 165 + \eta^2/\eta_0^2 > C'$ which shows up in our theorem statement. Plugging this in \eqref{eq:final-sample-cher1} we get that the expected sample complexity is upper bounded by \begin{align*}
    &\E[\tau_\delta] \overset{}{\leq} 1 + 2\left(\dfrac{\eta\log\left(165 + \dfrac{\eta^2}{\eta_0^2}\right)\log J}{D_1} + \dfrac{\log\left(1 + \dfrac{\eta^2}{\eta_0^2}\right) +\log(J/\delta)}{D_0} \right) + J^{1-\dfrac{\eta_0^2(b + 4)}{2\eta^3} + \dfrac{b}{\eta\log J}} \delta^{\dfrac{D_0}{\eta^2}}\\
    &\overset{(a)}{\leq} 1 + 2\left(\dfrac{\eta\log\left(165 + \dfrac{8\eta^2}{\eta_0^2}\right)\log J}{D_1}+ \dfrac{\log\left(\left(165 + \dfrac{\eta^2}{\eta_0^2}\right)\dfrac{J}{\delta}\right)}{D_0} \right) + J^{1 + \dfrac{\log\left(165 + \dfrac{\eta^2}{\eta_0^2}\right)}{\eta\log J}} \delta^{\dfrac{D_0}{\eta^2}}\\
    &\overset{(b)}{\leq} 1 + 2\left(\dfrac{\eta\log(C)\log J}{D_1} + \dfrac{\log(CJ/\delta)}{D_0} \right) + J^{1 + \dfrac{\log(C)}{\eta\log J}} \delta^{\dfrac{D_0}{\eta^2}}\\
    &\overset{}{=} 1 + 2\left(\dfrac{\eta\log(C)\log J}{D_1} + \dfrac{\log(CJ/\delta)}{D_0} \right) + J\cdot J^{ \dfrac{\log(C)^{1/\eta}}{\log J}} \delta^{\dfrac{D_0}{\eta^2}}\\
    &= O\left(\dfrac{\eta\log(C)\log J}{D_1} + \dfrac{\log(J/\delta)}{D_0} + J(C)^{1/\eta}\delta^{D_0/\eta^2} \right)
\end{align*}
where, $(a)$ follows as ${2\log\left(165 + \dfrac{\eta^2}{\eta_0^2}\right)} \geq \log\left(1 + \dfrac{\eta^2}{\eta_0^2}\right)$, and in $(b)$ we substitute $C = \left(165 + \dfrac{\eta^2}{\eta_0^2}\right)$. 
The claim of the Theorem follows.
\end{proof}

\subsection{Proof of \topllr Sample Complexity}
\label{app:thm:topllr-upper}

\subsubsection{Concentration Lemma}
\label{app:topllr-conc}
This section contains concentration lemma equivalent to the Lemma \ref{conc:lemma:1} of Appendix \ref{app:chernoff-sample-comp}.

\begin{lemma}
\label{conc:lemma:2}
Define $L_{t}(\btheta^*)$ as the sum of squared error of the hypothesis parameterized by  $\btheta^*$. Let $\tau_{\btheta^*} = \min\{t: L_{t}(\btheta') - L_{t}(\btheta^*) > \beta(J,\delta), \forall \btheta' \neq \btheta^* \}$. Then we can bound the probability of the event 
\begin{align*}
     \Pb(\tau_{\btheta^*} > t) \leq J C'_1\exp\left(-C'_2 t\right)
\end{align*}
where, $J \colonequals |\bTheta|$, $C'_1 \colonequals 110 + 55\max\left\{1, \dfrac{\eta^2}{2 D_1'^2 }\right\}$, $C'_2 \colonequals \dfrac{2 D_1'^2\min\{(c-1)^2, c\}}{\eta^2}$, $\eta >0$ defined in Definition \ref{def:bounded} and $D_1' \colonequals \min_{\btheta \neq  \btheta', \btheta' \neq \btheta^*}\sum_{i=1}^n u_{\btheta\btheta'}(i)(\mu_i(\btheta') - \mu_i(\btheta^*))^2$.
\end{lemma}

\begin{proof}
We define the event
\begin{align*}
    \xi_{\btheta'\btheta^*}(t) &= \{L_{t}(\btheta') - L_{t}(\btheta^*) < \beta(J,\delta)\}\\
    \tilde{\xi}_{\btheta'\btheta^*}(t) &= \{L_{t}(\btheta') - L_{t}(\btheta^*) < \alpha(J)\}
\end{align*}
Then we define the time $\tau_{\btheta^*}$ as follows:
\begin{align*}
    \tau_{\btheta^*} &= \min\{t: L_{t}(\btheta') - L_{t}(\btheta^*) > \beta(J,\delta),  \forall \btheta' \neq \btheta^* \}
\end{align*}
which is the first round $L_{t}(\btheta')$ crosses $\beta(J,\delta)$ threshold against $L_{t}(\btheta^*)$ for all $\btheta'\neq \btheta^*$. 
We also define the time $\tilde{\tau}_{\btheta^*}$ as follows:
\begin{align*}
    \tilde{\tau}_{\btheta'\btheta^*} &= \min\{t:  L_{t'}(\btheta') - L_{t'}(\btheta^*) > \alpha(J), \forall t' > t\}
\end{align*}
which is the first round when $L_{t}(\btheta^*)$ crosses $\alpha(J)$ threshold against $L_{t}(\btheta')$. Then we define the time
\begin{align*}
    \tilde{\tau}_{\btheta^*} = \max_k\{\tilde{\tau}_{\btheta^*\btheta'}\}
\end{align*}
as the last time $\tilde{\tau}_{\btheta^*\btheta'}$ happens. Define the term $D_0'$ and $D_1'$ as 
\begin{align*}
    D_0' \colonequals \min\limits_{\btheta,\btheta' \neq \btheta^*}\sum_{i=1}^n u_{\btheta^*\btheta}(i)(\mu_i(\btheta')- \mu_i(\btheta^*))^2, \qquad D_1' &\colonequals \min_{\btheta \neq  \btheta', \btheta' \neq \btheta^*}\sum_{i=1}^n u_{\btheta\btheta'}(i)(\mu_i(\btheta') - \mu_i(\btheta^*))^2
\end{align*}
Let $\Delta_{t}(\btheta')\colonequals L_{t}(\btheta') - L_{t}(\btheta^*)$ be the sum of squared errors between hypotheses parameterized by $\btheta'$ and $\btheta^\ast$. Then it follows that,
\begin{align*}
    \Pb(\tilde{\xi}_{\btheta'\btheta^*}(t)) &= \Pb(\Delta_{t}(\btheta') - tD_1' < \alpha(J) - tD_1')\\
    &\leq \Pb(\Delta_{t}(\btheta') - \E[\Delta_{t}(\btheta')] < \alpha(J) - tD_1')
\end{align*}
Similarly, we can show that,
\begin{align*}
    \Pb({\xi}_{\btheta'\btheta^*}(t)) &= \Pb(\Delta_{t}(\btheta') - \E[\Delta_{t}(\btheta')] < \beta(J,\delta) - \E[\Delta_{t}(\btheta')])\\
     &\leq \Pb(\Delta_{t}(\btheta') - \E[\Delta_{t}(\btheta')] < \beta(J,\delta) - (t-\tilde{\tau}_{\btheta^*})D_0')\\
    &= \Pb\left(\Delta_{t}(\btheta') - \E[\Delta_{t}(\btheta')] < D_0'\left(\dfrac{\beta(J,\delta)}{D_0'} - t + \tilde{\tau}_{\btheta^*})\right)\right)
\end{align*}
Then following the same approach as in Lemma \ref{conc:lemma:1} we can show that,
\begin{align*}
    \Pb(\tau_{\btheta^*} > t) &\leq \Pb(\bigcup_{\btheta'\neq \btheta^*} \xi_{\btheta'\btheta^*}(t)) \\
    &\leq \Pb\left(\left\{\bigcup_{\btheta'\neq \btheta^*}\xi_{\btheta'\btheta^*}(t)\right\}\bigcap\{\tilde{\tau}_{\btheta^*} < \dfrac{\alpha(J)}{D_1'} + tc\}\right) + \Pb\left(\left\{\xi_{\btheta'\btheta^*}(t)\right\}\bigcap\{\tilde{\tau}_{\btheta^*} \geq \dfrac{\alpha(J)}{D_1'} + tc\}\right) \\
    &\leq \sum_{\btheta'\neq \btheta^*}\Pb\left(\left\{\xi_{\btheta'\btheta^*}(t)\right\}\bigcap\{\tilde{\tau}_{\btheta^*} < \dfrac{\alpha(J)}{D_1'} + tc\}\right) + \sum_{\btheta'\neq \btheta^*}\sum_{t':t' \geq \frac{\alpha(J)}{D_1'} + tc }\Pb\left(\tilde{\xi}_{\btheta'\btheta^*}(t')\right)\\
    &\leq \sum_{\btheta'\neq \btheta^*}\Pb\left(\{\xi_{\btheta'\btheta^*}\}\bigcap\{\tilde{\tau}_{\btheta^*} < \dfrac{\alpha(J)}{D_1'} + tc\}\right) + \sum_{\btheta'\neq \btheta^*}\sum_{t':t' \geq \frac{\alpha(J)}{D_1'} + tc }\Pb\left(\Delta_{t'}(\btheta') - t'D_1' < \alpha(J) - t'D_1'\right)\\
    &\leq \sum_{\btheta'\neq \btheta^*} \Pb\left(\Delta_{t'}(\btheta') - \E[\Delta_{t'}(\btheta')] < D_0'\left(\dfrac{\beta(J,\delta)}{D_0'} + \dfrac{\alpha(J)}{D_1'} - t + tc \right)\right)\\
    &\quad + \sum_{\btheta'\neq \btheta^*}\sum_{t':t' \geq \frac{\alpha(J)}{D_1'} + tc }\Pb\left(\Delta_{t'}(\btheta') - t'D_1' < \alpha(J) - t'D_1'\right)\\
    &\overset{(a)}{\leq} \sum_{\btheta'\neq \btheta^*}\exp\left(4\right)\exp\left(-\dfrac{2D_1'^2t(\frac{c}{2} - 1)^2 }{\eta^2}\right) + \sum_{\btheta'\neq \btheta^*}\exp\left(4\right)\dfrac{\exp\left(-\dfrac{2D_1'^2  tc}{\eta^2}\right)}{1 - \exp\left(-\dfrac{2D_1'^2 }{\eta^2}\right)} \\
    &\overset{(b)}{\leq} 55\sum_{\btheta'\neq \btheta^*}\exp\left(-\dfrac{2D_1'^2t(\frac{c}{2} - 1)^2 }{\eta^2}\right) + 55\sum_{\btheta'\neq \btheta^*}\exp\left(-\dfrac{2D_1'^2  tc}{\eta^2}\right) + 55\sum_{\btheta'\neq \btheta^*}\exp\left(-\dfrac{2D_1'^2  tc}{\eta^2}\right) \max\left\{1, \dfrac{\eta^2}{2D_1'^2 }\right\}\\
    & \overset{(c)}{\leq} J C'_1\exp\left(-C'_2 t\right)
\end{align*}
where, $(a)$ follows from Lemma \ref{lemma:supp:1} and Lemma  \ref{lemma:supp:2} as their result holds for any $D_1', D_0' > 0$, $(b)$ follows from the same steps as in Lemma \ref{conc:lemma:1} as $D_1' \leq D_0'$, $D_1' > 0$, $D_0' > 0$, and in $(c)$ we substitute $C'_1 \colonequals 110 + 55\max\left\{1, \dfrac{\eta^2}{2 D_1^2 }\right\}$, $C'_2 \colonequals \dfrac{ 2D_1^2\min\{(\frac{c}{2}-1)^2, c\}}{\eta^2}$, and $J\colonequals |\bTheta|$.
\end{proof}

\subsubsection{Proof of \topllr Sample Complexity (\Cref{thm:topllr-upper})}
\label{app:topllr-upper}

\begin{customproposition}{1}\textbf{(Restatement)}
Let $\tau_\delta$ denote the stopping time of \topllr stops sampling following the sampling strategy of \eqref{eq:max-deriv-finite0}. 
Consider the set $\Is(\btheta, \btheta') \subset [n]$ of actions that could be sampled following \eqref{eq:max-deriv-finite0} when $\widehat{\btheta}(t) = \btheta$ and $\tilde{\btheta}(t) = \btheta'$, and let $\mathbf{u}_{\btheta\btheta'}$ denote a uniform pmf supported on $\Is(\btheta, \btheta')$. 
Define 
\begin{align*}
D_0' \colonequals \min\limits_{\btheta,\btheta' \neq \btheta^*}\sum_{i=1}^n u_{\btheta^*\btheta}(i)(\mu_i(\btheta')- \mu_i(\btheta^*))^2, \quad 
\\
D_1' \colonequals \min_{\btheta \neq  \btheta', \btheta' \neq \btheta^*}\sum_{i=1}^n u_{\btheta\btheta'}(i)(\mu_i(\btheta') - \mu_i(\btheta^*))^2, 
\end{align*}
where we assume that $D_1' > 0$. Then for a constant $C >0$ 
the sample complexity of \topllr has the following upper bound:
\begin{align*}
    \E[\tau_\delta] \leq O\left(\frac{\eta\log(C)\log J}{D_1'} + \frac{\log(J/\delta)}{D_0'} + JC^{1/\eta}\delta^{D_0'/\eta^2}\right).
\end{align*}
\end{customproposition}

\begin{proof}
\textbf{Step 1 (Definitions):} Let the action  $i_{\btheta\btheta'} \colonequals \argmax\limits_{i\in[n]}(\mu_{i}(\btheta) - \mu_{i}(\btheta'))^2$. Let $\widehat{\btheta}(t)$ 
denote the most likely hypothesis at round $s$ and $\tilde{\btheta}(t)$ 
be the second most likely hypothesis at round $s$. Note that \topllr only samples the action $i_{\btheta\btheta'}$ at round $t$ when $\widehat{\btheta}(t) = \btheta$ and $\tilde{\btheta}(t) = \btheta'$. Again, let $L_{t}(\btheta)$ denote the total sum of squared errors of hypothesis $\btheta$ till round $t$. 
We further define the set $\Is \colonequals \{i\in[n]: i=\argmax_{i'\in[n]}(\mu_{i'}(\btheta) - \mu_{i'}(\btheta'))^2 \text{ for some } \btheta,\btheta'\in \bTheta\}$.

\textbf{Step 2 (Define stopping time $\tau_\delta$):} We define the time $\tau_\delta$ for the policy \topllr as follows:
\begin{align}
    \tau_\delta \colonequals \min\{t: \exists \btheta\in\bTheta, L_{t}(\btheta') - L_{t}(\btheta) > \beta(J,\delta), \forall \btheta'\neq \btheta\} \label{eq:stop-time-topllr1}
\end{align}
where, $\beta(J,\delta)$ is the threshold function.

\textbf{Step 3 (Define bad event):} We define the bad event $\xi^\delta(\btheta)$ for the sub-optimal hypothesis $\btheta$ as follows:
\begin{align}
    \xi^\delta(\btheta) \colonequals \{L_{\tau_\delta}(\btheta') - L_{\tau_\delta}(\btheta) > \beta(J,\delta), \forall \btheta' \neq \btheta\}.
    \label{eq:bad-event-topllr1}
\end{align}
The event $\xi^\delta(\btheta)$ denotes that a sub-optimal hypothesis $\btheta$ has been declared optimal at time $\btheta$ and its sum of squared error is smaller than any other hypothesis $\btheta'\neq \btheta$ at $\tau_\delta$.

\textbf{Step 4 (Decomposition of bad event):} Decomposing the bad event follows the same approach in Theorem \ref{thm:chernoff-upper}. A crucial thing to note is that the stopping time $\tau_\delta$ only depends on the threshold function $\beta(J,\delta)$ and not on the sampling rule. Again we can decompose the bad event to show that only comparing $\btheta$ against $\btheta^*$ is enough to guarantee a $\delta$-PAC policy. Finally following \eqref{eq:cher-event-decompose} we can decompose the bad event $\xi^\delta(\btheta)$ as follows:
\begin{align}
    \xi^{\delta}(\btheta) 
    \subseteq & \{\widehat{\btheta}(\tau_\delta) = \btheta, L_{\tau_\delta}(\btheta^*) - L_{\tau_\delta}(\btheta) > \beta(J,\delta)\}
    \label{eq:topllr-event-decompose} 
\end{align}
such that we compare the sub-optimal hypothesis $\btheta$ only with optimal hypothesis $\btheta^*$.

\textbf{Step 5 (Control bad event):} The control of the bad event follows the same approach in Theorem \ref{thm:chernoff-upper}. We want to show that based on the definition of $\tau_\delta$ and the bad event $\xi^{\delta}(\btheta)$ the \topllr stops and outputs the correct hypothesis $\btheta^*$ with $1-\delta$ probability. As shown in Step $4$, we can define the error event $\xi^{\delta}(\btheta)$ as follows:
\begin{align*}
    \xi^{\delta}(\btheta) \subseteq & \{\widehat{\btheta}(\tau_\delta) = \btheta, L_{\tau_\delta}(\btheta^*) - L_{\tau_\delta}(\btheta) > \beta(J,\delta)\}
\end{align*}
Again define $\tau_{\btheta^*\btheta} \colonequals \min\{t: -\Delta_{t}(\btheta) > \beta(J,\delta)\}$. Then following the same steps as in Step 5 of Theorem \ref{thm:chernoff-upper} we can show that 
\begin{align*} 
\Pb\left(\tau_{\delta}<\infty, \widehat{\btheta}(\tau_{\delta}) \neq \btheta^*\right) & \leq \Pb\left(\exists \btheta \in \bTheta \setminus \{\btheta^*\}, \exists t \in \mathbb{N}: -\Delta_{t}(\btheta)>\beta(J, \delta)\right) \\
&\leq \sum_{\btheta \neq \btheta^*} \Pb\left(L_{\tau_{\btheta^*\btheta}}(\btheta^*) - L_{\tau_{\btheta^*\btheta}}(\btheta) > \beta(J,\delta), \tau_{\btheta^*\btheta}< \infty\right) 
\overset{(a)}{\leq} \sum_{\btheta \neq \btheta^*} \dfrac{\delta}{J} \leq \delta
\end{align*}
where $(a)$ follows follows from Lemma \ref{lemma:stop-time-general} and the definition of $\beta(J,\delta)$.

\textbf{Step 6 (Sample complexity analysis):} In this step we bound the total sample complexity of \topllr satisfying the $\delta$-PAC criteria. Recall that the set $\Is \colonequals \{i\in[n]: i=\argmax_{i'\in[n]}(\mu_{i'}(\btheta) - \mu_{i'}(\btheta'))^2 \text{ for some } \btheta,\btheta'\in \bTheta\}$. Note that \topllr does not sample by the Chernoff p.m.f. $\mathbf{p}_{\btheta}$. Rather it samples by the p.m.f.
\begin{align}
    u_{\btheta\btheta'} \colonequals \dfrac{1}{|\Is(\btheta\btheta')|} \label{eq:topllr-prop}
\end{align}
where, $\Is(\btheta\btheta') \colonequals \{i\in \Is: i=\argmax_{i'\in[n]}(\mu_{i'}(\btheta) - \mu_{i'}(\btheta'))^2 \text{ for } \btheta,\btheta'\in \bTheta\}$. Hence $u_{\btheta\btheta'}$ is a uniform random p.m.f between all the maximum mean squared difference actions between hypotheses $\btheta$ and $\btheta'$ which are $\widehat{\btheta}(t)$ and $\tilde{\btheta}(t)$ respectively for some rounds $s\in[\tau_\delta]$. The rest of the analysis follows the same steps as in Step 6 of Theorem \ref{thm:chernoff-upper} as the proof does not rely on any specific type of sampling proportion. We define the stopping time $\tau_\delta$ as follows:
\begin{align*}
    \tau_\delta =\min \left\{t: L_{t}(\btheta') - L_{t}(\widehat{\btheta}(t)) \geq \beta(J,\delta), \forall \btheta' \neq \widehat{\btheta}(t)\right\}.
\end{align*}
We further define the time $\tau_{\btheta^*}$ for the hypothesis $\btheta^*$ as follows:
\begin{align}
    \tau_{\btheta^*}:=\min \left\{t: L_{t}(\btheta')-L_{t}(\btheta^*) \geq \beta, \forall \btheta' \neq \btheta^{*}\right\}
\end{align}
We also define the critical number of samples as $(1+c)M'$ where $M'$ is defined as follows:
\begin{align}
    M \colonequals \left(\dfrac{\alpha(J)}{D'_1} + \dfrac{C' + \log(J/\delta)}{D'_0} \right) \label{eq:cric-pulls1}
\end{align}
where, $C' = \log(1 + \nicefrac{\eta^2}{\eta_0^2})$, $c >0$ is a constant, and we define the term $D_1'$ as follows:
\begin{align*}
    D_1' &\colonequals \min_{\btheta \neq  \btheta', \btheta' \neq \btheta^*}\sum_{i=1}^n u_{\btheta\btheta'}(i)(\mu_i(\btheta') - \mu_i(\btheta^*))^2
\end{align*}
and the term $D_0'$ as follows:
\begin{align*}
    D_0' \colonequals \min\limits_{\btheta,\btheta' \neq \btheta^*}\sum_{i=1}^n u_{\btheta^*\btheta}(i)(\mu_i(\btheta')- \mu_i(\btheta^*))^2.
\end{align*}
It then follows that
\begin{align}
    \E[\tau_\delta] &\leq \E[\tau_{\btheta^{*}}] = \sum_{t=0}^{\infty}\Pb(\tau_{\btheta^*} > t) 
    \overset{(a)}{\leq} 1 + (1 + c)M + \sum_{t: t > (1+c)M}\Pb(\tau_{\btheta^*} > t)\nonumber\\
    &\overset{(b)}{\leq} 1 + (1 + c)\left(\dfrac{b\log J}{D_1'} + \dfrac{C' + \log(J/\delta)}{D_0'} \right) + \sum_{\btheta'\neq \btheta^*}\sum_{t: t> \left(\dfrac{\alpha(J)}{D_1'} + \dfrac{4\log(J/\delta)}{D_0'}\right)(1+c)} C_1\exp(- C_2 t)\nonumber\\
    & \leq 1 + (1 + c)\left(\dfrac{b\log J}{D_1'} + \dfrac{C' + \log(J/\delta)}{D_0'} \right) + \sum_{\btheta'\neq \btheta^*}  C_1 \dfrac{\exp\left(-C_2\left(\dfrac{\alpha(J)}{D_1'} + \dfrac{4\log(J/\delta)}{D_0'} \right)\right)}{1 -\exp\left(-C_2\right)}\nonumber\\
    & \overset{(c)}{\leq} 1 + \underbrace{2\left(\dfrac{b\log J}{D'_1} + \dfrac{C' + \log(J/\delta)}{D'_0} \right)}_{\textbf{Term A}} + \underbrace{\left(165
    + \dfrac{2\eta^2}{\eta_0^2}\right)^2 J^{\big(1- \dfrac{\eta_0^2(b + 4)}{2 \eta^3}\big)}}_{\textbf{Term B}}\times\underbrace{ \bigg(\delta^{\dfrac{D'_0}{\eta^2}}\bigg)}_{\textbf{Term C}} \label{eq:final-sample-cher2}
\end{align}
where, $(a)$ follows from definition of $M$ in \eqref{eq:cric-pulls1}, $(b)$ follows from Lemma \ref{conc:lemma:2} where $C_1 \colonequals 110 + 55\max\left\{1, \dfrac{\eta^2}{2 D_1^2 }\right\}$, $C_2 \colonequals \dfrac{2 D_1^2\min\{(\frac{c}{2}-1)^2, c\}}{\eta^2}$, and $(c)$ follows the same steps in Theorem \ref{thm:chernoff-upper} by setting $c=\frac{1}{2}$ in $C_1$ and $C_2$.

Again, note that in \eqref{eq:final-sample-cher2} the Term C $\leq 1$ as $\delta \in (0,1)$. So for a constant $b > 1$ such that if $b$ satisfies the following condition
\begin{align}
    b = \eta\log\left(165 + \dfrac{\eta^2}{\eta_0^2}\right)
    >  \log\left(1 + \dfrac{\eta^2}{\eta_0^2}\right)
    \label{eq:const-C-eta1}
\end{align}
we have that Term B $\leq J^{1-\dfrac{\eta_0^2(b + 4)}{2\eta^3} + \dfrac{b}{\eta\log J}}$. Hence, Plugging this in \eqref{eq:final-sample-cher2} we get that the expected sample complexity is of the order of
\begin{align*}
    \E[\tau_\delta] \leq O\left(\dfrac{\eta\log(C)\log J}{D_1'} + \dfrac{\log(J/\delta)}{D_0'} + JC^{1/\eta}\delta^{D_0'/\eta^2}\right).
\end{align*}
where, $C = 165 + \dfrac{\eta^2}{\eta_0^2}$. The claim of the theorem follows.
\end{proof}

\subsection{Proof of \Cref{prop:batch-cher} (Batched Setting)}
\label{app:prop:batch-cher}

\begin{customproposition}{2}\textbf{(Restatement)}
Let $\tau_\delta$, $D_0$, $D_1$ be defined as in \Cref{thm:chernoff-upper} and $B$ be the batch size. Then the sample complexity of $\delta$-PAC \bcher is
\begin{align*}
    \E[\tau_\delta] \!\leq\! O\left(B \!+\! \frac{\eta\log( C)\log J}{D_1} \!+\! \frac{\log(J/\delta)}{D_0} \!+\! BJC^{\frac{1}{\eta}}\delta^{\frac{D_0}{\eta^2}}\right).
\end{align*}
\end{customproposition}

\begin{proof}
We follow the same proof technique as in Theorem \ref{thm:chernoff-upper}. We define the last phase after which the algorithm stops as $m_{\delta}$ defined as follows:
\begin{align*}
    m_{\delta} = \min\{m: L_{mB}(\btheta') - L_{mB}(\widehat{\btheta}(t)) > \beta(J,\delta), \forall \btheta'\neq \widehat{\btheta}(t)\}.
\end{align*}
We further define the phase $m_{\btheta^*}$ as follows:
\begin{align*}
    m_{\btheta^*} = \min\{m: L_{mB}(\btheta') - L_{mB}(\widehat{\btheta}(t)) > \beta(J,\delta), \forall \btheta'\neq \btheta^*\}.
\end{align*}
Then we can show that the expected last phase $m_{\delta}$ is bounded as follows:
\begin{align}
    \E[m_{\delta}] \leq \E[m_{\btheta^*}] = \sum_{m=1}^{\infty}\Pb(m_{\btheta^*} > m) \leq 1 + \underbrace{(1+c)M_1}_{\textbf{Part A}} + \sum_{m': m' > (1 + c)M_1}\underbrace{\Pb(m_{\btheta^*} > m')}_{\textbf{Part B}} \label{eq:prop-batch-main-1}
\end{align}
where, in Part A we define the critical number of phases
\begin{align*}
    M_1 = \dfrac{\alpha(J)}{B D_1} +  \dfrac{\beta(J,\delta)}{B D_0}.
\end{align*}
and $D_0$, $D_1$ as defined in Theorem \ref{thm:chernoff-upper}. 
Note that this definition of $M_1$ is different that the critical number of samples defined in Theorem \ref{thm:chernoff-upper}. Now we control the Part B. As like \Cref{conc:lemma:1} we define the following bad events
\begin{align*}
    \xi_{\btheta'\btheta^*}(m) &= \{L_{mB}(\btheta') - L_{mB}(\btheta^*) < \beta(J,\delta)\}\\
    \tilde{\xi}_{\btheta'\btheta^*}(m) &= \{L_{mB}(\btheta') - L_{mB}(\btheta^*) < \alpha(J)\}
\end{align*}
We further define the last good phase $\tilde{m}_{\btheta'\btheta^*}$ as follows:
\begin{align*}
    \tilde{m}_{\btheta'\btheta^*}& = \min\{m: L_{m'B}(\btheta') - L_{m'B}(\btheta^*) > \alpha(J), \forall m' > m\}\\
    \text{and } \tilde{m}_{\btheta^*} &= \max_{\btheta'\neq \btheta^*}\{\tilde{m}_{\btheta'\btheta^*}\}
\end{align*}
denote the last phase after which in all subsequent phases we have  $\btheta^*$ is $\widehat{\btheta}(t)$. We further define $\tilde{m}_{\btheta^*}$ as 
\begin{align*}
    \tilde{m}_{\btheta^*} = \dfrac{\alpha(J)}{B D_1} + \dfrac{mc}{2}.
\end{align*}
Note that this definition of $\tilde{m}_{\btheta^*}$ is different that $\tilde{\tau}_{\btheta^*}$ in \Cref{conc:lemma:1}. Using Lemma \ref{lemma:supp:1} we can further show the event $\xi_{\btheta'\btheta^*}$ is bounded as follows:
\begin{align}
    \Pb(\xi_{\btheta'\btheta^*}(m)) &= \Pb\left(\underbrace{L_{mB}(\btheta') - L_{mB}(\btheta^*)}_{\colonequals \Delta_{mB}(\btheta',\btheta^*)} < \beta(J,\delta)\right)\nonumber\\
    &\leq 
    \Pb\left(\Delta_{mB}(\btheta',\btheta^*) - \E[\Delta_{mB}(\btheta',\btheta^*)] < D_0\left(\dfrac{\beta(J,\delta)}{D_0} - B(m - \tilde{m}_{\btheta^*})\right)\right)\nonumber\\
    &\overset{(a)}{\leq} \exp\left(4B\right)\sum_{\btheta'\neq \btheta^*}\exp\left(-\dfrac{2D_1^2 Bm \left(\frac{c}{2}-1\right)^2}{\eta^2}\right)\label{eq:prop-batch-event1}
\end{align}
where, $(a)$ follows usinf the same steps as in \Cref{lemma:supp:1}. Similarly we can show that,
\begin{align}
    &\Pb\left(\tilde{\xi}_{\btheta'\btheta^*}(m)\right) = \sum_{\btheta\neq\btheta'}\sum_{m':m' > \dfrac{\alpha(J)}{B D_1} + \dfrac{mc}{2}}\Pb\left(\Delta_{m'B} - \E[\Delta_{m'B}] < \alpha(J) - \E[\Delta_{m'B}] \right)\nonumber\\
    &\overset{(a)}{\leq} \sum_{\btheta\neq\btheta'}\sum_{m':m' > \dfrac{\alpha(J)}{B D_1} + \dfrac{mc}{2}}\Pb\left(\Delta_{m'B} - \E[\Delta_{m'B}] < \alpha(J) -m'B D_1 \right) \overset{(b)}{\leq} \exp\left(4\right)\sum_{\btheta'\neq \btheta^*} \dfrac{\exp\left(-\dfrac{2 D_1^2 mcB}{2}\right)}{1 - \exp\left(-\dfrac{2 D_1^2 cB}{2}\right)} \label{eq:prop-batch-event2}
\end{align}
where, $(a)$ follows as $\E[\Delta_{m'B}] \geq m'B D_1$ for all $m' > \frac{\alpha(J)}{B D_1} + \frac{mc}{2}$, and $(b)$ follows using the same steps as in \Cref{lemma:supp:2}. 

Finally using \cref{eq:prop-batch-event1} and \cref{eq:prop-batch-event2} we can show that the Part B is bounded as follows:
\begin{align*}
    \Pb(m_{\btheta^*} > m) &\leq \Pb\left(\left\{\bigcup_{\btheta'\ne \btheta^*}\xi_{\btheta'\btheta^*}(m)\right\}\bigcap\left\{\tilde{m}_{\btheta^*}< \dfrac{\alpha}{B D_1} + \dfrac{mc}{2}\right\}\right) + \sum_{\btheta\neq\btheta'}\sum_{m':m' > \dfrac{\alpha(J)}{B D_1} + \dfrac{mc}{2}}\Pb\left(\tilde{\xi}_{\btheta'\btheta^*}(m')\right)\\
    &\overset{(a)}{\leq} \sum_{\btheta'\neq \btheta^*}\underbrace{\left[110 + 55\max\{1,\frac{\eta^2}{2 D_1^2 B}\}\right]}_{\colonequals C_1}\exp\left(-\underbrace{\dfrac{2 D_1^2 B \min\{(\nicefrac{c}{2}-1)^2, c\}}{\eta^2}}_{\colonequals C_2}m\right) \leq JC_1\exp\left(-C_2 m\right)
\end{align*}
where, $(a)$ follows from the same steps as in \Cref{conc:lemma:1}, and using \cref{eq:prop-batch-event1} and \cref{eq:prop-batch-event2}. Plugging this back in \cref{eq:prop-batch-main-1} we get that
\begin{align*}
    \E[m_\delta] &\leq 1 + (1+c)\left(\dfrac{\alpha(J)}{B D_1} +  \dfrac{\beta(J,\delta)}{B D_0}\right) + J\sum_{m : m >  \dfrac{\alpha(J)}{B D_1} +  \dfrac{\beta(J,\delta)}{B D_0}}C_1\exp\left(-C_2 m\right)\\
    &\overset{(b)}{\leq} 1 + 2\left(\dfrac{\eta\log\left(165 + \nicefrac{\eta^2}{B\eta_0^2}\right)\log J}{B D_1} + \dfrac{\log\left(1 + \nicefrac{\eta^2}{\eta_0^2 B}\right) + \log\left(\dfrac{J}{\delta}\right)}{B D_0}\right) + J^{1 + \dfrac{\log\left(165 + \nicefrac{\eta^2}{B\eta_0^2}\right)}{\log J}}\delta^{D_0/\eta^2}\\
    &\overset{(b)}{\leq} O\left(1 + \dfrac{ \eta \log (C)\log J}{B D_1} + \dfrac{\log (J/\delta)}{B D_0} + J C^{1/\delta}\delta^{D_0/\eta^2}\right)
\end{align*}
where, $(a)$ follows using the same steps as in Step 6 of \Cref{thm:chernoff-upper}, and in $(b)$ we substitute $C \colonequals \left(165 + \nicefrac{\eta^2}{\eta_0^2}\right) > \left(165 + \nicefrac{\eta^2}{B\eta_0^2}\right)$. Finally, the expected total number of samples is given by
\begin{align*}
    \E[\tau_{\delta}] \leq B\E[ m_{\delta}] = O\left(B +  \dfrac{\eta \log (C)\log J}{ D_1} + \dfrac{\log (J/\delta)}{ D_0} + B J C^{1/\delta}\delta^{D_0/\eta^2}\right).
\end{align*}
The claim of the proposition follows.
\end{proof}

\subsection{Sample Complexity Proof of \rcher}
\label{app:rcher-upper-bound}

\begin{customproposition}{3}\textbf{(Restatement)}
Let $\tau_\delta$, $D_0$, $C$ be defined as in \Cref{thm:chernoff-upper}, $D_e$ be defined as above, and $\epsilon_t \colonequals 1/\sqrt{t}$. Then the sample complexity bound of $\delta$-PAC \rcher with $\epsilon_t$ exploration is given by 
\begin{align*}
    E[\tau_\delta] \!\leq\! O\left(\frac{\eta\log(C)\log J}{D^{}_e} + \frac{\log(J/\delta)}{D_0} + JC^{1/\eta}\delta^{D_0/\eta^2} \right).
\end{align*}
\end{customproposition}

\begin{proof}
Recall that 
$$
D_1 \colonequals \min_{\{\mathbf{p}_{\btheta} : \btheta \in \bTheta\}} \min_{\btheta' \neq \btheta^\ast} \sum_{i=1}^n p_{\btheta}(i)(\mu_i(\btheta') - \mu_i(\btheta^*))^2.
$$
Define $D_e \colonequals  \min_{\btheta' \neq \btheta^\ast} \sum_{i=1}^n \frac{1}{n}(\mu_i(\btheta') - \mu_i(\btheta^*))^2$ as the objective value of uniform sampling optimization. 
Finally define the quantity at round $s$ as
\begin{align}
    D^{\epsilon_s}_1 \colonequals (1 - \epsilon_s)D_1 + \epsilon_s D_e \label{eq:d-alpha-e}
\end{align}
Let $\Delta_{t}(\btheta') \colonequals L_{t}(\btheta') - L_{t}(\btheta^*)$. Note that following Assumption 1 we can show that
\begin{align*}
    \Delta_{s}(\btheta') - D_1^{\epsilon_s} \overset{(a)}{\leq} 4\eta, \quad \Delta_{s}(\btheta') - D_0 \leq 4\eta
\end{align*}
where, $(a)$ follows as $D_1 \leq \eta$ and $D_e \leq \eta$ which implies $(1-\epsilon) D_1 + \epsilon D_e \leq \eta$ as $\epsilon \in (0,1)$. Now define the quantity
\begin{align*}
    \alpha(J) &\colonequals \dfrac{C' \log J}{D^{}_e}, \\
    \beta(J, \delta) &\colonequals \dfrac{C' + \log (J/\delta)}{D_0}
\end{align*}
where, the constant $C' \colonequals \log\left(1 + \nicefrac{\eta^2}{\eta_0^2}\right)$.
Now define the failure events
\begin{align*}
    \xi_{\btheta'\btheta^*}(t) &\colonequals \{L_{t}(\btheta') - L_{t}(\btheta^*) < \beta(J,\delta)\},\\
    \tilde{\xi}_{\btheta'\btheta^*}(t) &\colonequals \{L_{t}(\btheta') - L_{t}(\btheta^*) < 2\alpha(J)\}
\end{align*}
where, 
$\alpha(J) \colonequals \dfrac{C'\log J}{D_e}$, and $\beta(J, \delta) = \dfrac{C' + \log (J/\delta)}{D^{}_0}$.
Then we define the time $\tau_{\btheta^*}. \tilde{\tau}_{\btheta'\btheta^*}$ and $ \tilde{\tau}_{\btheta^*}$ as follows: 
\begin{align*}
    \tau_{\btheta^*} &\colonequals \min\{t: L_{t}(\btheta') - L_{t}(\btheta^*) > \beta(J,\delta),  \forall \btheta' \neq \btheta^* \}\\
    \tilde{\tau}_{\btheta'\btheta^*} &\overset{}{\colonequals} \min\{t: L_{t'}(\btheta') - L_{t'}(\btheta^*) > 2\alpha(J), \forall t' > t\}\\
    \tilde{\tau}_{\btheta^*} &\colonequals \max_{\btheta' \neq \btheta^*}\{\tilde{\tau}_{\btheta'\btheta^*}\}.
\end{align*}

Then we have that Assumption 2 is no longer required which can be shown as follows
\begin{align*}
    \E_{I^t,Y^t}[\Delta_{t}(\btheta')] = \E_{I^t,Y^t}[L_{t}(\btheta') - L_{t}(\btheta^*)] &= \E_{I^t,Y^t}\left[\sum_{s=1}^t(Y_s - \mu_{I_s}(\btheta^*))^2 - \sum_{s=1}^t(Y_s - \mu_{I_s}(\btheta))^2\right]\\
    & = \sum_{s=1}^t\E_{I_s}\E_{Y_s|I_s}\left[\left(\mu_{I_s}(\btheta^*) - \mu_{I_s}(\btheta)\right)^2|I_s\right]\\
    &= \sum_{s=1}^t \sum_{i=1}^n\Pb(I_s = i)\left(\mu_{i}(\btheta^*) - \mu_{i}(\btheta)\right)^2
\end{align*}
In the case when $D_1 \geq 0$, we can show that
\begin{align*}
    \E_{I^t,Y^t}[\Delta_{t}(\btheta')]&\overset{(a)}{\geq} \sum_{s=1}^t\left((1-\epsilon_s) D_1 + \epsilon_s D_e\right) \\
    & = \sum_{s=1}^{\tilde{\tau}_{\btheta^*}}(1-\epsilon_s)D_1 + \sum_{s=1}^{\tilde{\tau}_{\btheta^*}}\epsilon_s D_e + \sum_{s= \tilde{\tau}_{\btheta^*} + 1}^{t}(1-\epsilon_s)D_0 + \sum_{s= \tilde{\tau}_{\btheta^*} + 1}^{t}\epsilon_s D_e \\
    & \overset{(b)}{\geq} \sum_{s=1}^{\tilde{\tau}_{\btheta^*}}\epsilon_s D_e + \sum_{s= \tilde{\tau}_{\btheta^*} + 1}^{t}D_e -  \sum_{s= \tilde{\tau}_{\btheta^*} + 1}^{t}\epsilon_sD_0 + \sum_{s= \tilde{\tau}_{\btheta^*} + 1}^{t}\epsilon_s D_e \\
    & \overset{(c)}{\geq} 2\sqrt{\tilde{\tau}_{\btheta^*}}D_e - 2D_e + (t - \tilde{\tau}_{\btheta^*} - 1)D_e -kD_0  \\
    & \overset{}{=} (t - 1)D_e + 2\sqrt{\tilde{\tau}_{\btheta^*}}D_e - \tilde{\tau}_{\btheta^*}D_e - (kD_0 + 2D_e)\\
    &\overset{(d)}{\geq} (t - 1)D_e - \left(\dfrac{\alpha(J)}{D_e} + \dfrac{tc}{2}\right)D_e \\
    & = (t - 1)D_e - \alpha(J) - \frac{tc}{2}D_e\\
    & = \left(t - 1 - \frac{tc}{2}\right)D_e - \alpha(J)=  tD_e \left(1 - \frac{1}{t} - \frac{c}{2}\right) - \alpha(J)\\
    &\overset{(e)}{\geq}  tD_e\underbrace{\left(\frac{1}{2} - \frac{c}{2}\right)}_{c_1} - \alpha(J)\\
    & =  c_1tD_e- \alpha(J)
\end{align*}
where, $(a)$ follows due to the forced exploration definition, $(b)$ follows by dropping $D_1$, $(c)$ follows $\epsilon_s > \frac{1}{\sqrt{s}}$ and $\int_{1}^t \frac{1}{\sqrt{s}}ds = 2\sqrt{s} - 2$ and $D_0 \geq D_e$, and $(d)$ follows by definition of $\tilde{\tau}_{\btheta^*} = \frac{\alpha(J)}{D_e}  + \frac{tc}{2}$ and trivially assuming that $2\sqrt{\tilde{\tau}_{\btheta^*}}D_e - (kD_0+2D_e) > 0$ for large enough $\tilde{\tau}_{\btheta^*}$, and $(e)$ follows as $t\geq 2$. 
Next we can show that for $t\geq \tilde{\tau}_{\btheta^*}$
\begin{align*}
    \E[\Delta_{t}(\btheta')] \overset{(a)}{=} \sum_{s = \tilde{\tau}_{\btheta^*}}^t(1-\epsilon_s)D_0 + \sum_{s = \tilde{\tau}_{\btheta^*}}^t\epsilon_s D_e  &\overset{(b)}{\geq} \sum_{s = \tilde{\tau}_{\btheta^*}}^t(1-\epsilon_s)D_0 + \sum_{s = \tilde{\tau}_{\btheta^*}}^t\epsilon_s \frac{D_0}{n}\\
    &= (t - \tilde{\tau}_{\btheta^*}) D_0 - \sum_{s = \tilde{\tau}_{\btheta^*}}^t\epsilon_s D_0 + \sum_{s = \tilde{\tau}_{\btheta^*}}^t\epsilon_s \frac{D_0}{n}\\
    &= (t - \tilde{\tau}_{\btheta^*}) D_0 -  D_0\frac{n-1}{n}\sum_{s = \tilde{\tau}_{\btheta^*}}^t\epsilon_s \\
    &\geq (t - \tilde{\tau}_{\btheta^*})D_0 - kD_0\\
    &= (t - \tilde{\tau}_{\btheta^*} - k)D_0
\end{align*}
where, $(a)$ follows from the definition of exploration, and $(b)$ follows from as $D_0 \leq n D_e$, 
It follows that
\begin{align*}
    \Pb({\xi}_{\btheta',\btheta^*}(t)) 
    &= \Pb(\Delta_{t}(\btheta') < \beta(J,\delta))\\
    &\overset{}{=} \Pb(\Delta_{t}(\btheta') - \E[\Delta_{t}(\btheta')] < \beta(J,\delta) - \E[\Delta_{t}(\btheta')])\\
    &\overset{(a)}{\leq} \Pb\left(\Delta_{t}(\btheta') - \E[\Delta_{t}(\btheta')] < \beta(J,\delta) - (t- \tilde{\tau}_{\btheta^*} - k)D_0\right)\\
    &\overset{(b)}{=} \Pb\left(\Delta_{t}(\btheta') - \E[\Delta_{t}(\btheta')] < D_0\left(\dfrac{ \beta(J,\delta)}{D_0} - t + \tilde{\tau}_{\btheta^*} + k)\right)\right)
\end{align*}

Once we define the failure events we can follow the same proof technique as Theorem 2 and show that
\begin{align*}
    \Pb(\tilde{\xi}_{\btheta'\btheta^*}(t)) 
    & \overset{(a)}{\leq} \Pb(\Delta_{t}(\btheta') - \E[\Delta_{t}(\btheta')] < 2\alpha(J) - c_1tD_e)\\
\end{align*}
where, in $(a)$ the choice of $c_1tD_1$ follows as $\E[\Delta_{t}(\btheta')] \geq c_1tD_e$ as $D_e > 0$. 
It follows from \Cref{conc:lemma:1} since $D_0 \geq D_e$ and $D_e > 0$ that the probability of the failure event is bounded as follows:
\begin{align*}
    &\Pb(\tau_{\btheta^*} > t) \leq \Pb(\bigcup_{\btheta'\neq \btheta^*} \xi_{\btheta'\btheta^*}(t)) \\
    &= \Pb\left(\left\{\bigcup_{\btheta'\neq \btheta^*}\xi_{\btheta'\btheta^*}(t)\right\}\bigcap\{\tilde{\tau}_{\btheta^*} < \dfrac{\alpha(J)}{D^{}_e} + \frac{tc}{2}\}\right) + \Pb\left(\bigcup_{\btheta'\neq \btheta^*}\left\{\xi_{\btheta'\btheta^*}(t)\right\}\bigcap\{\tilde{\tau}_{\btheta^*} \geq \dfrac{\alpha(J)}{D^{}_e} + \frac{tc}{2}\}\right) \\
    &\leq \sum_{\btheta'\neq \btheta^*}\Pb\left(\left\{\xi_{\btheta'\btheta^*}(t)\right\}\bigcap\{\tilde{\tau}_{\btheta^*} < \dfrac{\alpha(J)}{D^{}_e} + \frac{tc}{2}\}\right) + \sum_{\btheta'\neq \btheta^*}\sum_{t':t' \geq \frac{\alpha(J)}{D^{}_e} + \frac{tc}{2} }\Pb\left(\tilde{\xi}_{\btheta'\btheta^*}(t')\right)
    \overset{(a)}{\leq} J C_1\exp\left(-C_2t\right)
\end{align*}
where, in $(a)$ we substitute $C_1 \colonequals 110 + 55\max\left\{1, \dfrac{\eta^2}{2 (D^{}_e)^2 }\right\}$, $C_2 \colonequals \dfrac{ 2(D^{}_e)^2\min\{(\nicefrac{c}{2}-1)^2, c\}}{\eta^2}$, $c > 0$ is a constant, and $J \colonequals |\bTheta|$. 
Now we define the critical number of samples as follows:
\begin{align*}
    M \colonequals \left(\dfrac{C' \log J}{D^{}_e} + \dfrac{C' + \log(J/\delta)}{D_0} \right)
\end{align*}
where, $C' = \log\left(1 + \nicefrac{\eta^2}{\eta_0^2}\right)$. Then we can bound the sample complexity for some constant $c > 0$ as follows:
\begin{align*}
    \E[\tau_\delta] &\leq \E[\tau_{\btheta^{*}}] = \sum_{t=0}^{\infty}\Pb(\tau_{\btheta^*} > t) 
    \overset{}{\leq} 1 + (1 + c)M + \sum_{t: t > (1+c)M}\Pb(\tau_{\btheta^*} > t) \\
    &\overset{(a)}{\leq} O\left(\dfrac{\eta\log(C)\log J}{D^{}_e} + \dfrac{\log(J/\delta)}{D^{}_0} + J(C)^{1/\eta}\delta^{D^{}_0/\eta^2} \right)
\end{align*}
where, in $(a)$ we substitute $\log(C) = \log(165 + \frac{\eta^2}{\eta_0^2}) > C'$ and the rest follows as $\log(C)/D^{}_0 < \log(C)/D^{}_e$.
\end{proof}

\subsection{Minimax Optimality Proof (\Cref{thm:minimax})}
\label{app:minimax}

\begin{example}
\label{ex:minimax}
We define an environment model $B_j$ consisting of $N$ actions and $J$ hypotheses with true hypothesis $\btheta^* = \btheta_j$ ($j$-th column) 
as follows:
\begin{align*}
\begin{matrix} 
    \btheta &= & \btheta_1 &\btheta_2  &  \btheta_3 & \ldots & 
    \btheta_J \\\hline
    \mu_1(\btheta) &=  & \Gamma & \Gamma\!-\!\frac{\Gamma}{J} & \Gamma\!-\!\frac{2\Gamma}{J} & \ldots & \Gamma\!-\!\frac{(J-1)\Gamma}{J}\\
    \mu_2(\btheta) &=   & \iota_{21} & \iota_{22} & \iota_{23} & \ldots & \iota_{2J}\\
    &\vdots & &&\vdots\\
    \mu_n(\btheta) &=   & \iota_{n1} & \iota_{n2} & \iota_{n3} & \ldots & \iota_{nJ}
\end{matrix}
\end{align*}
where, each $\iota_{ij}$ is distinct and satisfies $\iota_{ij} < \Gamma/4J$.
Note that we introduce such $\iota_{ij}$ for different hypotheses so as not to violate Assumption \ref{assm:positive-D1}. 
\(\btheta_1\) is the optimal hypothesis in \(A_1\), \(\btheta_2\) is the optimal hypothesis in \(A_2\) and so on such that for each $A_j$ and $j\in[J]$ we have column $j$ as the optimal hypothesis. 
\end{example}

\begin{customtheorem}{2}\textbf{(Restatement)}
Any $\delta$-PAC policy $\pi$ that identifies $\btheta^\ast$ 
in \eqref{eq:minimax-environment} 
satisfies $\E[\tau_\delta] \geq \Omega\left({J^2\Gamma^{-2}} \log({1}/{\delta})\right)$. 
Applying Theorem~\ref{thm:chernoff-upper} to the same environment, the sample complexity of \cher is  $O\left(J^2\Gamma^{-2}\log(J/\delta)\right)$ which matches the lower bound upto log factors.\end{customtheorem}

\begin{proof}
The proof follows the standard change of measure argument. We follow the proof technique in Theorem 1 of \citet{huang2017structured}. We first state a problem setup in Example \ref{ex:minimax}.

Let, \(\Lambda_1\) be the set of alternate  models having a different optimal hypothesis than \(\btheta^{*} = \btheta_1\) such that all models having different optimal hypothesis than $\btheta_1$ such as  $A_2, A_3, \ldots A_J$ 
are in $\Lambda_1$. Let $\tau_\delta$ be the stopping time for any $\delta$-PAC policy $\pi$.
Let $Z_i(t)$ denote the number of times the action $i$ has been sampled till round $t$.
Let $\widehat{\btheta}(t)$ be the predicted optimal hypothesis at round $\tau_\delta$. We first consider the  model $A_1$. Define the event \(\xi=\{\widehat{\btheta}(t) \neq \btheta^*\}\) as the error event in model $A_1$. Let the event \(\xi'=\{\widehat{\btheta}(t) \neq \btheta^{'*}\}\) be the corresponding error event in model $A_2$. Note that $\xi^{\complement} \subset \xi'$. Now since $\pi$ is $\delta$-PAC policy we have $\Pb_{A_1,\pi}(\xi) \leq \delta$ and $\Pb_{A_2,\pi}(\xi^{\complement}) \leq \delta$. Hence we can show that,

\begin{align}
2 \delta \geq \Pb_{A_1, \pi}(\xi)+ \Pb_{A_2, \pi}(\xi^{\complement})
&\overset{(a)}{\geq} \frac{1}{2} \exp\left(-\KL\left(P_{A_1, \pi} || P_{A_2, \pi}\right)\right)\nonumber\\
 \KL\left(P_{A_1, \pi} || P_{A_2, \pi}\right) & \geq \log\left(\dfrac{1}{4\delta}\right)\nonumber\\
 \sum_{i=1}^{n} \E_{A_1, \pi}[Z_i(\tau_\delta)]\cdot\left(\mu_{i}(\btheta^*)^{}-\mu_{i}( \btheta^{'*})^{}\right)^2&\overset{(b)}{\geq} \log\left(\dfrac{1}{4\delta}\right)\nonumber\\
 \left(\Gamma - \Gamma +\frac{\Gamma}{J}\right)^2\E_{A_1, \pi}[Z_1(\tau_\delta)] + \sum_{i=2}^n(\iota_{i1} - \iota_{i2})^2\E_{A_1, \pi}[Z_i(\tau_\delta)] &\overset{(c)}{\geq} \log\left(\dfrac{1}{4\delta}\right)\nonumber\\
 \left(\dfrac{1}{J}\right)^2\Gamma^2\E_{A_1, \pi}[Z_1(\tau_\delta)] + \sum_{i=2}^n(\iota_{i1} - \iota_{i2})^2\E_{A_1, \pi}[Z_i(\tau_\delta)] &\overset{}{\geq} \log\left(\dfrac{1}{4\delta}\right)\nonumber\\
 \left(\dfrac{1}{J}\right)^2\Gamma^2\E_{A_1, \pi}[Z_1(\tau_\delta)] + \sum_{i=2}^n\frac{\Gamma^2}{4J^2}\E_{A_1, \pi}[Z_i(\tau_\delta)] &\overset{(d)}{\geq} \log\left(\dfrac{1}{4\delta}\right)\label{eq:minimax-1}
\end{align}
where, $(a)$ follows from Lemma \ref{lemma:tsybakov}, $(b)$ follows from Lemma \ref{lemma:divergence-decomp},  $(c)$ follows from the construction of the bandit environments, and $(d)$ follows as  $(\iota_{ij} - \iota_{ij'})^2 \leq \frac{\Gamma^2}{4J^2}$ for any $i$-th action and $j$-th hypothesis pair.

Now, we consider the alternate model $A_3$. Again define the event \(\xi=\{\widehat{\btheta}(t) \neq \btheta^*\}\) as the error event in model $A_1$ and the event  \(\xi'=\{\widehat{\btheta}(t) \neq \btheta^{''*}\}\) be the corresponding error event in model $A_3$. Note that $\xi^{\complement} \subset \xi'$. Now since $\pi$ is $\delta$-PAC policy we have $\Pb_{B_1,\pi}(\xi) \leq \delta$ and $\Pb_{A_3,\pi}(\xi^{\complement}) \leq \delta$. Following the same way as before we can show that,
\begin{align}
 \left(\dfrac{2}{J}\right)^2\Gamma^2\E_{A_1, \pi}[Z_1(\tau_\delta)] + \sum_{i=2}^n\frac{\Gamma^2}{4J^2}\E_{A_1, \pi}[Z_i(\tau_\delta)] &\overset{(d)}{\geq} \log\left(\dfrac{1}{4\delta}\right)\label{eq:minimax-2}.
\end{align}
Similarly we get the equations for all the other $(J-2)$ alternate models in $\Lambda_1$. Now consider an optimization problem 
\begin{align*}
    &\min_{x_i : i \in [n]} \sum x_i\\
    s.t. \quad & \left( \frac{1}{J}\right)^2\Gamma^2 x_1 + \frac{\Gamma^2}{4J^2} \sum_{i=2}^n x_i \geq \log(1/4\delta)\\
    &\left(\frac{2}{J}\right)^2\Gamma^2 x_1 + \frac{\Gamma^2}{4J^2} \sum_{i=2}^n x_i \geq \log(1/4\delta)\\
    &\vdots\\
    &\left(\frac{J-1}{J}\right)^2\Gamma^2 x_1 + \frac{\Gamma^2}{4J^2} \sum_{i=2}^n x_i \geq \log(1/4\delta)\\
    &x_i \geq 0,  \forall i\in [n]
\end{align*}
where the optimization variables are $x_i$. 
It can be seen that the optimum objective value is $J^2\Gamma^{-2} \log(1/4\delta)$. Interpreting $x_i = \mathbb{E}_{A_1,\pi}[Z_i(\tau_{\delta})]$ for all $i$, we get that $\E_{A_1,\pi}[\tau_\delta] = \sum_{i}x_i$ which gives us the required lower bound.
Let $\mathbf{p}_{\btheta^*}$ be the sampling p.m.f for the environment $A_1$ for verifying $\btheta^*$. We also know that the Chernoff verification in  \eqref{eq:opt-lower00} has a nice linear programming formulation as stated in \eqref{eq:opt:linear0}. Using that for $A_1$ we can show that,
\begin{align*}
    & \max z \\
    s.t.\quad &  \mathbf{p}_{\btheta^*}(1)\Gamma^2\left( \frac{1}{J}\right)^2 + \sum_{i=2}^n \mathbf{p}_{\btheta^*}(i)(\iota_{i\btheta^*} - \iota_{i\btheta'})^2 \geq z\\
    & \mathbf{p}_{\btheta^*}(1)\Gamma^2\left( \frac{2}{J}\right)^2 + \sum_{i=2}^n \mathbf{p}_{\btheta^*}(i)(\iota_{i\btheta^*} - \iota_{i\btheta''})^2 \geq z\\
    &\vdots\\
    & \mathbf{p}_{\btheta^*}(1)\Gamma^2\left(\frac{J-1}{J}\right)^2 + \sum_{i=2}^n \mathbf{p}_{\btheta^*}(i)(\iota_{i\btheta^*} - \iota_{i\btheta'''})^2 \geq z\\
    &\sum_{i=1}^n \mathbf{p}_{\btheta^*}(i) = 1.
\end{align*}
We can relax this further by noting that 
$(\iota_{i\btheta^*} - \iota_{i\btheta'})^2 \leq \frac{\Gamma^2}{4J^2}$ for all $i\in[n]$ and $\btheta^*,\btheta'\in\bTheta$. Hence,
\begin{align*}
    \max z \quad
    s.t. \quad & \mathbf{p}_{\btheta^*}(1)\Gamma^2\left( \frac{1}{J}\right)^2 + \sum_{i=2}^n \mathbf{p}_{\btheta^*}(i)\frac{\Gamma^2}{4J^2} \geq z\\
    & \mathbf{p}_{\btheta^*}(1)\Gamma^2\left(\frac{2}{J}\right)^2 + \sum_{i=2}^n \mathbf{p}_{\btheta^*}(i)\frac{\Gamma^2}{4J^2} \geq z\\
    &\vdots\\
    & \mathbf{p}_{\btheta^*}(1)\Gamma^2\left(\frac{J-1}{J}\right)^2 + \sum_{i=2}^n \mathbf{p}_{\btheta^*}(i)\frac{\Gamma^2}{4J^2} \geq z\\
    &\sum_{i=1}^n \mathbf{p}_{\btheta^*}(i) = 1.
\end{align*}
Solving this above optimization gives that $\mathbf{p}_{\btheta^*}(1) = 1$, and $\mathbf{p}_{\btheta^*}(2) = \mathbf{p}_{\btheta^*}(3) =\ldots= \mathbf{p}_{\btheta^*}(n) =0$. Similarly, for verifying any hypothesis $\btheta\in \bTheta$ we can show that the verification proportion is given by $\mathbf{p}_{\btheta} = (1,\underbrace{0,0,\ldots,0}_{\text{(J-1) zeros}})$. This also shows that for Example \ref{ex:minimax}, 
\begin{align*}
    D_0 &= \min_{\btheta' \neq \btheta^\ast} \sum_{i=1}^n p_{\btheta^*}(i) (\mu_i(\btheta') -\mu_i(\btheta^\ast))^2 = p_{\btheta^*}(1)\Gamma^2\left( \frac{1}{J}\right)^2 +  \sum_{i=2}^np_{\btheta^*}(i)(\iota_{i\btheta^*}-\iota_{i\btheta'})^2 = \dfrac{\Gamma^2}{J^2}\\
    D_1 &= \min_{\{\mathbf{p}_{\btheta} : \btheta \in \bTheta\}} \min_{\btheta' \neq \btheta^\ast} \sum_{i=1}^n p_{\btheta}(i)(\mu_i(\btheta') - \mu_i(\btheta^*))^2 \overset{(a)}{\geq} p_{\btheta}(1)\Gamma^2\left( \frac{1}{J}\right)^2 +  \sum_{i=2}^np_{\btheta}(i)(\iota_{i\btheta}-\iota_{i\btheta'})^2 = \dfrac{\Gamma^2}{J^2}
\end{align*}
where, $(a)$ follows as the verification of any hypothesis $\btheta$ is a one hot vector $p_{\btheta} = (1,\underbrace{0,0,\ldots,0}_{\text{(J-1) zeros}})$. Note that $\eta/4 = \Gamma^2$. Plugging this in Theorem \ref{thm:chernoff-upper} gives us that the upper bound of \cher as
\begin{align*}
    \E[\tau_\delta] &\leq O\left(\dfrac{J^2 \Gamma^2 \log(\Gamma^4/\eta_0^4) \Gamma\log J}{\Gamma^2} + \dfrac{J^2\log(J/\delta)}{\Gamma^2} + J\log(\Gamma^4/\eta_0^4)^{1/\Gamma^2}\delta^{\Gamma^2/(J^2\Gamma^4)}\right)\\
    &\leq O\left(J^2  \log(\Gamma^4/\eta_0^4) \Gamma\log J + \dfrac{J^2\log(J/\delta)}{\Gamma^2}\right) \leq O\left(\dfrac{J^2\log(J/\delta)}{\Gamma^2}\right).
\end{align*}
The claim of the theorem follows.
\end{proof}

\subsection{Proof of \Cref{thm:max-min-eig} (Continuous hypotheses)}
\label{app:thm:max-min-eig}
\begin{customtheorem}{3} \textbf{\textit{(Restatement)}}
Assume that $\mu_i(\btheta)$ for all $i \in [n]$ 
is a differentiable function, 
and the set 
$\{\nabla \mu_i (\wtheta(t)) : i \in [n]\}$ 
of gradients evaluated at $\wtheta(t)$ span 
$\mathbb{R}^d$. 
Consider a p.m.f. $\mathbf{p}_{\wtheta(t), r}$ 
from \eqref{eq:albert-opt} for verifying $\wtheta(t)$ against 
all alternatives in $\mathcal{B}^\complement_r(\wtheta(t))$. 
The limiting value of $\mathbf{p}_{\wtheta(t), r}$ as $r \rightarrow 0$ is
\begin{align*}
\mathbf{p}_{\wtheta(t)}:=\arg\max_{\text{ }\mathbf{p}}  
\lambda_{\min}
\left(
\sum_{i=1}^{n} p(i) \nabla\mu_i(\wtheta(t))
\nabla\mu_i(\wtheta(t))^T
\right).
\end{align*}
\end{customtheorem}
\begin{proof}
For brevity we drop the $\wtheta(t)$ argument from the notation for the closed ball $\mathcal{B}_r$ and its complement $\mathcal{B}_r^\complement$ centered at $\wtheta(t)$.  Denoting 
$g_i(\btheta) \colonequals (\mu_i(\btheta) - \mu_i(\widehat{\btheta}(t)))^2$ for any generic $\btheta$,
we can rewrite the optimization using a probability density function (p.d.f.) 
$\mathbf{q}$ over parameters in $\bTheta$, as explained subsequently.
\begin{align}
\mathbf{p}_{\widehat{\btheta}(t), r} &=\argmax_{\text{ }\mathbf{p}} \inf_{\btheta \in \mathcal{B}^\complement_r} \sum_{i=1}^n p(i) g_i(\btheta) \nonumber\\
&= 
\argmax_{\text{ } \mathbf{p}} \inf_{\mathbf{q}: q(\btheta) = 0 \forall \btheta \in \mathcal{B}_{r}} 
\int_{\bTheta} q(\btheta) \sum_{i=1}^{n} p(i) 
g_i(\btheta) d\btheta. \label{eq:pdf-opt2}
\end{align}
The above equality is true because if in the LHS, 
the inner infimum was attained at a $\btheta \in \bTheta$, 
then the same value can be attained in the RHS by a 
degenerate p.m.f. $\mathbf{q}$ that puts all its mass on that $\btheta$. 
Conversely, suppose the inner infimum in the RHS was attained at 
a p.d.f. $\mathbf{q}^\ast$. 
Since the objective function is a linear in $\mathbf{q}$, 
the objective value is the same for a degenerate pdf that puts 
all its mass on one of the support points of $\mathbf{q}^\ast$, 
and this value can also be attained by the LHS infimum. 

Let $\kappa_i = \sup_{\btheta \in \mathcal{B}_r} g_i(\theta), \kappa = \max_i \kappa_i$, then $\kappa \rightarrow 0$ as $r \rightarrow 0$. Since $g_i(\wtheta(t))=0$ for all $i$ we have that $\kappa > 0$. 
Consider a family $\mathcal{Q}_r$ of pdfs 
supported on the boundary of 
$\mathcal{B}_r$, i.e.,
\begin{align*}
\mathcal{Q}_r \colonequals \left\lbrace \mathbf{q}: \int_\bTheta q(\btheta) d\btheta = 1, 
q(\btheta') = 0 \text{ if } \lVert \btheta' 
- \widehat{\btheta}(t)\rVert \neq \epsilon
\right \rbrace.
\end{align*}
For any pmf $\mathbf{p}$, 
the suboptimality gap in \eqref{eq:pdf-opt2}
by restricting the infimum to be over the set 
$\mathcal{Q}_r$ is non-negative, and is upper 
bounded by
\begin{align}
\label{eq:kappa}
\inf_{\mathbf{q} \in \mathcal{Q}_r} \sum_{i=1}^{n} 
p(i) \int_{\bTheta} q(\btheta) g_i(\btheta) d\btheta 
\leq \sum_{i=1}^{n} 
p(i) \int_{\btheta: \lVert\btheta - \widehat{\btheta}(t)\rVert = r} q(\btheta) \kappa d\btheta \leq \kappa, 
\end{align}
where the first inequality is true for any $q(\btheta) \in \mathcal{Q}_r$ and the second inequality is true because $\kappa$ is an upper bound to the integrand at any point on the surface of $\mathcal{B}_r$. 
For any $r > 0$, we have that 
\begin{align*}
0 \leq
\inf_{\mathbf{q}:q(\btheta)=0 \forall \theta \in \mathcal{B}_r} \int_\bTheta q(\btheta) \sum_{i=1}^n p(i) g_i(\btheta) d\btheta 
\leq 
\inf_{q \in \mathcal{Q}_r} \int_\bTheta q(\btheta) \sum_{i=1}^n p(i) g_i(\btheta) d\btheta
\leq \kappa,
\end{align*}
where the first inequality is due to $g_i(\btheta) \geq 0$, the second inequality is because the domain of $\inf$ is reduced, and the third inequality is by \eqref{eq:kappa}. 
As $r \rightarrow 0$, the quantity $\kappa \rightarrow 0$ and the suboptimality 
gap also tends to zero. Hence for any pmf $\mathbf{p}$, 
\begin{align*}
\lim_{r \rightarrow 0} \inf_{\mathbf{q}: q(\btheta)=0 \forall \btheta \in \mathcal{B}_{r}} 
\int_\bTheta q(\btheta) \sum_{i=1}^{n} p(i) g_i(\btheta) d\btheta = \lim_{r \rightarrow 0}\inf_{\mathbf{q} \in \mathcal{Q}_r} \int_\bTheta q(\btheta) \sum_{i=1}^{n} p(i) g_i(\btheta) d\btheta.
\end{align*}
Since the above is true for any $\mathbf{p}$, it also holds for the maximizer of the infimum at each value of $r$ in the convergent series to $0$. Hence, we have that
\begin{align*}
\lim_{r \rightarrow 0} \mathbf{p}_{\wtheta(t), r} 
&= 
\lim_{r \rightarrow 0} \arg\max_{\mathbf{p}} 
\inf_{\mathbf{q}: q(\btheta)=0 \forall \btheta \in \mathcal{B}_{r}} 
\int_\bTheta q(\btheta) \sum_{i=1}^{n} p(i) g_i(\btheta) d\btheta\\
&=
\lim_{r \rightarrow 0} \arg\max_{\mathbf{p}} \inf_{\mathbf{q} \in \mathcal{Q}_r} \int_\bTheta q(\btheta) \sum_{i=1}^{n} p(i) g_i(\btheta) d\btheta.
\end{align*}
Consider the multivariable Taylor series of $g_i$ 
around $\widehat{\btheta}(t)$, for a $\btheta \in \mathcal{B}_r$.
\begin{align*}
g_i(\btheta) &= g_i(\widehat{\btheta}(t)) + (\btheta - \widehat{\btheta}(t))^T \nabla g_i(\widehat{\btheta}(t))\\
&\mspace{18mu} + 0.5 (\btheta - \widehat{\btheta}(t))^T \nabla^2 g_i(\widehat{\btheta}(t)) (\btheta - \widehat{\btheta}(t)) + o(\lVert\btheta - \widehat{\btheta}(t)\rVert^3).
\end{align*}
For indices $j,k \in [d]$ we can evaluate 
\begin{align*}
\nabla g_i(\btheta) &= \begin{bmatrix}
\frac{\partial g_i}{\partial \btheta_j} (\btheta)
\end{bmatrix}_j 
= \begin{bmatrix}
2(\mu_i(\btheta) - \mu_i(\widehat{\btheta}(t))) \frac{\partial \mu_i}{\partial \btheta_j}(\btheta)
\end{bmatrix}_j, \text{ and}\\
\nabla^2 g_i(\btheta) &= \begin{bmatrix}
\frac{\partial^2 g_i}{\partial \btheta_j \partial \btheta_k}(\btheta)
\end{bmatrix}_{j, k} \\
&= \begin{bmatrix}
2(\mu_i(\btheta) - \mu_i(\widehat{\btheta}(t))) \frac{\partial^2 \mu_i}{\partial \btheta_j \partial \btheta_k} (\btheta) + 2 \frac{\partial \mu_i}{\partial \btheta_j} (\btheta) \frac{\partial \mu_i}{\partial \btheta_k} (\btheta)
\end{bmatrix}_{j,k}
\end{align*}
giving that $\nabla g_i(\widehat{\btheta}(t)) = 0$ and 
$\nabla^2 g_i(\widehat{\btheta}(t)) = 2\nabla \mu_i (\widehat{\btheta}(t)) \nabla \mu_i (\widehat{\btheta}(t))^T$. 
Then $\mathbf{p}_{\wtheta(t), r}$ is the solution to the following:
\begin{align*}
&\max_{\mathbf{p}} \inf_{\mathbf{q} \in \mathcal{Q}_r} \int_{\Theta} q(\btheta) \sum_{i=1}^n p(i)g_i(\btheta)d\btheta \\
&= \max_{\mathbf{p}} \inf_{\mathbf{q}\in \mathcal{Q}_r} \int_{\btheta: \lVert\btheta - \widehat{\btheta}(t)\rVert = r} q(\btheta)
\sum_{i=1}^{n} p(i)
\left(
(\btheta - \widehat{\btheta}(t))^T \nabla \mu_i (\widehat{\btheta}(t)) \nabla \mu_i (\widehat{\btheta}(t))^T (\btheta - \widehat{\btheta}(t))
+ o(\lVert \btheta - \wtheta(t)\rVert^3)\right)
d\btheta\\
&= \max_{\mathbf{p}} \inf_{\mathbf{q}\in \mathcal{Q}_r}
\int_{\btheta: \lVert\btheta - \widehat{\btheta}(t)\rVert = r} \frac{(\btheta - \widehat{\btheta}(t))^T\sum_{i=1}^{n} p(i) \nabla \mu_i (\widehat{\btheta}(t)) \nabla \mu_i (\widehat{\btheta}(t))^T(\btheta - \widehat{\btheta}(t))}{(\btheta - \widehat{\btheta}(t))^T(\btheta - \widehat{\btheta}(t))} q(\btheta) \lVert \btheta - \wtheta(t) \rVert^2 d\btheta + o(r^3) \\
&= \max_{\mathbf{p}} \text{min eigenvalue}\left( 
\sum_{i=1}^{n} p(i) \nabla \mu_i (\widehat{\btheta}(t)) \nabla \mu_i (\widehat{\btheta}(t))^T \right) r^2 + o(r^3).
\end{align*}
The last equality uses the variational characterization of the minimum eigenvalue of a matrix and the fact that the $\inf$ would put all its mass on the $\btheta$ aligned with the corresponding eigenvector to attain the minimum value. 
In the limit $r \rightarrow 0$, the second term is insignificant compared to the first and we get the required result.
\end{proof}

\subsubsection{How to solve the optimization}
\label{app:optimization-todo}
The optimization in Theorem~\ref{thm:max-min-eig} can be solved using convex optimization software. This is because the objective function, i.e., the minimum eigenvalue function is a concave function of the matrix argument, and the domain of the optimization $\{\mathbf{p} : \sum_{i=1}^n p(i) = 1, p(i)\geq 0 \forall i \in [n]\}$ is a convex set. Hence we can maximize the objective over the domain. The set of gradients $\{\nabla \mu_i(\widehat{\btheta}(t)) : i \in [n]\}$ span $\mathbb{R}^d$. Hence the optimal objective value is positive. Note that the verification proportions are the solution to a convex optimization problem. So we can terminate it early to get an approximate solution. 
Ignoring accuracy factors, a solution can be obtained in $O((n^3 + n^2d^2 + nd^3)\sqrt{n+d})$ operations \citep{nesterov1994interior}.

\subsection{\cher Convergence Proof for Smooth Hypotheses Space}
\label{app:cher-smooth-upper}

\subsubsection{Theoretical Comparisons for Active Regression}
\label{app:active-regression-discussion}
From the result of \citet{chaudhuri2015convergence} we can show that the \actives algorithm enjoys a convergence guarantee as follows:
\begin{align}
    \mathbb{E}\left[P_{U}\left(\wtheta(t)\right)-P_{U}\left(\theta^{*}\right)\right] 
    \leq O\left( \frac{\sigma^2_{U}\sqrt{\log(dt)}}{t}+\frac{R}{t^{2}}\right) \label{eq:kamalika-result}
\end{align}
where, $P_U$ is the loss under uniform measure, $R$ is the maximum loss under any measure for any $\btheta\in\bTheta$, and $\sigma^2_U$ is defined as 
\begin{align}
    \sigma^{2}_U & \overset{(a)}{=} \frac{1}{t^{2}} \operatorname{Trace}\left[\left(\sum_{s=1}^{t} \sum_{s'=1}^{t} \E_{\ell_{s} \sim D}\left[\nabla \ell_{s}\left(\btheta^{*}\right)\right]\E_{\ell_{s'} \sim D} \left[ \nabla \ell_{s'}\left(\btheta^{*}\right)^{\top}\right]\right) \ I_{\Gamma}\left(\btheta^{*}\right)^{-1} \ I_{U}\left(\btheta^{*}\right) \ I_{\Gamma}\left(\btheta^{*}\right)^{-1}\right]\nonumber\\
    & = \frac{1}{t^{2}} \operatorname{Trace}\left[\left(\sum_{s=1}^{t} \sum_{s'=1}^{t} \E_{\ell_{s} \sim D}\left[\nabla \ell_{s}\left(\btheta^{*}\right)\right]\E_{\ell_{s'} \sim D} \left[ \nabla \ell_{s'}\left(\btheta^{*}\right)^{\top}\right]\right) \ I_{\Gamma}\left(\btheta^{*}\right)^{-1} I_{U}\left(\btheta^{*}\right) \ I_{\Gamma}\left(\btheta^{*}\right)^{-1}\right]\nonumber\\
    & \overset{(b)}{=} \frac{1}{t^{2}} \operatorname{Trace}\left[\left(\sum_{s=1}^{t}\sum_{s'=1}^{t}  I_{\Gamma}\left(\btheta^{*}\right)\right)  I_{\Gamma}\left(\btheta^{*}\right)^{-1} I_{U}\left(\btheta^{*}\right)  I_{\Gamma}\left(\btheta^{*}\right)^{-1}\right]\nonumber\\
    &= \operatorname{Trace}\left[  I_{U}\left(\btheta^{*}\right)  I_{\Gamma}\left(\btheta^{*}\right)^{-1}\right] \label{eq:sigma-U}
\end{align}
where, in $(a)$ the loss $\ell_{s}$ and $\ell_{s'}$ are i.i.d drawn from the same distribution $D$, and $I_\Gamma$, $I_{U}$ are the Fisher information matrix under the sampling distribution $\Gamma$ and $U$, and  $(b)$ follows from Lemma 5 of \citet{chaudhuri2015convergence}. Note that \actives is a two-stage process that samples according to the uniform distribution $U$ to build an estimate of $\btheta^*$ and then solves an SDP to build the sampling proportion $\Gamma$ that minimizes the quantity $\sigma_U^2$ and follows that sampling proportion $\Gamma$ for the second stage. It follows that
\begin{align*}
    \sigma^{2}_U &= \operatorname{Trace}\left[  I_{U}\left(\btheta^{*}\right) I_{\Gamma}\left(\btheta^{*}\right)^{-1}\right] \\
    &\leq \lambda_{\max}(I_{U}\left(\btheta^{*})\right)\lambda_{\max}(I_{\Gamma}\left(\btheta^{*})^{-1}\right)d  \leq \dfrac{\lambda_{1}}{\lambda_{\min, \actives}}d C_3 \eta.
\end{align*}
where $\lambda_{\min, \actives} = \lambda_{\min}(I_{\Gamma}\left(\btheta^{*})^{}\right)$. We also have that
\begin{align*}
    I_U(\btheta^*) &= \E_{I_s\sim U }\nabla^2_{\btheta = \btheta^*} \ell_{s}(\btheta) = \E_{I_s\sim U }\nabla^2_{\btheta = \btheta^*}(Y_s - \mu_{I_s}(\btheta))^2\\
    &=\E_{I_s\sim U }2\left(Y_s - \mu_i(\btheta^*)\right)\nabla^2 \mu_{I_s}(\btheta^*) - 2\nabla \mu_{I_s}(\btheta^*)\nabla \mu_{I_s}(\btheta^*)^T\\
    &= 2\sum_{i=1}^n p_{\mathbf{unif}}(i)\left[\left(Y_s - \mu_i(\btheta^*)\right)\nabla^2 \mu_{I_s}(\btheta^*) - \nabla \mu_{I_s}(\btheta^*)\nabla \mu_{I_s}(\btheta^*)^T\right]
\end{align*}
This leads to the following bound on the maximum eigenvalue of the matrix $I_U(\btheta^*)$
\begin{align*}
    \lambda_{\max}[I_U(\btheta^*)] \leq \lambda_{\max}[\sum_{i=1}^n p_{\mathbf{unif}}(i)\left(Y_s - \mu_i(\btheta^*)\right)\nabla^2 \mu_{I_s}(\btheta^*)]&\leq \lambda_1 C_3\eta.
\end{align*}
Plugging this in the statement of the result in \cref{eq:kamalika-result} we get that
\begin{align*}
    \mathbb{E}\left[P_{U}\left(\wtheta(t)\right)-P_{U}\left(\theta^{*}\right)\right] 
    \leq O\left( \frac{d\sqrt{\log(dt)}}{t}+\frac{R}{t^{2}}\right)
\end{align*}
where we are only concerned with the scaling with the dimension $d$. 
Comparing this to our result in \Cref{thm:cher-smooth} we have the following convergence rate
\begin{align}
    \E\left[P_t(\wtheta_t)-P_t\left(\btheta^{*}\right)\right] \leq\left(1+\rho_t\right) \frac{\sigma^2_t}{t}+\frac{R}{t^{2}} \label{eq:our-result}
\end{align}
where $P_t$ is a worst case measure over the data points. Next, we can show that,
\begin{align*}
    \sigma^2_t \colonequals \E\left[\left\|\nabla \wP_{t}\left(\btheta^{*}\right)\right\|_{\left(\nabla^{2} P_{t}\left(\btheta^{*}\right)\right)^{-1}}^{2}\right] &= \E\left[\nabla \wP_{t}\left(\btheta^{*}\right)^T \nabla^2 P_{t}\left(\btheta^{*}\right)^{-1}\nabla \wP_{t}\left(\btheta^{*}\right)\right] \\
    &\overset{(a)}{\leq} \E\left[\lambda_{\max}\left(\nabla^2 P_{t}\left(\btheta^{*}\right)^{-1}\right)\|\nabla \wP_{t}\left(\btheta^{*}\right)\|^2\right] \\
    &\overset{(b)}{\leq} \E\left[\lambda_{\max}\left(\nabla^2 P_{t}\left(\btheta^{*}\right)^{-1}\right) dC_3\eta\right]\\
    &= \E\left[\left(\lambda_{\min}\left(\nabla^2 P_{t}\left(\btheta^{*}\right)^{}\right)\right)^{-1} dC_3\eta\right]\\
    & = \E\left[\left(\lambda_{\min}\left(\frac{2}{t}\sum_{s=1}^t \sum_{i=1}^n p_{\wtheta_s}(i)\nabla\mu_{i}(\btheta^*)\nabla\mu_{i}(\btheta^*)^T\right)\right)^{-1} dC_3\eta\right]\\
    &\leq \frac{2 d C_3\eta}{\lambda_{\min, \cher}}
\end{align*}
where, $(a)$ follows from the min-max theorem (variational characterization of the maximum eigenvalue), in $(b)$ the quantity $\|\nabla \wP_{t}\left(\btheta^{*}\right)\|^2 \leq d^2C_3\eta$ almost surely by assumption and $C_3$ is a constant, and $(c)$ follows where $\lambda_{\min, \cher}$ is a lower bound to $\E\left[\left(\lambda_{\min}\left(\frac{1}{t}\sum_{s=1}^t \sum_{i=1}^n p_{\wtheta_s}(i)\nabla\mu_{i}(\btheta^*)\nabla\mu_{i}(\btheta^*)^T\right)\right)^{-1}\right]$. Plugging this in our result in \cref{eq:our-result} we get
\begin{align*}
    \E\left[P_t(\wtheta_t)-P_t\left(\btheta^{*}\right)\right] \leq O\left( \frac{ d 
    \sqrt{\log(dt)}}{t}+\frac{R}{t^{2}}\right),
\end{align*}
where again we are only concerned with the scaling with the dimension $d$. The convergence result is summarized below in this table:

\begin{table}[!ht]
    \centering
\begin{tabular}{|p{23.5em}|p{14em}|}
\hline
 \textbf{Sample Complexity Bound} & \textbf{Comments}\\
\hline \hline
      $\mathbb{E}\left[P_{U}\left(\wtheta(t)\right)-P_{U}\left(\theta^{*}\right)\right] 
    \leq O\left( d\frac{\sqrt{\log(dt)}}{t}+\frac{R}{t^{2}}\right)$ & Loss of \actives \citep{chaudhuri2015convergence}. The $P_U$ is loss under uniform measure over data points in pool.\\
     $\E\left[P_t(\wtheta_t)-P_t\left(\btheta^{*}\right)\right] \leq  O\left( d \frac{\sqrt{\log(dt)}}{t}+\frac{R}{t^{2}}\right)$ & Loss for \cher \textcolor{red}{(Ours)}. The $P_t$ is loss under a worst-case measure over data points.\\
    \hline
    \end{tabular}
\caption{Active Regression comparison. }
    \label{tab:Active-regression0}
\end{table}
The two upper bounds have the same scaling, even though $P_t$ is a different loss measure than $P_U$. The proof has steps similar to that of \citet{chaudhuri2015convergence}, with some additional arguments to handle the fact that our loss measure varies with time. 


\subsubsection{Discussion on Definitions and Assumptions for Continuous Hypotheses}
\label{app:cher-smooth-upper-assm}
\begin{definition}
We define the following star-norm quantity at round $t$ as
\begin{align*}
    \|A\|_{*}=\left\|\left(\nabla^{2} P_t\left(\btheta^*\right)\right)^{-1 / 2} \cdot A \cdot\left(\nabla^{2} P_t\left(\btheta^*\right)\right)^{-1 / 2}\right\|.
\end{align*}
\end{definition}
Now we state the two following assumptions required by the \Cref{thm:cher-smooth}. Also note that we define the squared loss function $\ell_{s}(\btheta) = (\mu_{I_s}(\btheta) - Y_s)^2$, the cumulative loss function $L_{s}(\btheta) = \sum_{s'=1}^s\ell_{s'}(\btheta^*) =  \sum_{s' = 1}^s(\mu_{I_{s'}}(\btheta) - Y_{s'})^2$, and $\ell_{1}(\btheta), \ell_{2}(\btheta),\ldots, \ell_{t}(\btheta)$ for any $\btheta\in\bTheta$ are not independent. In contrast \citet{chaudhuri2015convergence} assumes that the loss functions are independent for any time $s\in[t]$. Next we state the assumptions used for the proof of \Cref{thm:cher-smooth}.  

\Cref{assm:posdef2} in \Cref{sec:active-regression} is a mild assumption on the bounded nature of the eigenvalues of the Hessian matrix $\nabla_{\btheta = \btheta^{\prime}}^{2} \mu_{I_{s}}\left(\btheta^{}\right)$ evaluated at any $\btheta'\in\bTheta$. Then following assumption states a few regularity assumptions required for \Cref{thm:cher-smooth}. A similar set of assumptions has also been used by \citet{chaudhuri2015convergence, bu2019active}.


\begin{assumption}\textbf{(Assumptions for \Cref{thm:cher-smooth}):}
\label{assm:thm}
We assume the following assumptions hold with probability $1$:
\begin{enumerate}
    \item \textbf{(Convexity of $\ell_{s}$):} The loss function $\ell_{s}$ is convex for all time $s\in[t]$.
    \item \textbf{(Smoothness of $\ell_{s}$):} The $\ell_{s}$ is smooth such that the first, second, and third derivatives exist at all interior points in $\bTheta$.
    \item \textbf{(Regularity Conditions):} \begin{enumerate}
        \item $\bTheta$ is compact and $\ell_{s}(\btheta)$ is bounded for all $\btheta\in\bTheta$ and for all $s\in[t]$.
        \item $\btheta^*$ is an interior point in $\bTheta$.
        \item $\nabla^2 \ell_{s}(\btheta^*)$ is positive definite, for all $s\in[t]$ .
        \item There exists a neighborhood $\mathcal{B}$ of $\btheta^*$ and a constant $C_{1}$, such that $\nabla^{2} \ell_{s}(\btheta)$ is $C_{1}$ -Lipschitz. Hence, we have that $\left\|\nabla^{2} \ell_{s}(\btheta)-\nabla^{2} \ell_{s}\left(\btheta^{\prime}\right)\right\|_{*} \leq C_{1}\left\|\btheta-\btheta^{\prime}\right\|_{\nabla^{2} P_s\left(\btheta^{*}\right)}$, for $\btheta, \btheta^{\prime}$ in this neighborhood.
    \end{enumerate}
    \item \textbf{(Concentration at $\btheta^{*}$):} We further assume that $\left\|\nabla \ell_{s}\left(\btheta^{*}\right)\right\|_{\left(\nabla^{2} P_s\left(\btheta^{*}\right)\right)^{-1}} \leq C_{2}$ hold with probability one.
\end{enumerate}
\end{assumption}
\Cref{assm:thm} (c) is different from that of \citet{chaudhuri2015convergence}, where they assumed that $\nabla^2 \mathbb{E}[\psi_s](\btheta^*)$ is positive definite, where $\psi_s$ are i.i.d.\ loss functions from some distribution. In our case the loss functions are not i.i.d., which is why we make the assumption on the loss at every time~$s$.



\subsubsection{Concentration Lemmas for Continuous Hypotheses}
\label{app:cher-conc-tyt1}

\begin{lemma}
\label{lemma:vector-conc}
  The probability that $\|\nabla \wP_t(\btheta^*)\|_{\left(\nabla^{2} P\left(\btheta^{*}\right)\right)^{-1}} $ crosses the threshold $\sqrt{\dfrac{c\gamma\log (dt)}{t}} > 0$ is bounded by
 \begin{align*}
     \Pb\left(\|\nabla \wP_t(\btheta^*)\|_{\left(\nabla^{2} P_{t}\left(\btheta^{*}\right)\right)^{-1}}\geq C_2\sqrt{\dfrac{c\gamma\log (dt)}{t}}\right) \leq \frac{1}{t^{c\gamma}}.
 \end{align*}
\end{lemma}

\begin{proof}
Define $\mathbf{u_s} \coloneqq \nabla (Y_s -\mu_{I_s}(\btheta^*))^2$. Then we have $\mathbf{u}_{1}, \mathbf{u}_{2}, \ldots, \mathbf{u}_{t}$ as random vectors such that
\begin{align*}
    & \mathbb{E}\left[\left\|\sum_{s=1}^t \mathbf{u_s}\right\|_{\left(\nabla^{2} P_t\left(\btheta^{*}\right)\right)^{-1}}^{2} \bigg | \mathbf{u}_{1}, \ldots, \mathbf{u}_{s-1}\right] = \mathbb{E}\left[\sum_{s=1}^t \mathbf{u_s}^{\top}\left(\nabla^{2} P_t\left(\btheta^{*}\right)\right)^{-1}\mathbf{u_s} \mid \mathbf{u}_{1}, \ldots, \mathbf{u}_{s-1}\right] \leq t C^2_2
\end{align*}
Also we have that $\|\mathbf{u_s}\| \leq C_2$. Finally we have that 
\begin{align*}
    \E[\nabla_{\btheta = \btheta^*}\mathbf{u_s}] = -2\sum_{i=1}^n p_{\wtheta_{s-1}}(\mu_i(\btheta^*) - \mu_{i}(\btheta^*))\nabla_{\btheta = \btheta^*}\mu_{i}(\btheta^*) = 0.
\end{align*}
Then following \Cref{lemma:vector-martingale} and by setting $\epsilon = c\gamma\log(dt)$ we can show that
\begin{align*}
    \Pb&\left(\|\frac{1}{t}\sum_{s=1}^t \mathbf{u_s}\|^2_{_{\left(\nabla^{2} P_{t}\left(\btheta^{*}\right)\right)^{-1}}} - \E\left[\|\frac{1}{t}\sum_{s=1}^t \mathbf{u_s}\|^2_{_{\left(\nabla^{2} P_{t}\left(\btheta^{*}\right)\right)^{-1}}}\right] >  \frac{1}{t}\sqrt{8tC_2^2 \epsilon} + \dfrac{4 C_2 }{3\epsilon}   \right) \\
    &= \Pb\left(\|\frac{1}{t}\sum_{s=1}^t \mathbf{u_s}\|^2_{_{\left(\nabla^{2} P_{t}\left(\btheta^{*}\right)\right)^{-1}}} >  C_1^2 +  C_2\sqrt{\frac{8\epsilon}{t}} + \dfrac{4 C_2 }{3\epsilon}   \right)\\
    &\leq \Pb\left(\|\sum_{s=1}^t \mathbf{u_s}\|^2_{_{\left(\nabla^{2} P_{t}\left(\btheta^{*}\right)\right)^{-1}}} >  C_2\sqrt{\frac{8\epsilon}{t}}  \right) = \Pb\left(\|\sum_{s=1}^t \mathbf{u_s}\|^2_{_{\left(\nabla^{2} P_{t}\left(\btheta^{*}\right)\right)^{-1}}} >  4C_2\sqrt{\dfrac{  c\gamma \log(dt)}{t}}  \right)\\
    &\leq \exp(- c\gamma\log (dt)) = \left(\frac{1}{dt}\right)^{c\gamma} \leq \frac{1}{t^{c\gamma}}
\end{align*}
The claim of the lemma follows.
\end{proof}

\begin{lemma}
\label{lemma:matrix-conc3}
Let $\wP_t(\btheta^*) = \frac{1}{t}\sum_{s=1}^t \ell_{s}(\btheta^*)$ and $\nabla^2 P_t(\btheta^*) = \frac{1}{t}\sum_{s=1}^t \nabla^2\E[\ell_{s}(\btheta^*)|\F^{s-1}]$. Then we can bound the 
\begin{align*}
    \Pb&\left(\lambda_{\max}(\nabla^2\wP_t(\btheta^*) - \nabla^2 P_t(\theta^*)) > \sqrt{\dfrac{8\eta^2\lambda^2_1c \gamma \log(dt)}{t}}\right) \leq \dfrac{2}{(dt)^{\gamma}},
\end{align*}
where $c > 0$ is a constant.
\end{lemma}

\begin{proof}
Recall that $ \wP_t(\btheta^*) = \frac{1}{t}\sum_{s=1}^t \ell_{s}(\btheta^*)$ and $\nabla^2 P_s(\theta^*) = \nabla^2\E[\ell_{s}(\btheta^*)|\F^{s-1}]$. We define $\nabla^2 P_t(\btheta^*) = \frac{1}{t}\sum_{s=1}^t\nabla^2\E[\ell_{s}(\btheta^*)|\F^{s-1}]$. Denote,
    $\mathbf{V}_s = 2\nabla_{\btheta = \btheta^*}\mu_{I_s}(\btheta)\nabla_{\btheta = \btheta^*}\mu_{I_s}(\btheta)^\top - 2\sum_{i=1}^n p_{\wtheta_{s-1}}(i)\nabla_{\btheta = \btheta^*}\mu_{i}(\btheta)\nabla_{\btheta = \btheta^*}\mu_{i}(\btheta)^\top$.
Then we can show that,
\begin{align}
    \Pb&\left(\lambda_{\max}(\nabla^2\wP_t(\btheta^*) - \nabla^2 P_t(\theta^*)) > \sqrt{\dfrac{8\eta^2\lambda^2_1c\gamma\log(dt)}{t}}\right)\nonumber\\
    &= \Pb\left(\lambda_{\max}\left(\nabla^2_{\btheta=\btheta^*}\frac{1}{t}\sum_{s=1}^t \ell_{s}(\btheta) - \frac{1}{t}\sum_{s=1}^t\nabla^2_{\btheta=\btheta^*} \E[\ell_{s}(\btheta)|\F^{s-1}]\right) > \sqrt{\dfrac{8\eta^2\lambda^2_1c\gamma\log(dt)}{t}}\right)\nonumber\\
    &= \Pb\left(\lambda_{\max}\left(\nabla^2_{\btheta=\btheta^*}\frac{1}{t}\sum_{s=1}^t\left( \ell_{s}(\btheta) - \nabla^2_{\btheta=\btheta^*} \E[\ell_{s}(\btheta)|\F^{s-1}]\right)\right) > \sqrt{\dfrac{8\eta^2\lambda^2_1c\gamma\log(dt)}{t}}\right)\nonumber
\end{align}
\begin{align}
    &\overset{(a)}{=} \Pb\left(\lambda_{\max}\left(\frac{1}{t}\sum_{s=1}^t\left(Y_s - \mu_{I_s}(\btheta^*)\right)\nabla^2_{\btheta=\btheta^*}\mu_{I_s}(\btheta^*) + \frac{1}{t}\sum_{s=1}^t \mathbf{V}_s\right) > \sqrt{\dfrac{8\eta^2\lambda^2_1c\gamma\log(dt)}{t}}\right)\nonumber\\
    &\leq \Pb\left(\lambda_{\max}\left(\frac{1}{t}\sum_{s=1}^t-2\left(Y_s - \mu_{I_s}(\btheta^*)\right)\nabla^2_{\btheta=\btheta^*}\mu_{I_s}(\btheta^*)\right) > \frac{1}{2}\sqrt{\dfrac{8\eta^2\lambda^2_1c\gamma\log(dt)}{t}}\right)\nonumber\\
    &\qquad + \Pb\left( \lambda_{\max}\left(\frac{1}{t}\sum_{s=1}^t \mathbf{V}_s\right) > \frac{1}{2}\sqrt{\dfrac{8\eta^2\lambda^2_1c\gamma\log(dt)}{t}}\right)  \nonumber\\
    &\overset{(b)}{\leq} \Pb\left(\frac{1}{t}\sum_{s=1}^t-2\left(Y_s - \mu_{I_s}(\btheta^*)\right)\lambda_{\max}\left(\nabla^2_{\btheta=\btheta^*}\mu_{I_s}(\btheta^*)\right) > \frac{1}{2}\sqrt{\dfrac{8\eta^2\lambda^2_1c\gamma\log(dt)}{t}}\right)\nonumber\\
    &\qquad + \Pb\left( \frac{1}{t}\sum_{s=1}^t \lambda_{\max}\left(\mathbf{V}_s\right) > \frac{1}{2}\sqrt{\dfrac{8\eta^2\lambda^2_1c\gamma\log(dt)}{t}}\right)\label{eq:regression-conc}\\
    &\overset{(c)}{\leq} 2\exp\left(- \dfrac{t^2 8\eta^2\lambda_1^2c\gamma\log(dt)}{4t}\cdot\dfrac{1}{2tc\eta^2\lambda_1^2}\right) \overset{(d)}{\leq} 2\left(\dfrac{1}{dt}\right)^{\gamma}.\nonumber 
\end{align}
where, $(a)$ follows from substituting the value of $\nabla^2_{\btheta=\btheta^*} \ell_{s}(\btheta) - \nabla^2_{\btheta=\btheta^*} \E[\ell_{s}(\btheta)|\F^{s-1}]$ from \Cref{lemma:support-lemma3}, and $(b)$ follows by triangle inequality, $(c)$ follows by using two concentration inequalities stated below, and $(d)$ follows by simplifying the equations. 


Denote $Q_s = -2\left(Y_s - \mu_{I_s}(\btheta^*)\right)\lambda_{\max}\left(\nabla^2_{\btheta=\btheta^*}\mu_{I_s}(\btheta^*)\right)$. Also note that $\lambda_{\max}\left(\nabla^2_{\btheta=\btheta^*}\mu_{I_s}(\btheta^*)\right) \leq \lambda_1$.
\begin{align*} 
\Pb(\sum_{s=1}^t-2&\left(Y_s - \mu_{I_s}(\btheta^*)\right)\lambda_{\max}\left(\nabla^2_{\btheta=\btheta^*}\mu_{I_s}(\btheta^*)\right) \geq  \epsilon) =\Pb\left(-\sum_{s=1}^{t} Q_{s} \geq \epsilon\right) \\
&=\Pb\left(e^{-\lambda \sum_{s=1}^{t} Q_{s}} \geq e^{\lambda \epsilon}\right) \overset{(a)}{\leq} e^{-\lambda \epsilon} \E\left[e^{-\lambda \sum_{s=1}^{t} Q_{s}}\right] 
=  e^{-\lambda \epsilon} \E\left[\E\left[e^{-\lambda \sum_{s=1}^{t} Q_{s}}\big|\wtheta(t-1)\right] \right]\\
&\overset{(b)}{=} e^{-\lambda \epsilon} \E\left[\E\left[e^{-\lambda  Q_{t}}|\wtheta(t-1)\right]\E\left[e^{-\lambda \sum_{s=1}^{t-1} Q_{s}} \big|\wtheta(t-1)\right]  \right]\\
&\leq e^{-\lambda \epsilon} \E\left[\exp\left(2\lambda^2\lambda_1^2\eta^2\right)\E\left[e^{-\lambda \sum_{s=1}^{t-1} Q_{s}}\big |\wtheta(t-1)\right]  \right]\\
& \overset{}{=} e^{-\lambda \epsilon} e^{2\lambda^{2} \eta^{2}\lambda_1^2} \mathbb{E}\left[e^{-\lambda \sum_{s=1}^{t-1} Q_{s}}\right] \\ 
& \vdots \\ 
& \overset{(c)}{\leq} e^{-\lambda \epsilon} e^{2\lambda^{2} t \eta^{2}\lambda^2_1} 
\overset{(d)}{\leq} \exp\left(-\dfrac{2\epsilon^2}{t\eta^2\lambda_1^2}\right).
\end{align*}
where $(a)$ follows by Markov's inequality, $(b)$ follows as $Q_s$ is conditionally independent given $\wtheta(s-1)$, $(c)$ follows by unpacking the term for $t$ times and $(d)$  follows by taking $\lambda= \epsilon / 4t\lambda_1^2\eta^2$. Next we bound the second term of \eqref{eq:regression-conc} below.

\begin{align*} 
\Pb(\sum_{s=1}^t \lambda_{\max}\left(\mathbf{V}_s\right) \geq  \epsilon) &=\Pb\left(\lambda\sum_{s=1}^{t} \lambda_{\max}\left(\mathbf{V}_s\right) \geq \lambda\epsilon\right) 
=\Pb\left(e^{\lambda \sum_{s=1}^{t} \lambda_{\max}\left(\mathbf{V}_s\right)} \geq e^{\lambda \epsilon}\right) \overset{(a)}{\leq} e^{-\lambda \epsilon} \E\left[e^{\lambda \sum_{s=1}^{t} \lambda_{\max}\left(\mathbf{V}_s\right)}\right] \\
&=  e^{-\lambda \epsilon} \E\left[\E\left[e^{\lambda \sum_{s=1}^{t} \lambda_{\max}\left(\mathbf{V}_s\right)}\big|\wtheta(t-1)\right] \right]\\
&\overset{(b)}{=} e^{-\lambda \epsilon} \E\left[\E\left[e^{\lambda  \lambda_{\max}(\mathbf{V}_{t})}|\wtheta(t-1)\right]\E\left[e^{\lambda \sum_{s=1}^{t-1} \lambda_{\max}\left(\mathbf{V}_s\right)} \big|\wtheta(t-1)\right]  \right]\\
&\overset{(c)}{\leq} e^{-\lambda \epsilon} \E\left[\exp\left(2c\lambda^2\lambda^2_1\eta^2\right)\E\left[e^{\lambda \sum_{s=1}^{t-1} \lambda_{\max}\left(\mathbf{V}_s\right)}\big |\wtheta(t-1)\right]  \right]\\
& \overset{}{=} e^{-\lambda \epsilon} e^{2c\lambda^{2} \eta^{2}\lambda_1^2} \mathbb{E}\left[e^{\lambda \sum_{s=1}^{t-1} \lambda_{\max}\left(\mathbf{V}_s\right)}\right] \\ 
& \vdots \\ 
& \overset{(d)}{\leq} e^{-\lambda \epsilon} e^{2c\lambda^{2} t \eta^{2}\lambda^2_1} \overset{(e)}{\leq} \exp\left(-\dfrac{2\epsilon^2}{tc\eta^2\lambda_1^2}\right)
\end{align*}

where $(a)$ follows by Markov's inequality, $(b)$ follows as $\lambda_{\max}(\mathbf{V}_s)$ is conditionally independent given $\wtheta(s-1)$. In the inequality $(c)$ using the always valid upper bound of $2\lambda_1$, we have that $\E[\lambda_{\max}(\mathbf{V}_t)] \leq 2\lambda_1$. So the term in inequality $(c)$ will become 
$e^{-\lambda \epsilon} e^{2\lambda^2 t\eta^2 \lambda_1^t + 4t\lambda \lambda_1}$. Hence, we can upper bound the inequality $(c)$ by a constant $c > 0$ such that we have $\E[e^{\lambda \lambda_{\max}(V_t)} \mid \wtheta(t-1)] \leq e^{2\lambda^2\lambda_1^2\eta^2}e^{2\lambda \times 2\lambda_1} = \exp(2\lambda^2\lambda_1^2\eta^2 + 4 \lambda \lambda_1) \leq  \exp(2c\lambda^2\lambda_1^2\eta^2)$. The inequality $(d)$ follows by unpacking the term for $t$ times and $(e)$  follows by taking $\lambda= \epsilon / 4tc\lambda_1^2\eta^2$.

\end{proof}

\subsubsection{Support Lemma for Continuous Hypotheses}
\label{app:smooth-support-lemma}

\begin{lemma}
\label{lemma:support-lemma1}
Let the $j$-th row and $k$-th column entry in the Hessian matrix $\nabla^2_{\btheta = \btheta^{\prime}}(\ell_{s}(\btheta))$ be denoted as $[\nabla^2_{\btheta = \btheta^{\prime}}(\ell_{s}(\btheta))]_{jk}$. Then we have that
\begin{align*}
    [\nabla^2_{\btheta = \btheta^{\prime}}(\ell_{s}(\btheta))]_{jk} = 2 \dfrac{\partial \mu_{I_s}(\btheta)}{\partial \btheta_j}\dfrac{\partial \mu_{I_s}(\btheta)}{\partial \btheta_k} + 2\left(\mu_{I_s}(\btheta) - Y_s\right)\dfrac{\partial^2 \mu_{I_s}(\btheta)}{\partial \btheta_j\partial\btheta_k}.
\end{align*}
\end{lemma}

\begin{proof}
We want to evaluate the Hessian $\nabla^2_{\btheta = \btheta^{\prime}}(\ell_{s}(\btheta))$ at any $\btheta^{\prime}\in\bTheta$. We denote the $j$-th row and $k$-th column entry in the Hessian matrix as $[\nabla^2_{\btheta = \btheta^{\prime}}(\ell_{s}(\btheta))]_{jk}$. Then we can show that
\begin{align*}
    [\nabla^2_{\btheta = \btheta^{\prime}}(\ell_{s}(\btheta))]_{jk}&\coloneqq \frac{\partial }{\partial \btheta_j} \left[\frac{\partial (\mu_{I_s}(\btheta) - Y_s)^2}{\partial \btheta_k}\right] = \frac{\partial }{\partial \btheta_j}\left[2(\mu_{I_s}(\btheta) - Y_s) \frac{\partial \mu_{I_s}(\btheta)}{\partial \btheta_k}\right]\\
    &= \frac{\partial }{\partial \btheta_j}\left[ 2\mu_{I_s}(\btheta)\dfrac{\partial \mu_{I_s}(\btheta)}{\partial \btheta_k} - 2Y_s\dfrac{\partial \mu_{I_s}(\btheta)}{\partial \btheta_k}\right]\\
    &= 2 \dfrac{\partial \mu_{I_s}(\btheta)}{\partial \btheta_j}\dfrac{\partial \mu_{I_s}(\btheta)}{\partial \btheta_k} + 2 \mu_{I_s}(\btheta)\dfrac{\partial^2 \mu_{I_s}(\btheta)}{\partial \btheta_j\partial\btheta_k} - 2 Y_s\dfrac{\partial^2 \mu_{I_s}(\btheta)}{\partial \btheta_j\partial\btheta_k} - 2 \dfrac{\partial \mu_{I_s}(\btheta)}{\partial \btheta_j}\dfrac{\partial Y_s}{\partial \btheta_k}\\
    &= 2 \dfrac{\partial \mu_{I_s}(\btheta)}{\partial \btheta_j}\dfrac{\partial \mu_{I_s}(\btheta)}{\partial \btheta_k} + 2\left(\mu_{I_s}(\btheta) - Y_s\right)\dfrac{\partial^2 \mu_{I_s}(\btheta)}{\partial \btheta_j\partial\btheta_k}
\end{align*}
The claim of the lemma follows.
\end{proof}

\begin{lemma}
\label{lemma:support-lemma2}
Let the $j$-th row and $k$-th column entry in the Hessian matrix $\nabla^2_{\btheta = \btheta^{\prime}}(\E[\ell_{s}(\btheta)|\F^{s-1}])$ be denoted as $[\nabla^2_{\btheta = \btheta^{\prime}}(\E[\ell_{s}(\btheta)|\F^{s-1}])]_{jk}$. Then we have that
\begin{align*}
     \left[\nabla^2_{\btheta = \btheta^{\prime}}  \E[\ell_{s}(\btheta)|\F^{s-1}]\right]_{jk} =  2\sum_{i=1}^n p_{\wtheta_{s-1}}(i)\left( \dfrac{\partial \mu_{i}(\btheta)}{\partial \btheta_j}\dfrac{\partial \mu_{i}(\btheta)}{\partial \btheta_k} + 2\left(\mu_{i}(\btheta) - \mu_{i}(\btheta^*)\right)\dfrac{\partial^2 \mu_{i}(\btheta)}{\partial \btheta_j\partial\btheta_k}\right).
\end{align*}
\end{lemma}

\begin{proof}
Now we want to evaluate the Hessian $\nabla^2_{\btheta = \btheta^{\prime}}(\E[\ell_{s}(\btheta)|\F^{s-1}])$ at any $\btheta^{\prime}\in\bTheta$. We denote the $j$-th row and $k$-th column entry in the Hessian matrix as $[\nabla^2_{\btheta = \btheta^{\prime}}(\E[\ell_{s}(\btheta)|\F^{s-1}])]_{jk}$. Then we can show that
\begin{align}
     \nabla^2_{\btheta = \btheta^{\prime}}  \E[\ell_{s}(\btheta)|\F^{s-1}]
    &= \nabla^2_{\btheta = \btheta^{\prime}} \left(\mu^2_{I_s}(\btheta) + \E[Y^2_s|\F^{s-1}] - 2\E[Y_s|\F^{s-1}]\mu_{I_s}(\btheta)\right)\nonumber\\
    &= \nabla^2_{\btheta = \btheta^{\prime}} \sum_{i=1}^n p_{\wtheta_{s-1}}(i)\left(\mu^2_{i}(\btheta) + \mu^2_{i}(\btheta^{\prime}) + \frac{1}{2} - 2\mu_{i}(\btheta^*)\mu_{i}(\btheta)\right)\nonumber\\
    &=  \nabla^2_{\btheta = \btheta^{\prime}} \sum_{i=1}^n p_{\wtheta_{s-1}}(i)\left(\left(\mu_{i}(\btheta^*) - \mu_{i}(\btheta)\right)^2 + \frac{1}{2} \right)\nonumber\\
    &=  \nabla^2_{\btheta = \btheta^{\prime}}\sum_{i=1}^n p_{\wtheta_{s-1}}(i)\left(\left(\mu_{i}(\btheta^*) - \mu_{i}(\btheta)\right)^2  \right)\label{eq:Hessian-expectation}
\end{align}
We now denote the $j$-th row and $k$-th column entry of the Hessian Matrix $\nabla^2_{\btheta = \btheta^{\prime}}((\mu_{i}(\btheta) - \mu_i(\btheta^*))^2)$ as $\big[\nabla^2_{\btheta = \btheta^{\prime}}((\mu_{i}(\btheta) - \mu_{i}(\btheta^*))^2)\big]_{jk}$. Then we can show that
\begin{align*}
    \big[\nabla^2_{\btheta = \btheta^*}((\mu_{i}(\btheta) - \mu_{i}(\btheta^*))^2)\big]_{jk} &\coloneqq \frac{\partial }{\partial \btheta_j} \left[\frac{\partial (\mu_{i}(\btheta) - \mu_{i}(\btheta^*))^2}{\partial \btheta_k}\right] = \frac{\partial }{\partial \btheta_j}\left[2(\mu_{i}(\btheta) - \mu_{i}(\btheta^*)) \frac{\partial \mu_{i}(\btheta)}{\partial \btheta_k}\right]\\
    &= \frac{\partial }{\partial \btheta_j}\left[ 2\mu_{i}(\btheta)\dfrac{\partial \mu_{i}(\btheta)}{\partial \btheta_k} - 2\mu_i(\btheta^*)\dfrac{\partial \mu_{i}(\btheta)}{\partial \btheta_k}\right]\\
    &= 2 \dfrac{\partial \mu_{i}(\btheta)}{\partial \btheta_j}\dfrac{\partial \mu_{i}(\btheta)}{\partial \btheta_k} + 2 \mu_{i}(\btheta)\dfrac{\partial^2 \mu_{i}(\btheta)}{\partial \btheta_j\btheta_k} \\
    &- 2 \mu_{i}(\btheta^*)\dfrac{\partial^2 \mu_{i}(\btheta)}{\partial \btheta_j\btheta_k} - 2 \dfrac{\partial \mu_{i}(\btheta)}{\partial \btheta_j}\dfrac{\partial \mu_{i}(\btheta^*)}{\partial \btheta_k}\\
    &= 2 \dfrac{\partial \mu_{i}(\btheta)}{\partial \btheta_j}\dfrac{\partial \mu_{i}(\btheta)}{\partial \btheta_k} + 2\left(\mu_{i}(\btheta) - \mu_{i}(\btheta^*)\right)\dfrac{\partial^2 \mu_{i}(\btheta)}{\partial \btheta_j\partial\btheta_k}
\end{align*}
Plugging this back in \cref{eq:Hessian-expectation} we get that
\begin{align*}
    \left[\nabla^2_{\btheta = \btheta^{\prime}}  \E[\ell_{s}(\btheta)|\F^{s-1}]\right]_{jk} = 2\sum_{i=1}^n p_{\wtheta_{s-1}}(i)\left( \dfrac{\partial \mu_{i}(\btheta)}{\partial \btheta_j}\dfrac{\partial \mu_{i}(\btheta)}{\partial \btheta_k} + 2\left(\mu_{i}(\btheta) - \mu_{i}(\btheta^*)\right)\dfrac{\partial^2 \mu_{i}(\btheta)}{\partial \btheta_j\partial\btheta_k}\right).
\end{align*}
\end{proof}

\begin{lemma}
\label{lemma:support-lemma3}
The sum of the difference of the Hessians $\sum_{s=1}^{t} \nabla_{\btheta = \btheta'}^{2} \ell_{s}\left(\btheta^{}\right)-\E\left[\nabla_{\btheta=\btheta'}^{2} \ell_{s}\left(\btheta^{}\right) \mid \F^{s-1}\right]$ is given by 
\begin{align*}
    \sum_{s=1}^{t}\nabla_{\btheta = \btheta'}^{2} \ell_{s}\left(\btheta^{}\right)-\E\left[\nabla_{\btheta=\btheta'}^{2} \ell_{s}\left(\btheta^{}\right) \mid \F^{s-1}\right] \!\! =\!\! \sum_{s=1}^t\bigg( -2(Y_s - \mu_{I_s}(\btheta))\dfrac{\partial^2 \mu_{I_s}(\btheta)}{\partial \btheta_j\partial\btheta_k} &+ 2\dfrac{\partial \mu_{I_s}(\btheta)}{\partial \btheta_j}\dfrac{\partial \mu_{I_s}(\btheta)}{\partial \btheta_k} \\
    &- 2 \sum_{i=1}^n p_{\wtheta_{s-1}}(i)\dfrac{\partial \mu_{i}(\btheta)}{\partial \btheta_j}\dfrac{\partial \mu_{i}(\btheta)}{\partial \btheta_k}\bigg).
\end{align*}
\end{lemma}

\begin{proof}
First note that the difference $\nabla_{\btheta = \btheta'}^{2} \ell_{s}\left(\btheta^{}\right)-\E\left[\nabla_{\btheta=\btheta'}^{2} \ell_{s}\left(\btheta^{}\right) \mid \F^{s-1}\right]_{jk}$ is given by
\begin{align}
   \nabla_{\btheta = \btheta'}^{2} \ell_{s}\left(\btheta^{}\right)-\E\left[\nabla_{\btheta=\btheta'}^{2} \ell_{s}\left(\btheta^{}\right) \mid \F^{s-1}\right] \overset{(a)}{=} & 2\dfrac{\partial \mu_{I_s}(\btheta)}{\partial \btheta_j}\dfrac{\partial \mu_{I_s}(\btheta)}{\partial \btheta_k} + 2\left(\mu_{I_s}(\btheta) - Y_s\right)\dfrac{\partial^2 \mu_{I_s}(\btheta)}{\partial \btheta_j\partial\btheta_k}\nonumber\\
    &- 2 \sum_{i=1}^n p_{\wtheta_{s-1}}(i)\bigg( \dfrac{\partial \mu_{i}(\btheta)}{\partial \btheta_j}\dfrac{\partial \mu_{i}(\btheta)}{\partial \btheta_k} - \left(\mu_{i}(\btheta) - \mu_{i}(\btheta^*)\right)\dfrac{\partial^2 \mu_{i}(\btheta)}{\partial \btheta_j\partial\btheta_k}\bigg)\nonumber\\
    = & -2(Y_s - \mu_{I_s}(\btheta))\dfrac{\partial^2 \mu_{I_s}(\btheta)}{\partial \btheta_j\partial\btheta_k} + 2\dfrac{\partial \mu_{I_s}(\btheta)}{\partial \btheta_j}\dfrac{\partial \mu_{I_s}(\btheta)}{\partial \btheta_k}\nonumber\\
    &\qquad - 2 \sum_{i=1}^n p_{\wtheta_{s-1}}(i)\dfrac{\partial \mu_{i}(\btheta)}{\partial \btheta_j}\dfrac{\partial \mu_{i}(\btheta)}{\partial \btheta_k} \label{eq:support-equality}
\end{align}
where, $(a)$ follows from \Cref{lemma:support-lemma1} and \Cref{lemma:support-lemma2}. Plugging this equality in \Cref{eq:support-equality} below we get
\begin{align*}
    \sum_{s=1}^{t} \nabla_{\btheta = \btheta'}^{2} \ell_{s}\left(\btheta^{}\right)-\E\left[\nabla_{\btheta=\btheta'}^{2} \ell_{s}\left(\btheta^{}\right) \mid \F^{s-1}\right]  &= \sum_{s=1}^t \bigg(-2(Y_s - \mu_{I_s}(\btheta))\dfrac{\partial^2 \mu_{I_s}(\btheta)}{\partial \btheta_j\partial\btheta_k} + 2\dfrac{\partial \mu_{I_s}(\btheta)}{\partial \btheta_j}\dfrac{\partial \mu_{I_s}(\btheta)}{\partial \btheta_k}\\
    &\qquad  - 2 \sum_{i=1}^n p_{\wtheta_{s-1}}(i)\bigg(\dfrac{\partial \mu_{i}(\btheta)}{\partial \btheta_j}\dfrac{\partial \mu_{i}(\btheta)}{\partial \btheta_k} - 2\left(\mu_{i}(\btheta) - \mu_{i}(\btheta^*)\right)\dfrac{\partial^2 \mu_{i}(\btheta)}{\partial \btheta_j\partial\btheta_k}\bigg)\bigg).
\end{align*}
The claim of the lemma follows.
\end{proof}

\begin{lemma}
\label{lemma:inequality}
Let $\wtheta_t - \btheta^* = \left(\nabla^2 \wP_t(\ttheta_t)\right)^{-1}\nabla \wP_t(\btheta^*)$ where $\ttheta_t$ is between $\wtheta_{t}$ and $\btheta^{*}$. Then we can show that
\begin{align*}
    \left\|\wtheta_{t}-\btheta^{*}\right\|_{\nabla^{2} P_{t}\left(\btheta^{*}\right)} \leq \left\|\left(\nabla^{2} P_{t}\left(\btheta^{*}\right)\right)^{1 / 2}\left(\nabla^{2} \wP_{t}(\ttheta_{t})\right)^{-1}\left(\nabla^{2} P_{t}\left(\btheta^{*}\right)\right)^{1 / 2}\right\|\left\|\nabla \wP_{t}\left(\btheta^{*}\right)\right\|_{\left(\nabla^{2} P_{t}\left(\btheta^{*}\right)\right)^{-1}} .
\end{align*}
\end{lemma}

\begin{proof}
We begin with the definition of $\left\|\wtheta_{t}-\btheta^{*}\right\|_{\nabla^{2} P_{t}\left(\btheta^{*}\right)}$ as follows:
\begin{align*}
\left\|\wtheta_{t}-\btheta^{*}\right\|_{\nabla^{2} P_{t}\left(\btheta^{*}\right)} &\overset{(a)}{=} \sqrt{(\wtheta_{t}-\btheta^{*})^T\nabla^{2} P_{t}\left(\btheta^{*}\right)(\wtheta_{t}-\btheta^{*})}\\
&\overset{(b)}{=} \sqrt{\left(\left(\nabla^{2} \wP_{t}(\ttheta_{t})\right)^{-1} \nabla \wP_{t}\left(\btheta^{*}\right)\right)^T\nabla^{2} P_{t}\left(\btheta^{*}\right)\left(\left(\nabla^{2} \wP_{t}(\ttheta_{t})\right)^{-1} \nabla \wP_{t}\left(\btheta^{*}\right)\right)}\\
&\overset{(c)}{\leq}  \left\|\nabla^2 P_{t}\left(\btheta^{*}\right)^{1/2} \left(\nabla^{2} \wP_{t}(\ttheta_{t})\right)^{-1}\nabla^2 P_{t}\left(\btheta^{*}\right)^{1/2}\right\| \sqrt{\left(\nabla^{} \wP_{t}\left(\btheta^{*}\right)^T\left(\nabla^{2} P_{t}(\btheta^*)\right)^{-1} \nabla \wP_{t}\left(\btheta^{*}\right)\right)}\\
& = \left\|\left(\nabla^{2} P_{t}\left(\btheta^{*}\right)\right)^{1 / 2}\left(\nabla^{2} \wP_{t}(\ttheta_{t})\right)^{-1}\left(\nabla^{2} P_{t}\left(\btheta^{*}\right)\right)^{1 / 2}\right\|\left\|\nabla \wP_{t}\left(\btheta^{*}\right)\right\|_{\left(\nabla^{2} P_{t}\left(\btheta^{*}\right)\right)^{-1}} .
\end{align*}
where, $(a)$ follows as $\|x\|_{M} = \sqrt{x^T M x}$, $(b)$ follows as $\|\wtheta_t - \btheta^*\|_{\nabla^2 P_t(\btheta^*)} = \left(\nabla^2 \wP_t(\ttheta)\right)^{-1}\nabla \wP_t(\btheta^*)$, and $(c)$ follows from Cauchy Schwarz inequality.

The claim of the lemma follows.
\end{proof}

\subsubsection{Proof of Main \Cref{thm:cher-smooth}}
\label{app:thm:cher-smooth}

\begin{customtheorem}{4}\textbf{(Restatement)}
Suppose $\ell_{1}(\btheta), \ell_{2}(\btheta), \cdots, \ell_{t}(\btheta): \mathbb{R}^{d} \rightarrow \mathbb{R}$ are squared loss functions from a distribution that satisfies \Cref{assm:posdef2} and \Cref{assm:thm} in  \Cref{app:cher-smooth-upper-assm}. Further define 
    $P_t(\btheta) = \frac{1}{t}\sum_{s=1}^t\E_{I_s\sim \mathbf{p}_{\wtheta_{s-1}}}[\ell_s(\btheta)|\F^{s-1}]$
where, $\wtheta_t =\argmin_{\btheta \in \bTheta} \sum_{s = 1}^t \ell_{s}(\btheta)$. If $t$ is large enough such that $ \frac{\gamma\log(dt)}{t}\leq c^{\prime} \min \left\{\frac{1}{C_{1}C_{2} }, \frac{\operatorname{diameter}(\mathcal{B})}{C_{2}}\right\}$
then for a constant $\gamma \geq 2$ and universal constants $C_1,C_2,c'$,  we can show that 
\begin{align*}
\left(1-\rho_{t}\right) \frac{\sigma_t^2}{t}- \frac{C_1^2}{t^{\gamma / 2}} 
&\leq \E\left[P_t(\wtheta_t)-P_t\left(\btheta^{*}\right)\right] \leq \left(1+\rho_{t}\right) \frac{\sigma_t^2}{t}\!+\!\frac{\max\limits_{\btheta \in \bTheta}\left(\!P_{t}(\btheta)\!-\!P_{t}\left(\btheta^{*}\!\right)\right)}{t^{\gamma}},
\end{align*}
where 
$\sigma^{2}_t \coloneqq \E_{}\left[\frac{1}{2}\left\|\nabla \wP_{t}\left(\btheta^{*}\right)\right\|_{\left(\nabla^{2} P_t\left(\btheta^{*}\right)\right)^{-1}}^{2}\right]$, 
and $\rho_t \coloneqq \left(C_1C_2 + 2\eta^2\lambda_1^2\right)\sqrt{\frac{\gamma\log(dt)}{t}}$.
\end{customtheorem}

\begin{proof}
\textbf{Step 1:} We first bound the $\left\|\nabla^{2} \wP_{t}(\btheta)-\nabla^{2} P_t\left(\btheta^{*}\right)\right\|_{*}$ as follows
\begin{align}
\left\|\nabla^{2} \wP_{t}(\btheta)-\nabla^{2} P_t\left(\btheta^{*}\right)\right\|_{*} & \overset{(a)}{\leq}\left\|\nabla^{2} \wP_{t}(\btheta)-\nabla^{2} \wP_{t}\left(\btheta^{*}\right)\right\|_{*}+\left\|\nabla^{2} \wP_{t}\left(\btheta^{*}\right)-\nabla^{2} P_t\left(\btheta^{*}\right)\right\|_{*} \nonumber\\
& \overset{(b)}{\leq} C_{1}\left\|\btheta-\btheta^{*}\right\|_{\nabla^{2} P_t\left(\btheta^{*}\right)}+ \sqrt{\dfrac{8\eta^2\lambda_1^2c \gamma\log(dt)}{t}}\label{eq:1}
\end{align}
where, $(a)$ follows from triangle inequality, and $(b)$ is due to \Cref{assm:thm}.3.d and \Cref{lemma:matrix-conc3}.

\textbf{Step 2 (Approximation of $\nabla^{2} P_t\left(\btheta^{*}\right)$):} By choosing a sufficiently smaller ball $\mathcal{B}_{1}$ of radius of $\min \left\{1 /\left(10 C_{1}\right), \right.$ diameter $\left.(\mathcal{B})\right\}$ ), the first term in \eqref{eq:1} can be made small for $\btheta \in \mathcal{B}_{1}$. Also, for sufficiently large $t$, the second term in \eqref{eq:1} can be made arbitrarily small (smaller than $1 / 10$ ), which occurs if $\sqrt{\frac{\gamma \log (dt)}{t}} \leq \frac{c^{\prime}}{\sqrt{2\eta^2\lambda_1^2}}$. Hence for large $t$ and $\btheta\in \mathcal{B}_1$ we have 
\begin{align}
    \frac{1}{2} \nabla^{2} \wP_{t}(\btheta) \preceq \nabla^{2} P_t\left(\btheta^{*}\right) \preceq 2 \nabla^{2} \wP_{t}(\btheta) \label{eq:relation}
\end{align}

\textbf{Step 3 (Show $\wtheta_t$ in $\mathcal{B}_1$):} Fix a $\ttheta$ between $\btheta$ and $\btheta^*$ in $\mathcal{B}_1$. Apply Taylor's series approximation
\begin{align*}
    \wP_{t}(\btheta)=\wP_{t}\left(\btheta^{*}\right)+\nabla \wP_{t}\left(\btheta^{*}\right)^{\top}\left(\btheta-\btheta^{*}\right)+\frac{1}{2}\left(\btheta-\btheta^{*}\right)^{\top} \nabla^{2} \wP_{t}(\ttheta)\left(\btheta-\btheta^{*}\right)
\end{align*}
We can further reduce this as follows:
\begin{align}
\wP_{t}(\btheta)-\wP_{t}\left(\btheta^{*}\right) &\overset{(a)}{=}\nabla \wP_{t}\left(\btheta^{*}\right)^{\top}\left(\btheta-\btheta^{*}\right)+\frac{1}{2}\left\|\btheta-\btheta^{*}\right\|_{\nabla^{2} \wP_t(\ttheta)}^{2} \nonumber\\
& \overset{(b)}{\geq} \nabla \wP_{t}\left(\btheta^{*}\right)^{\top}\left(\btheta-\btheta^{*}\right)+\frac{1}{4}\left\|\btheta-\btheta^{*}\right\|_{\nabla^{2} P_{t}\left(\btheta^{*}\right)}^{2} \nonumber\\
&\geq -\left\|\btheta-\btheta^{*}\right\|_{\nabla^{2} P_{t}\left(\btheta^{*}\right)}\left\|\nabla \wP_{t}\left(\btheta^{*}\right)\right\|_{\left(\nabla^{2} P_{t}\left(\btheta^{*}\right)\right)^{-1}} + \frac{1}{4}\left(\left\|\btheta-\btheta^{*}\right\|_{\nabla^{2} P_{t}\left(\btheta^{*}\right)}\right)^{\top}\left(\left\|\btheta-\btheta^{*}\right\|_{\nabla^{2} P_{t}\left(\btheta^{*}\right)}\right)\nonumber\\
& =\left\|\btheta-\btheta^{*}\right\|_{\nabla^{2} P_{t}\left(\btheta^{*}\right)}\left(-\left\|\nabla \wP_{t}\left(\btheta^{*}\right)\right\|_{\left(\nabla^{2} P_{t}\left(\btheta^{*}\right)\right)^{-1}}+\frac{1}{4}\left\|\btheta-\btheta^{*}\right\|_{\nabla^{2} P_{t}\left(\btheta^{*}\right)}\right) \label{eq:2}
\end{align}
where, $(a)$ follows as $\left\|\btheta-\btheta^{*}\right\|_{\nabla^{2} \wP_t(\ttheta)}^{2}\coloneqq \left(\btheta-\btheta^{*}\right)^{\top} \nabla^{2} \wP_{t}(\ttheta)\left(\btheta-\btheta^{*}\right)$, and $(b)$ follows as $\ttheta$ is in between $\btheta$ and $\btheta^*$ and then using \eqref{eq:relation}. 
Note that in \eqref{eq:2} if the right hand side is positive for some $\btheta \in \mathcal{B}_{1}$, then $\btheta$ is not a local minimum. Also, since $\left\|\nabla \wP_{t}\left(\btheta^{*}\right)\right\| \rightarrow 0,$ for a sufficiently small value of $\left\|\nabla \wP_{t}\left(\btheta^{*}\right)\right\|,$ all points on the boundary of $\mathcal{B}_{1}$ will have values greater than that of $\btheta^{*} .$ Hence, we must have a local minimum of $\wP_{t}(\btheta)$ that is strictly inside $\mathcal{B}_{1}$ (for $t$ large enough). We can ensure this local minimum condition is achieved by choosing an $t$ large enough so that $\sqrt{\frac{\gamma \log (dt)}{t}} \leq c^{\prime} \min \left\{\frac{1}{C_{1}C_{2} }, \frac{\operatorname{diameter}(\mathcal{B})}{C_{2}}\right\},$ using \Cref{lemma:vector-conc} (and
our bound on the diameter of $\mathcal{B}_{1}$ ). By convexity, we have that this is the global minimum, $\wtheta_{t},$ and so $\wtheta_{t} \in \mathcal{B}_{1}$ for $t$ large enough. We will assume now that $t$ is this large from here on.

\textbf{Step 4 (Bound $\left\|\wtheta_{t}-\btheta^{*}\right\|_{\nabla^{2} P_t\left(\btheta^{*}\right)}$):} For the $\wtheta(t)$ that minimizes the sum of squared errors, $0=\nabla \wP_{t}(\wtheta_{t})$. Again, using Taylor's theorem if $\wtheta_{t}$ is an interior point, we have:
\begin{align}
0=\nabla \wP_{t}(\wtheta_{t})=\nabla \wP_{t}\left(\btheta^{*}\right)+\nabla^{2} \wP_{t}(\ttheta_{t})\left(\wtheta_{t}-\btheta^{*}\right)\label{eq:taylor}
\end{align}
for some $\ttheta_{t}$ between $\btheta^{*}$ and $\wtheta_{t}$. Now observe that $\ttheta_{t}$ is in $B_{1}$ (since, for $t$ large enough, $\wtheta_{t} \in \mathcal{B}_{1}$ ). Thus it follows from \eqref{eq:taylor} that,
\begin{align}
\wtheta_{t} - \btheta^{*}=\left(\nabla^{2} \wP_{t}(\ttheta_{t})\right)^{-1} \nabla \wP_{t}\left(\btheta^{*}\right)    \label{eq:erm}
\end{align}
where the invertibility is guaranteed by \eqref{eq:relation} and the positive definiteness of $\nabla^2 P_{t}\left(\btheta^{*}\right)$ (by \Cref{assm:thm} (3c)). We finally derive the upper bound to $\left\|\wtheta_{t}-\btheta^{*}\right\|_{\nabla^{2} P_{t}\left(\btheta^{*}\right)}$ as follows
\begin{align}
\left\|\wtheta_{t}-\btheta^{*}\right\|_{\nabla^{2} P_{t}\left(\btheta^{*}\right)} 
&\overset{(a)}{\leq} \left\|\left(\nabla^{2} P_{t}\left(\btheta^{*}\right)\right)^{1 / 2}\left(\nabla^{2} \wP_{t}(\ttheta_{t})\right)^{-1}\left(\nabla^{2} P_{t}\left(\btheta^{*}\right)\right)^{1 / 2}\right\|\left\|\nabla \wP_{t}\left(\btheta^{*}\right)\right\|_{\left(\nabla^{2} P_{t}\left(\btheta^{*}\right)\right)^{-1}} \nonumber\\
&\overset{(b)}{\leq} c C_{2} \sqrt{\frac{\gamma \log (dt)}{t}} \label{eq:3}
\end{align}
where $(a)$ follows from \Cref{lemma:inequality}, and $(b)$ from \Cref{lemma:vector-conc}, \eqref{eq:2}, and $c$ is some universal constant.

\textbf{Step 5 (Introducing $\tz$):} Fix a $\tz_t$ between $\btheta^*$ and $\wtheta_t$. Apply Taylor's series 
\begin{align}
    P_{t}(\wtheta_{t})-P_{t}\left(\btheta^{*}\right)=\frac{1}{2}\left(\wtheta_{t}-\btheta^{*}\right)^{\top} \nabla^{2} P_{t}\left(\tz_{t}\right)\left(\wtheta_{t}-\btheta^{*}\right) \label{eq:Pt-z}
\end{align}
Now note that both $\ttheta_{t}$ and $\tz_{t}$ are between $\wtheta_{t}$ and $\btheta^{*},$ which implies $\ttheta_{t} \rightarrow \btheta^{*}$ and $\tz_{t} \rightarrow \btheta^{*}$ since $\wtheta_{t} \rightarrow \btheta^{*}$. By \eqref{eq:1} and \eqref{eq:3} and applying the concentration inequalities give us
\begin{align}
&\left\|\nabla^{2} \wP_{t}(\ttheta_{t})-\nabla^{2} P_{t}\left(\btheta^{*}\right)\right\|_{*} \leq \rho_t \label{eq:theta-ttheta}\\
&\left\|\nabla^{2} P_{t}\left(\tz_{t}\right)-\nabla^{2} P_{t}\left(\btheta^{*}\right)\right\|_{*} \leq C_{1}\left\|\tz_{t} - \btheta^{*}\right\|_{\nabla^{2} P_{t}\left(\btheta^{*}\right)} \leq \rho_t \label{eq:theta-tz}
\end{align}
where $\rho_t=c\left(C_{1}C_{2} + 2\eta^2\lambda_1^2\right) \sqrt{\frac{\gamma \log (dt)}{t}}$.

\textbf{Step 6 (Define $\bM_{1, t}$ and $\bM_{2, t}$):} It follows from the inequality \eqref{eq:theta-ttheta} that 
\begin{align*}
&\nabla^{2} \wP_{t}(\ttheta_{t}) \preceq\left(1+\rho_t\right) \nabla^{2} P_{t}\left(\btheta^{*}\right)
\implies \nabla^{2} \wP_{t}(\ttheta_{t}) - \nabla^{2} P_{t}\left(\btheta^{*}\right) \preceq \rho_t \nabla^{2} P_{t}\left(\btheta^{*}\right) \\
&\implies \nabla^{2} P_{t}\left(\btheta^{*}\right)^{-1/2}(\wP_{t}(\ttheta_{t}) - \nabla^{2} P_{t}\left(\btheta^{*}\right)) \nabla^{2} P_{t}\left(\btheta^{*}\right)^{-1/2} \preceq \rho_t I
\\
&\implies \lVert \nabla^2\wP_{t}(\ttheta_{t}) - \nabla^{2} P_{t}\left(\btheta^{*}\right) \rVert_* \leq \rho_t.
\end{align*}
Then we can use the inequalities \eqref{eq:theta-ttheta} and \eqref{eq:theta-tz} to show that
\begin{align*}
&\left(1-\rho_t\right) \nabla^{2} P_{t}\left(\btheta^{*}\right) \preceq \nabla^{2} \wP_{t}(\ttheta_{t}) \preceq\left(1+\rho_t\right) \nabla^{2} P_{t}\left(\btheta^{*}\right)\\
&\left(1-\rho_t\right) \nabla^{2} P_{t}\left(\btheta^{*}\right) \preceq \nabla^{2} P_{t}\left(\tz_{t}\right) \preceq\left(1+\rho_t\right) \nabla^{2} P_{t}\left(\btheta^{*}\right).
\end{align*}
Now we define the two quantities $\bM_{1, t}$ and $\bM_{2, t}$ as follows:
\begin{align*}
\bM_{1, t} &\colonequals \left(\nabla^{2} P_{t}\left(\btheta^{*}\right)\right)^{1 / 2}\left(\nabla^{2} \wP_{t}(\ttheta_{t})\right)^{-1}\left(\nabla^{2} P_{t}\left(\btheta^{*}\right)\right)^{1 / 2} \\
\bM_{2, t} &\colonequals \left(\nabla^{2} P_{t}\left(\btheta^{*}\right)\right)^{-1 / 2} \nabla^{2} P_{t}\left(\tz_{t}\right)\left(\nabla^{2} P_{t}\left(\btheta^{*}\right)\right)^{-1 / 2}.
\end{align*}

\textbf{Step 7 (Lower bound $P_{t}(\wtheta_{t})-P_{t}\left(\btheta^{*}\right)$):} Now for the lower bound it follows from \Cref{eq:Pt-z} that
\begin{align*}
    P_{t}(\wtheta_{t})-P_{t}\left(\btheta^{*}\right) & = \frac{1}{2}\left(\wtheta_{t}-\btheta^{*}\right)^{\top} \nabla^{2} P_{t}\left(\tz_{t}\right)\left(\wtheta_{t}-\btheta^{*}\right)\\
    &= \frac{1}{2}\left(\wtheta_{t}-\btheta^{*}\right)^{\top}\nabla^2 P_t(\btheta^*)^{\frac{1}{2}}\nabla^2 P_t(\btheta^*)^{-\frac{1}{2}}\nabla^{2} P_{t}\left(\tz_{t}\right)\nabla^2P_t(\btheta^*)^{-\frac{1}{2}} \nabla^2 P_t(\btheta^*)^{\frac{1}{2}}\left(\wtheta_{t}-\btheta^{*}\right)\\
    &\overset{(a)}{=} \frac{1}{2} \mathbf{u}^T \mathbf{M}_{2,t} \mathbf{u}
\end{align*}
where, in $(a)$ we define the vector $\mathbf{u} \colonequals \left(\wtheta_{t}-\btheta^{*}\right)^{\top}\nabla^2 P_t(\btheta^*)^{\frac{1}{2}}$. Now observe from the definition of and then using the min-max theorem we can show that
\begin{align*}
P_{t}(\wtheta_{t})-P_{t}\left(\btheta^{*}\right) & \geq \frac{1}{2} \lambda_{\min }\left(\bM_{2, t}\right) \mathbf{u}^T\mathbf{u}\\
& = \frac{1}{2} \lambda_{\min }\left(\bM_{2, t}\right)\left\|\wtheta_{t}-\btheta^{*}\right\|_{\nabla^{2} P_{t}\left(\btheta^{*}\right)}^{2} \\
&\overset{}{=}\frac{1}{2} \lambda_{\min }\left(\bM_{2, t}\right)\left\|\nabla^{2} \wP_{t}(\ttheta_{t})\left(\wtheta_{t}-\btheta^{*}\right)\right\|_{\left(\nabla^{2} \wP_{t}(\ttheta_{t})\right)^{-1} \nabla^{2} P_{t}\left(\btheta^{*}\right)\left(\nabla^{2} \wP_{t}(\ttheta_{t})\right)^{-1}}^{2} \\
& \geq \frac{1}{2}\left(\lambda_{\min }\left(\bM_{1, t}\right)\right)^{2} \lambda_{\min }\left(\bM_{2, t}\right)\left\|\nabla^{2} \wP_{t}(\ttheta_{t})\left(\wtheta_{t}-\btheta^{*}\right)\right\|_{\left(\nabla^{2} P_{t}\left(\btheta^{*}\right)\right)^{-1}}^{2} \\
&\overset{(a)}{=}\frac{1}{2}\left(\lambda_{\min }\left(\bM_{1, t}\right)\right)^{2} \lambda_{\min }\left(\bM_{2, t}\right)\left\|\nabla \wP_{t}\left(\btheta^{*}\right)\right\|_{\left(\nabla^{2} P_{t}\left(\btheta^{*}\right)\right)^{-1}}^{2}
\end{align*}
where, in $(a)$ we use the \cref{eq:erm}.

\textbf{Step 8:} Define $I(\mathcal{E})$ as the indicator that the desired previous events hold, which we can ensure with probability greater than $1-2\left(\dfrac{1}{dt}\right)^{\gamma}$. Then we can show that:

\begin{align*}
 \E\left[P_{t}(\wtheta_{t})-P_{t}\left(\btheta^{*}\right)\right] 
\geq & \E\left[\left(P_{t}(\wtheta_{t})-P_{t}\left(\btheta^{*}\right)\right) I(\mathcal{E})\right] \\
\geq & \frac{1}{2} \E\left[\left(\lambda_{\min }\left(\bM_{1, t}\right)\right)^{2} \lambda_{\min }\left(\bM_{2, t}\right)\left\|\nabla \wP_{t}\left(\btheta^{*}\right)\right\|_{\left(\nabla^{2} P_{t}\left(\btheta^{*}\right)\right)^{-1}}^{2} I(\mathcal{E})\right] \\
\geq &\left(1-c^{\prime} \rho_t\right) \frac{1}{2} \E\left[\left\|\nabla \wP_{t}\left(\btheta^{*}\right)\right\|_{\left(\nabla^{2} P_{t}\left(\btheta^{*}\right)\right)^{-1}}^{2} I(\mathcal{E})\right] \\
=&\left(1-c^{\prime} \rho_t\right) \frac{1}{2} \E\left[\left\|\nabla \wP_{t}\left(\btheta^{*}\right)\right\|_{\left(\nabla^{2} P_{t}\left(\btheta^{*}\right)\right)^{-1}}^{2}(1-I(\operatorname{not} \mathcal{E}))\right] \\
\overset{(a)}{=}&\left(1-c^{\prime} \rho_t\right)\left(\sigma^{2}_t-\frac{1}{2} \E\left[\left\|\nabla \wP_{t}\left(\btheta^{*}\right)\right\|_{\left(\nabla^{2} P_{t}\left(\btheta^{*}\right)\right)^{-1}}^{2} I(\operatorname{not} \mathcal{E})\right]\right) \\
\geq &\left(1-c^{\prime} \rho_t\right) \sigma^{2}_t-\E\left[\left\|\nabla \wP_{t}\left(\btheta^{*}\right)\right\|_{\left(\nabla^{2} P_{t}\left(\btheta^{*}\right)\right)^{-1}}^{2} I(\operatorname{not} \mathcal{E})\right]
\end{align*}
where, in $(a)$ we have $\sigma^2_t\colonequals \left\|\nabla \wP_{t}\left(\btheta^{*}\right)\right\|_{\left(\nabla^{2} P_{t}\left(\btheta^{*}\right)\right)^{-1}}^{2}$, and $c'$ is an universal constant.

\textbf{Step 9:} Define the random variable $Z=\left\|\nabla \wP_{t}\left(\btheta^{*}\right)\right\|_{\left(\nabla^{2} P_{t}\left(\btheta^{*}\right)\right)^{-1}}$. With a failure event probability of less than $2\left(\dfrac{1}{dt}\right)^{\gamma}$ for any $z_{0},$ we have:
\begin{align*}
\mathbb{E}\left[Z^{2} I(\operatorname{not} \mathcal{E})\right] &=\E\left[Z^{2} I(\operatorname{not} \mathcal{E}) I\left(Z^{2} < z_{0}\right)\right]+\E\left[Z^{2} I(\operatorname{not} \mathcal{E}) I\left(Z^{2} \geq z_{0}\right)\right] \\
& \leq z_{0} \E[I(\operatorname{not} \mathcal{E})]+\E\left[Z^{2} I\left(Z^{2} \geq z_{0}\right)\right] \\
& \leq \frac{z_{0}}{2 t^{\gamma}}+\E\left[Z^{2} \frac{Z^{2}}{z_{0}}\right] \\
& \leq \frac{z_{0}}{2 t^{\gamma}}+\frac{\E\left[Z^{4}\right]}{z_{0}} \\
& \leq \frac{\sqrt{\E\left[Z^{4}\right]}}{t^{\gamma / 2}}
\end{align*}
where $z_{0}=t^{\gamma / 2} \sqrt{\mathbb{E}\left[Z^{4}\right]}$.

\textbf{Step 10 (Upper Bound): } For an upper bound we have that:
\begin{align*}
\E\left[P_{t}(\wtheta_{t})-P_{t}\left(\btheta^{*}\right)\right] &=\E\left[\left(P_{t}(\wtheta_{t})-P_{t}\left(\btheta^{*}\right)\right) I(\mathcal{E})\right]+\E\left[\left(P_{t}(\wtheta_{t})-P_{t}\left(\btheta^{*}\right)\right) I(\operatorname{not} \mathcal{E})\right] \\
& \leq \E\left[\left(P_{t}(\wtheta_{t})-P_{t}\left(\btheta^{*}\right)\right) I(\mathcal{E})\right]+\frac{\max_{\btheta \in \bTheta}\left(P_{t}(\btheta)-P_{t}\left(\btheta^{*}\right)\right)}{t^{\gamma}}
\end{align*}
since the probability of not $\mathcal{E}$ is less than $\dfrac{1}{t^{\gamma}}$. Now for an upper bound of the first term, observe that
\begin{align*}
\E\left[\left(P_{t}(\wtheta_{t})-P_{t}\left(\btheta^{*}\right)\right) I(\mathcal{E})\right] 
\leq & \frac{1}{2} \E\left[\left(\lambda_{\max }\left(\bM_{1, t}\right)\right)^{2} \lambda_{\max }\left(\bM_{2, t}\right)\left\|\nabla \wP_{t}\left(\btheta^{*}\right)\right\|_{\left(\nabla^{2} P_{t}\left(\btheta^{*}\right)\right)^{-1}}^{2} I(\mathcal{E})\right] \\
\leq &\left(1+c^{\prime} \rho_t\right) \frac{1}{2} \E\left[\left\|\nabla \wP_{t}\left(\btheta^{*}\right)\right\|_{\left(\nabla^{2} P_{t}\left(\btheta^{*}\right)\right)^{-1}}^{2} I(\mathcal{E})\right] \\
\leq &\left(1+c^{\prime} \rho_t\right) \frac{1}{2} \E\left[\left\|\nabla \wP_{t}\left(\btheta^{*}\right)\right\|_{\left(\nabla^{2} P_{t}\left(\btheta^{*}\right)\right)^{-1}}^{2}\right] \\
=&\left(1+c^{\prime} \rho_t\right) \frac{\sigma^{2}_t}{t}
\end{align*}
where, $c'$ is another universal constant.
\end{proof}

\subsection{Additional Experiment Details}
\label{app:expt}

\subsubsection{Hypothesis Testing Experiments}
\label{app:testing-addl-expt}
In all the active testing experiments 
we use the threshold function for the Gaussian distribution as proved in Lemma \ref{lemma:stop-time}. Hence the threshold function used is
\begin{align*}
    \beta = \log(J/\delta).
\end{align*}
Note that this threshold function is smaller than the general sub-Gaussian threshold function proved in Lemma \ref{lemma:stop-time-general}. 

\textbf{\Cref{ex:unif-dominates}: } Recall that the \Cref{ex:unif-dominates} is given by the following table under the three different hypotheses $\{\btheta^\ast, \btheta', \btheta''\}$
\begin{align*}
    \begin{matrix}
     \btheta &= & \btheta^* &\btheta'  & \btheta^{''} \\\hline
    \mu_1(\btheta) &=  & 1 & 0.001 & 0 \\
    \mu_2(\btheta) &=   & 1 & 1.002 & 0.998
\end{matrix}
\end{align*}
We can show that under $p_{\btheta^*}$ we have the following optimization problem
\begin{align*}
    &\max z \\
    \textbf{s.t.  } & p(1)0.999^2 + p(2)0.002^2 \geq z\\
    & p(1)1^2 + p(2)0.002^2 \geq z.
\end{align*}
The solution to the above optimization is given by $p_{\btheta^*} = [p(1), p(2)] = [1, 0]$. 
Similarly we can show that $p_{\btheta'}$ we have the following optimization problem
\begin{align*}
    &\max z \\
    \textbf{s.t.  } & p(1)0.001^2 + p(2)0.002^2 \geq z\\
    & p(1)0.001^2 + p(2)0.004^2 \geq z.
\end{align*}
The solution to the above optimization is given by $p_{\btheta'} = [p(1), p(2)] = [0, 1]$. Finally we can show that $p_{\btheta''}$ we have the following optimization problem
\begin{align*}
    &\max z \\
    \textbf{s.t.  } & p(1)0.001^2 + p(2)0.004^2 \geq z\\
    & p(1)1^2 + p(2)0.002^2 \geq z.
\end{align*}
The solution to the above optimization is given by $p_{\btheta''} = [p(1), p(2)] = [0, 1]$. Hence, $D_1 \colonequals \min_{p_{\btheta}\in\bTheta}\min_{\btheta'\neq \btheta^*}\sum_{i=1}^n p_{\btheta}(i)(\mu_{i}(\btheta') - \mu_{i}(\btheta^*))^2 = 0.002^2 = 4\times 10^{-6}$. Similarly, we can compute that $D_0 \colonequals \max_{\mathbb{p}} \min_{\btheta'\neq \btheta^*} \sum_{i=1}^n (\mu_i(\btheta') - \mu_i(\btheta^*))^2 = 0.999^2$. Hence the non-asymptotic term $(\log J)/D_1 = 0.3\times 10^6$ and the asymptotic term $\log(J/\delta)/D_0 = 3.4$. 

\textbf{
Active Testing environment ($3$ Group setting): } In this setting there are three groups of actions. In first group there is a single action that discriminates best between all pair of hypotheses. In second group there are $5$ actions which can discriminate one hypotheses from others. Finally in the third group there are $44$ actions which cannot discriminate between any pair of hypotheses. The following table describes the $\mu_1(\cdot), \mu_2(\cdot), \ldots, \mu_{50}(\cdot)$ under different hypotheses as follows:
\begin{align*}
    \begin{matrix}
     \btheta &= & \btheta^* &\btheta_2  & \btheta_3 & \btheta_4 & \btheta_5 & \btheta_{6} \\\hline
    \mu_1(\btheta) &=  & \textcolor{red}{3} & 0 & 0 & 0 & 0 & 0 \\
    \mu_2(\btheta) &=  & 2 & \textcolor{red}{3} & 2 & 2 & 2 & 2 \\
    \mu_3(\btheta) &=  & 2 & 2 & \textcolor{red}{3} & 2 & 2 & 2 \\
    \mu_4(\btheta) &=  & 2 & 2 & 2 & \textcolor{red}{3} & 2 & 2 \\
    \mu_5(\btheta) &=  & 2 & 2 & 2 & 2 & \textcolor{red}{3} & 2 \\
    \mu_6(\btheta) &=  & 2 & 2 & 2 & 2 & 2 & \textcolor{red}{3} \\
    \mu_7(\btheta) &=  & 1 + \iota_{7,1} & 1 + \iota_{7,2} & 1 + \iota_{7,3} & 1 + \iota_{7,4} & 1 + \iota_{7,5} & 1 + \iota_{7,6} \\
    \vdots & & & \vdots\\
    \mu_{50}(\btheta) &=  & 1 + \iota_{50,1} & 1 + \iota_{50,2} & 1 + \iota_{50,3} & 1 + \iota_{50,4} & 1 + \iota_{50,5} & 1 + \iota_{50,6}
\end{matrix}
\end{align*}
In the above setting,we define $\iota_{i,j}$ for the $i$-th action and $j$-th hypothesis as a small value close to $0$ and $\iota_{i,j} \neq \iota_{i',j'}$ for any pair of hypotheses $j,j'\in[J]$ and actions $i,i'\in[n]$.

\subsubsection{Active Regression Experiment for Non-linear Reward Model 
}
\label{app:active-learning-expt-addl}


\textbf{Algorithmic Details: } We describe each of the algorithm used in this setting as follows:
\begin{enumerate}
    \item \textbf{\emcm:} The \emcm algorithm of \citet{cai2016batch} first quantifies the  change  as  the  difference  between  the  current model  parameters  and  the  new  model  parameters  learned from  enlarged  training  data,  and then chooses  the  data  examples that  result  in  the  greatest  change. 
    \item \textbf{\cher:} The \cher policy used is stated as in \Cref{sec:active-regression}. To calculate the least square estimate $\widehat{\btheta}(t)$ we use the python scipy.optimize least-square function which solves a nonlinear least-squares problem.
    \item \textbf{\unif:} The \unif policy samples each action uniform randomly at every round.
    \item \textbf{\actives:} The \actives policy in  \citet{chaudhuri2015convergence} is a two-stage algorithm. It first samples all actions uniform randomly to build an initial estimate of $\btheta^*$. It then solves an Semi-definite Programming (SDP) to obtain a new sampling distribution that minimizes the quantity $\sigma^2_U$ as defined in \Cref{eq:sigma-U}. In the second stage \actives follows this new sampling distribution to sample actions.
\end{enumerate}

\textbf{Implementation Details: } This setting consist of $50$ measurement actions divided into three groups. The first group consist of the optimal action $\mathbf{x}_{i^*} \colonequals (1,0)$ in the direction of $\btheta^* \colonequals (1,0)$. The second group consist of the informative action $\mathbf{x}_2 \colonequals (0,1)$ which is orthogonal to $\mathbf{x}_{i^*}$ and selecting it maximally reduces the uncertainty of $\widehat{\btheta}(t)$. Finally the third group consist of $48$ actions such that $\mathbf{x}_{i} \colonequals (0.71 \pm \iota_i, 0.71 \mp \iota_i)$ for $i\in[3,50]$ where $\iota_i$ is a small value close to $0$ and $\iota_{i}\neq \iota_{i'}$. Note that these $48$ actions are less informative in comparison to action $2$. This is shown in \Cref{fig:nonlinerregression}.

\subsubsection{Active Regression Experiment for Neural Networks 
}
\label{app:active-learning-expt-NN-addl}


\textbf{Implementation Details:} At every time step, we use the the least squares optimizer of scipy to find $\widehat{\btheta}_t$. Since $c_1, c_2 \in \{-1, 1\}$, we solve four different least squares problems at each step corresponding to all $(c_1, c_2)$ choices, and use the values returned by the problem having the smallest sum of squares as our current estimate for $(\mathbf{w}_1, \mathbf{w}_2, b_1, b_2)$. The derivative with respect to any parameter is found by the backward pass of automatic differentiation.

\subsubsection{Active Regression for the UCI Datasets}
\label{app:real-dataset-expt}

\textbf{Implementation Details:} The UCI Red Wine Quality dataset consist of $1600$ samples of red wine with each sample $i$ having feature $\mathbf{x}_i\in\mathbb{R}^{11}$. We first fit a least square estimate to the original dataset and get an estimate of $\btheta^*$. The reward model is linear and given by $\mathbf{x}_{I_t}^T\btheta^* + \text{noise}$ where $x_{I_t}$ is the observed action at round $t$, and the noise is a zero-mean additive noise. Note that we consider the $1600$ samples as actions. Then we run each of our benchmark algorithms on this dataset and reward model and show the result in \Cref{fig:red-wine}. We further show the \cher proportion on this dataset in \cref{fig:red-wine-prop} and show that it is indeed sparse with proportion concentrated on few actions. The Air quality dataset consist of $1500$ samples each of which consist of $6$ features. We again build an estimate of $\btheta^*$ by fitting a least square regression on this dataset. We use a similar additive noise linear reward model as described before and run all the benchmark algorithms on this dataset. 

\begin{figure}
    \centering
    \includegraphics[scale = 0.4]{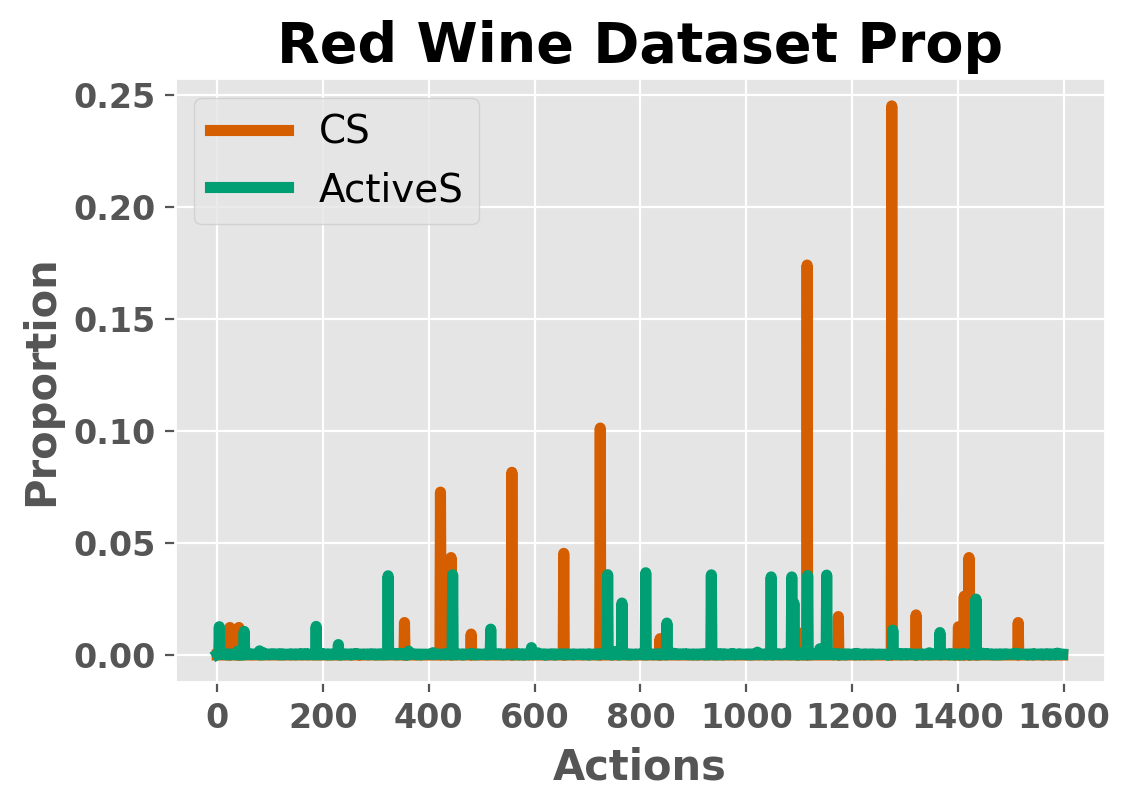}
    \caption{\cher Proportions over $1600$ actions in Red Wine Dataset. Note that \cher Proportion is sparse.}
    \label{fig:red-wine-prop}
\end{figure}

\newpage
\subsection{Table of Notations}
\label{app:notation}
\begin{table}[!th]
    \centering
    \begin{tabular}{c|c}
    \textbf{Notation} & \textbf{Definition}\\\hline
        $n$ & Total number of actions\\
        $J$ & Total number of hypotheses \\
        $\bTheta$& Parameter Space\\ 
        $\mu_{i}(\btheta)$ & Mean of action $i$ under hypothesis $\btheta$\\
        $\pi$ & Policy\\
        $\delta$ & Probability of error of $\delta$-PAC policy\\
        $\tau_\delta$ & Stopping time of $\delta$-PAC policy\\
        $\beta(J,\delta)$ & $\log(CJ/\delta)$, $C$ is a constant depending on $\eta, \eta_0$\\
        $\alpha(J)$ & $b\log(J), b >0$\\
        $Y^t$ & Vector of rewards observed till round $t$\\
        $I^t$ & Vector of actions sampled till round $t$\\
        $I_s$ & action sampled at round $s$\\
        $Z_i(t)$ & Number of time action $i$ is sampled till round $t$\\
        $\eta$ & Constant $>0$ s.t. $Y_s\in [-\eta/2, \eta/2]$\\
        $\eta_0$ & $ \min_{i \in [n]}\min_{\btheta \neq \btheta'} (\mu_{i}(\btheta) - \mu_{i}(\btheta'))^2$\\
        $\mathbf p_{\btheta}$ & p.m.f. to verify hypothesis $\btheta$ (Solution to Chernoff optimization in \eqref{eq:opt-lower00})\\
        $\KL(.||.)$ & Kullback-Leibler divergence\\
        $\widehat{\btheta}(t)$ & Most likely hypothesis at round $t$\\
        $\tilde{\btheta}(t)$ & Second most likely hypothesis at round $t$\\
        $L_{t}(\btheta)$ &  Sum of squared errors till round $t$ under hypothesis $\btheta$\\
        $\ell_{s}(\btheta)$ & Squared error at a specific round $s$ under hypothesis $\btheta$\\
        $\Delta_{t}(\btheta, \btheta^*)$ & $L_{t}(\btheta) - L_{t}(\btheta^*)$\\
         $\Delta_{s}(\btheta, \btheta^*)$ & $\ell_{s}(\btheta) - \ell_{s}(\btheta^*)$ at a specific round $s$\\
        $\xi^\delta(\btheta, \btheta^*)$ & Event that $ \{L_{\tau_\delta}(\btheta') - L_{\tau_\delta}(\btheta) > \beta(J,\delta), \forall \btheta' \neq \btheta\}$\\
        $(1+c)M$ & Critical number of samples $(1+c)O\left(\nicefrac{\log J}{D_1} + \nicefrac{\log(J/\delta)}{D_0}\right)$, for a constant $c>0$\\
        $\Is$ &  $ \{i\in[n]: i=\argmax_{i'\in[n]}(\mu_{i'}(\btheta) - \mu_{i'}(\btheta'))^2 \text{ for some } \btheta,\btheta'\in \bTheta\}$\\
        $\gamma$ & Constant $\geq 2$, controlling the convergence rate \\
        $d$ & Dimension of the parameter space $\bTheta$\\
        $C$ & $165 + \eta^2/\eta_0^2$\\
        $D_0$ &$ \max_{\text{ }\mathbf{p}} \min_{\btheta' \neq \btheta^\ast} \sum_{i=1}^n p(i) (\mu_i(\btheta') -\mu_i(\btheta^\ast))^2$\\
        $D_1$ &$\min_{\{\mathbf{p}_{\btheta} : \btheta \in \bTheta\}} \min_{\btheta' \neq \btheta^\ast} \sum_{i=1}^n p_{\btheta}(i)(\mu_i(\btheta') - \mu_i(\btheta^*))^2$\\
        $D_0'$ & $\min\limits_{\btheta,\btheta' \neq \btheta^*}\sum_{i=1}^n u_{\btheta^*\btheta}(i)(\mu_i(\btheta')- \mu_i(\btheta^*))^2$ \\
        $D_1'$ & $\min_{\btheta \neq  \btheta', \btheta' \neq \btheta^*}\sum_{i=1}^n u_{\btheta\btheta'}(i)(\mu_i(\btheta') - \mu_i(\btheta^*))^2$\\
        $P_t(\btheta)$ & $ \frac{1}{t}\sum_{s=1}^t\E_{I_s\sim \mathbf{p}_{\wtheta_{s-1}}}[L_s(\btheta)|\F^{s-1}]$\\
        $\mathbf{p}_{\mathbf{unif}}$ & pmf of a uniform distribution over the actions.\\
        $P_U(\btheta)$ & $\E_{I_s\sim \mathbf{p}_{\mathbf{unif}}}[L_s(\btheta)]$\\
        $\sigma^2_t$ & $\E\left\|\nabla \wP_{t}\left(\btheta^{*}\right)\right\|_{\left(\nabla^{2} P_{t}\left(\btheta^{*}\right)\right)^{-1}}^{2}$\\
        $\overline{\sigma^2}_t$ & $\E\left\|\nabla \wP_{t}\left(\btheta^{*}\right)\right\|_{\left(\nabla^{2} P_{U}\left(\btheta^{*}\right)\right)^{-1}}^{2}$ \\
        $\sigma^2_U$ & $\operatorname{Trace}\left[  I_{U}\left(\btheta^{*}\right)  I_{\Gamma}\left(\btheta^{*}\right)^{-1}\right]$ where $I_U$ and $I_\Gamma$ are Fisher Information matrices.\\
        \\\hline
    \end{tabular}
    \caption{Table of Notations}
    \label{tab:notations}
\end{table}

\end{document}